\PassOptionsToPackage{unicode}{hyperref}
\PassOptionsToPackage{hyphens}{url}
\PassOptionsToPackage{dvipsnames,svgnames,x11names}{xcolor}

\documentclass[ 
  12pt]{article}

\usepackage{amsmath,amssymb}
\usepackage{iftex}
\ifPDFTeX
  \usepackage[T1]{fontenc}
  \usepackage[utf8]{inputenc}
  \usepackage{textcomp} 
\else 
  \usepackage{unicode-math}
  \defaultfontfeatures{Scale=MatchLowercase}
  \defaultfontfeatures[\rmfamily]{Ligatures=TeX,Scale=1}
\fi
\usepackage{lmodern}
\ifPDFTeX\else  
\fi
\IfFileExists{upquote.sty}{\usepackage{upquote}}{}
\IfFileExists{microtype.sty}{
  \usepackage[]{microtype}
  \UseMicrotypeSet[protrusion]{basicmath} 
}{}
\makeatletter
\@ifundefined{KOMAClassName}{
  \IfFileExists{parskip.sty}{%
    \usepackage{parskip}
  }{
    \setlength{\parindent}{0pt}
    \setlength{\parskip}{6pt plus 2pt minus 1pt}}
}{
  \KOMAoptions{parskip=half}}
\makeatother
\makeatletter
\ifx\paragraph\undefined\else
  \let\oldparagraph\paragraph
  \renewcommand{\paragraph}{
    \@ifstar
      \xxxParagraphStar
      \xxxParagraphNoStar
  }
  \newcommand{\xxxParagraphStar}[1]{\oldparagraph*{#1}\mbox{}}
  \newcommand{\xxxParagraphNoStar}[1]{\oldparagraph{#1}\mbox{}}
\fi
\ifx\subparagraph\undefined\else
  \let\oldsubparagraph\subparagraph
  \renewcommand{\subparagraph}{
    \@ifstar
      \xxxSubParagraphStar
      \xxxSubParagraphNoStar
  }
  \newcommand{\xxxSubParagraphStar}[1]{\oldsubparagraph*{#1}\mbox{}}
  \newcommand{\xxxSubParagraphNoStar}[1]{\oldsubparagraph{#1}\mbox{}}
\fi
\makeatother

\usepackage{longtable,booktabs,array}
\usepackage{calc} 
\usepackage{etoolbox}
\makeatletter
\patchcmd\longtable{\par}{\if@noskipsec\mbox{}\fi\par}{}{}
\makeatother
\IfFileExists{footnotehyper.sty}{\usepackage{footnotehyper}}{\usepackage{footnote}}
\makesavenoteenv{longtable} 
\usepackage{graphicx}
\makeatletter
\def\maxwidth{\ifdim\Gin@nat@width>\linewidth\linewidth\else\Gin@nat@width\fi}
\def\maxheight{\ifdim\Gin@nat@height>\textheight\textheight\else\Gin@nat@height\fi}
\makeatother
\setkeys{Gin}{width=\maxwidth,height=\maxheight,keepaspectratio}
\makeatletter
\def\fps@figure{htbp}
\makeatother

\makeatletter
\@ifpackageloaded{caption}{}{\usepackage{caption}}
\AtBeginDocument{%
\ifdefined\contentsname
  \renewcommand*\contentsname{Table of contents}
\else
  \newcommand\contentsname{Table of contents}
\fi
\ifdefined\listfigurename
  \renewcommand*\listfigurename{List of Figures}
\else
  \newcommand\listfigurename{List of Figures}
\fi
\ifdefined\listtablename
  \renewcommand*\listtablename{List of Tables}
\else
  \newcommand\listtablename{List of Tables}
\fi
\ifdefined\figurename
  \renewcommand*\figurename{Figure}
\else
  \newcommand\figurename{Figure}
\fi
\ifdefined\tablename
  \renewcommand*\tablename{Table}
\else
  \newcommand\tablename{Table}
\fi
}
\@ifpackageloaded{float}{}{\usepackage{float}}
\floatstyle{ruled}
\@ifundefined{c@chapter}{\newfloat{codelisting}{h}{lop}}{\newfloat{codelisting}{h}{lop}[chapter]}
\floatname{codelisting}{Listing}

\makeatother
\makeatletter
\@ifpackageloaded{caption}{}{\usepackage{caption}}
\@ifpackageloaded{subcaption}{}{\usepackage{subcaption}}
\makeatother

\ifLuaTeX
  \usepackage{selnolig}  
\fi
\usepackage[]{natbib}
\usepackage{bookmark}

\IfFileExists{xurl.sty}{\usepackage{xurl}}{} 
\hypersetup{
  pdftitle={Title},
  pdfauthor={Author 1; Author 2},
  pdfkeywords={3 to 6 keywords, that do not appear in the title},
  colorlinks=true,
  linkcolor={blue},
  filecolor={Maroon},
  citecolor={Blue},
  urlcolor={Blue},
  pdfcreator={LaTeX via pandoc}}

\newcommand{\anon}{1}

\usepackage{amsfonts}
\usepackage{mathrsfs}
\usepackage[english]{babel}
\usepackage{mathtools}
\usepackage{amsthm}
\usepackage{comment}
\usepackage{blindtext}
\usepackage[table]{xcolor}
\usepackage{float}
\usepackage{booktabs}
\usepackage{diagbox}
\usepackage{listings}
\usepackage{algorithm}
\usepackage[noend]{algpseudocode}
\usepackage{xr}
\DeclareMathOperator*{\argmin}{arg\,min}

\usepackage[margin=1.0in]{geometry}

\theoremstyle{plain}
\newtheorem{theorem}{Theorem}
\newtheorem{corollary}[theorem]{Corollary}
\newtheorem{proposition}{Proposition}
\newtheorem{lemma}[proposition]{Lemma}
\theoremstyle{definition}
\newtheorem{definition}{Definition}

\theoremstyle{remark}
\newtheorem{remark}{Remark}

\newtheorem{example}{Example}

\usepackage{calc}
\usepackage{tikz}
\tikzstyle{vertex}=[circle, draw, inner sep=0pt, minimum size=6pt]

\newcommand{\eps}{\varepsilon}
\newcommand{\cW}{\check{W}}
\newcommand{\cY}{\check{Y}}

\newcommand{\R}{\mathbb{R}}

\definecolor{britishracinggreen}{rgb}{0.0, 0.26, 0.15}
\definecolor{burntorange}{rgb}{0.8, 0.33, 0.0}

\definecolor{blendedblue}{rgb}{0.2,0.2,0.7}

\newcommand{\be}{\beta}
\newcommand{\ga}{\gamma}

\usepackage{soul, ulem}

\begin{document}

\def\spacingset#1{\renewcommand{\baselinestretch}%
{#1}\small\normalsize} \spacingset{1}


\if1\anon
{
  \title{\bf A variational Bayes approach to inference for low-dimensional parameters in high-dimensional linear regression}
  \author{Isma\"el Castillo, Alice L'Huillier 
    \hspace{.2cm}\\
    LPSM, Sorbonne Universit\'e\\
    and \\
    Kolyan Ray, Luke Travis\\
    Department of Mathematics, Imperial College London}
  \maketitle
} \fi

\if0\anon
{
  \bigskip
  \bigskip
  \bigskip
  \begin{center}
    {\LARGE\bf Title}
\end{center}
  \medskip
} \fi

\bigskip
\begin{abstract}
We propose a scalable variational Bayes method for statistical inference for a single or pre-specified low-dimensional subset of the coordinates of a high-dimensional parameter in sparse linear regression. Our approach relies on assigning a mean-field approximation to the nuisance coordinates and carefully modelling the conditional distribution of the target given the nuisance. This requires only a preprocessing step and preserves the computational advantages of mean-field variational Bayes, while ensuring accurate and reliable inference for the target parameter, including for uncertainty quantification. We investigate the numerical performance of our algorithm, showing that it performs competitively with existing methods. We further establish accompanying theoretical guarantees for estimation and uncertainty quantification in the form of a Bernstein--von Mises theorem for a possibly growing subset.
\end{abstract}

\noindent%
\vfill

\newpage
\spacingset{1.8} 

\section{Introduction}
Consider high-dimensional linear regression
\begin{equation}
Y = X\beta + \sigma\eps, \hspace{10mm}\eps \sim \mathcal{N}_n(0, I_n),	\label{eq:linear_regression}
\end{equation}
with response $Y \in \R^n$, design matrix $X \in \mathbb{R}^{n\times p}$, parameter vector $\beta \in \mathbb{R}^p$ and noise level $\sigma>0$. We are interested in the sparse setup, where $p \geq n$ and typically $p \gg n$, and many of the coefficients $\beta_i$ are (approximately) zero. We focus here on estimation and uncertainty quantification for a fixed or slowly growing prespecified \textit{low-dimensional} number $k=k_n$ of coordinates of the high-dimensional parameter $\beta$. This setting arises when scientific interest concerns a small prespecified set of variables while a much larger collection of potential predictors or controls must be accommodated. Examples include genetic association studies, estimation of treatment or policy effects with high-dimensional controls \citep{Belloni2012} and responses to economic shocks in macroeconomic models \citep{adamek2024}. For notational simplicity and without loss of generality, we consider the first $k$ coordinates $\beta_{1:k}=(\beta_1,\dots,\beta_k)^T$, relabeling the coordinates if necessary. 

Frequentist estimation of the entire parameter $\beta$ has been extensively studied in the sparse setting, with the most common approach being the LASSO, which has many good estimation properties \citep{Buhlmann2011}. However, it is well-known that high-dimensional procedures, including the LASSO, can provide biased estimators for low-dimensional parameters \citep{Zou2006}, which is especially harmful for uncertainty quantification. This has led to several approaches to debias such frequentist methods \citep{Zhang2014,Javanmard2014,vandeGeerBuhlmann2014,JM2018,DY2019,CM24} and thus provide reliable confidence sets in the present problem.

Sparse Bayesian methods have similarly received much recent study for \textit{global} estimation and inference, see for instance \cite{Castillo_2012, jr10, CS-HV2015,Ray_Szabo_2020,Bai2022, bcg21} and the references therein. However, much less is known concerning Bayesian estimation of low-dimensional functionals in sparse high-dimensional models. Indeed, similar to frequentist methods, naively using high-dimensional Bayesian methods to estimate functionals can lead to `regularization bias' and poor uncertainty quantification, meaning care must be taken with the choice of prior. 
In particular, there has been little work on Bayesian approaches to the present problem of estimating a low-dimensional subset of the parameter $\beta$ in model \eqref{eq:linear_regression}. A notable exception is \cite{DY2019}, who combines a clever reparametrization of the likelihood with certain sparse priors \citep{GvdVZ_2020} to establish a one-dimensional posterior asymptotic normality result for the first coordinate $\beta_1$. Another result is \cite{CS-HV2015}, who provide a global Bernstein-von Mises theorem under strong signal strength conditions that implies a similar result, see also \cite{wuetal23}. However, the discrete model selection priors used in both approaches can make computation hugely challenging, especially for modern problem sizes of interest. Our goal here is thus to provide {\it computationally scalable and statistically reliable Bayesian inference for a low-dimensional subset of the coordinates of $\beta$}, especially for uncertainty quantification.
  
A popular scalable posterior approximation method is variational Bayes (VB), which approximates the true posterior by the closest element in Kullback-Leibler sense from a family of more tractable distributions. This involves solving an optimization problem which, for suitable variational families, can dramatically increase scalability \citep{Bishop2006,Blei2017}. However, this requires a statistical versus computational tradeoff, with simpler variational families leading to faster computation but worse posterior approximation, potentially leading to poor statistical behaviour. An especially popular variational approximation is \textit{mean field} (MF) VB, where the variational family consists of distributions under which the model parameters are independent, breaking any dependencies in the posterior approximation and losing any correlation information. MF VB has been successfully used to approximate sparse priors, such as the spike and slab, in a number of settings \citep{logsdon2010,carbonetto2012,ormerod2017,Ray_Szabo_2020,RaySzaboClara2020,komodromos2022}. While fast, available results on mean-field VB often show that this factorizable approach provides overconfident (wrong) uncertainty quantification by underestimating the posterior variance \citep{Bishop2006,Blei2017}, including for the \textit{marginal} distributions of the coordinates in sparse settings \citep{komodromos2022}.

The problem with standard MF VB is that it selects a diagonal covariance matrix to match the \textit{precision} rather than the \textit{covariance} matrix of the posterior. In correlated settings, as are common in practice, this can lead to an underestimate of the posterior variance. Our solution is to consider a parameter transformation $\beta_{1:k} \to \beta_{1:k}^*$ used in \cite{DY2019} that orthogonalizes the likelihood of $\beta_{1:k}^*$ and $\beta_{-k} = (\beta_{k+1},\dots,\beta_p)$, thereby decorrelating $\beta_{1:k}^*$ and $\beta_{-k}$ under the posterior. We then use a variational family making $\beta_{1:k}^*$ and $\beta_{-k}$ independent, with a rougher MF approximation on the high-dimensional nuisance parameter $\beta_{-k}$ and a richer approximation on the low-dimensional $\beta_{1:k}^*$, the latter of which correctly captures the posterior covariance of $\beta_{1:k}^*$ as the two components are now decorrelated.
Undoing this transformation, our variational approximation then correctly captures the posterior variance component of $\beta_{1:k}$ coming from $\beta_{1:k}^*$, with a potential underestimation of the component coming from $\beta_{-k}$ due to the MF approximation. This last variance component is typically of smaller order, so its underestimation has much less of an effect than when using full MF VB, leading to dramatically improved uncertainty quantification.
More broadly, our approach can be viewed as a form of generalized MF approximation in a transformed parameter space which decorrelates the posterior. Viewed in the original untransformed parameter space, this is equivalent to using MF VB for the nuisance parameter $\beta_{-k}$, and then more carefully modelling the conditional distribution $\beta_{1:k}|\beta_{-k}$, which remains computationally feasible since $\beta_{1:k}$ is low-dimensional. Our main contribution is to propose a variational family based on this orthogonalizing transformation that allows statistically reliable inference for $\beta_{1:k}$ (assuming the underlying posterior does) while preserving the computational advantages of MF VB.

For computational reasons, we consider a spike--and--slab type prior related to the theoretically motivated but computationally infeasible prior from \cite{DY2019}. 
Since this prior is based on the same reparametrization discussed above, it aligns with our variational family and thus has the attractive feature of decoupling computation of $\beta_{1:k}^*$ and $\beta_{-k}$. Both optimization problems are then obtained by a simple pre-processing step consisting of projecting the data and design matrix onto orthogonal spaces, giving two new linear regression problems that can be solved in parallel. In particular, one can directly use existing implementations for both these new linear regression steps without requiring any modification to the original algorithms, for instance using precise (low-dimensional) samplers for $\beta_{1:k}^*$ and fast coordinate ascent variational ascent (CAVI) for the more expensive nuisance parameter $\beta_{-k}$. Our approach can thus exploit the computational scalability of standard MF VB while still providing statistically reliable inference for $\beta_{1:k}$. 

We apply our method to numerical simulations and real covariate data, showing that it provides fast and reliable statistical inference for $\beta_{1:k}$ in practice, including for uncertainty quantification, performing at least as well as existing frequentist methods, notably the debiased LASSO \citep{Zhang2014,Javanmard2014,vandeGeerBuhlmann2014}. A particular strength of our method is that it appears to approximate quite well the correlation structure of the underlying posterior for $\beta_{1:k}$ in a variety of settings, leading to (variational) Bayesian credible sets with good frequentist coverage and size. In particular, it continues to perform well in correlated settings, in stark contrast to full MF VB. We further provide asymptotic theoretical guarantees for our method in the form of a semiparametric Bernstein--von Mises theorem, showing  that when the design is not too correlated, the marginal posterior for $\beta_{1:k}$ will be asymptotically normal and centered at an efficient estimator and with optimal covariance as $n,p \to \infty$. One consequence is that standard credible sets are asymptotic confidence sets of the right level, thereby providing a frequentist justification for this variational Bayes approach.

Our method shares connections to existing work studying modified versions of the mean-field approach to improve VB inference. For instance, \cite{Tan2013,Tan2021} also consider reparametrizations based on linear transformations before applying MF VB to improve global posterior approximation. Other approaches use blocking techniques \citep{Goplerud2022,Menictas2023}, i.e. grouping together certain parameters in the MF factorization (as we do for $\beta_{1:k}^*$ here), or partial factorizations \citep{Fasano2022,goplerud2023}. Another approach is to apply MF VB on a subset of parameters and some other approximation methods on the remaining parameters \citep{yuetal20, youetal23}.
Our work can be viewed as extending some of these concepts to semiparametric estimation, with a carefully tailored variational family designed to provide fast and reliable VB inference for the low-dimensional functional $\beta_{1:k}$ without requiring modifications to existing software implementations. On the theory side, there has been much recent work on the asymptotic properties of VB methods, but this has mostly focused on convergence rates or posterior approximation, e.g. \cite{ar20, ypb20, zg20,Ray_Szabo_2020}. For uncertainty quantification, \cite{WangBlei2019} establish a Bernstein--von Mises result for low-dimensional parametric models, confirming that credible sets from full MF VB can suffer from under coverage.

\textbf{Organization.} In Section \ref{sec:methodology}, we describe our variational Bayesian methodology and provide a sampling algorithm. Section \ref{sect:simulations} presents our numerical results, Section \ref{sec:BvM_VB_1d} our theoretical asymptotic guarantees and Section \ref{sec:discussion} our conclusions. In the supplementary material, we present additional numerical results in Section \ref{sec:additional_simulations} and heuristic calculations explaining the performance of our method in the simulations for strongly correlated settings in Section \ref{sec:heuristics}.
All proofs and further theoretical results are deferred to Section \ref{sec:proofs}, a discussion of the computational cost of our algorithm in Section \ref{sec:computational_cost}, a discussion of the conditions for our theoretical results is given in Section \ref{sec::Additional_result} and further background on our variational method is given in Section \ref{sec:additional_VB}.

{\bf Notation.} For $\beta \in \R^p$, we denote $\beta_{1:k}=(\beta_1,\dots,\beta_k)^T \in \R^{k}$ the vector consisting of the first $k$ coordinates of $\beta$ and $\beta_{-k}=(\beta_{k+1},\dots,\beta_{p})^T \in \R^{p-k}$. We write $X_i \in \R^n$ for the $i^{th}$ column of $X$, $X_{1:k}=(X_1, \dots, X_k) \in \R^{n \times k}$ for the matrix consisting of the first $k$ columns of $X$ and $X_{-k} = (X_{k+1}, \dots, X_p) \in \R^{n \times (p-k)}$. Similarly, for $\Pi$ a measure on $\R^p$ we denote $\Pi_{-k}$ for the measure induced on the last $p-k$ coordinates. The Lebesgue measure on $\R^k$ is denoted $\Lambda= \Lambda_k$, $\|\cdot\|_2$ is the Euclidean norm, and for a matrix $X$, we define $\|X\| =\max_{i = 1, \dots, p}\|X_i\|_2$.
For $\lambda>0$, we denote $\text{Lap}(\lambda)$ the Laplace distribution on $\R$ with density $(\lambda/2) e^{-\lambda|x|}$.
 The Kullback-Leibler divergence between probability distributions $P$ and $Q$ is denoted $\operatorname{KL}(P||Q)$. For $\beta\in \R^p$, denote by $S_\beta = \{i: \beta_i \neq 0 \}$ the set of non-zero coefficients of $\beta$. For $S \subset \{1,\dots,p\}$ let $\beta_S = (\beta_i)_{i\in S}$ and $|S|$ be the cardinality of $S$. For our theoretical results, we take $\sigma^2=1$ and write $P_0$ for the probability measure induced by the model $Y = X\beta^0 + \eps$, where $\eps \sim \mathcal{N}_n(0, I_n)$.

\section{Methodology}\label{sec:methodology}
We now introduce our variational Bayes methodology to conduct fast and accurate inference, including uncertainty quantification, on a low-dimensional subset $\beta_{1:k} = (\beta_1,\dots,\beta_k)^T$ of the high dimensional parameter $\beta \in \R^p$ in model \eqref{eq:linear_regression} with $k\ll p$ a fixed integer. Our method applies equally to any subset $\beta_{S}$, $S \subseteq \{1,\dots,p\}$ with $|S|=k$, but without loss of generality we may write $S = \{1,\dots,k\}$ by relabelling the parameters.
The noise variance $\sigma^2$ in \eqref{eq:linear_regression} is typically unknown and must be estimated. We follow \cite{CS-HV2015,Ray_Szabo_2020} and compute an estimator $\hat{\sigma}^2$ of $\sigma^2$, and divide both sides of the regression model by $\hat{\sigma}$, yielding $\tilde{Y} = Y/\hat{\sigma}$, $\tilde{X} = X/\hat{\sigma}$ and $\tilde{\eps} = (\sigma/\hat{\sigma})\eps$, where $\eps \sim \mathcal{N}_n(0, I_n)$. If the estimator $\hat{\sigma}$ is close to $\sigma$, we should approximately recover the $\sigma=1$ case using instead $(\tilde{X},\tilde{Y})$. For our method, we first require such a transformation and thus apply the following methodology with $\sigma =1$, using the notation $(X,Y)$ to refer to these normalized values (i.e. $(\tilde{X},\tilde{Y})$) for simplicity.
One could alternatively place a hyperprior on $\sigma^2$, but using a plug-in estimate brings certain computational advantages as we describe below.
For clarity of exposition, we first describe the case $k = 1$ before treating the general case.

\subsection{Prior and posterior for $k=1$}
For our prior on $\beta$, we follow \cite{DY2019} and place a sparse model selection prior on the nuisance parameter $\beta_{-1} = (\beta_2,\dots,\beta_p)^T$ and a carefully chosen conditional distribution on $\beta_1|\beta_{-1}$. The latter choice is motivated by the following decomposition of the likelihood, discussed further below.
Define $H = X_1 X_1^T/\|X_1\|_2^2$ to be the projection matrix onto span$(X_1)$, and $\gamma_i = X_1^T X_i/\|X_1\|_2^2$, $i=2,\dots, p$, to be the rescaled correlations between the $i^{th}$ and $1^{st}$ columns. By using the orthogonal decomposition $I_n = H + (I_n - H)$, the likelihood equals
\begin{align}
	\mathcal{L}_n(\beta, Y) 
	 &\propto \exp\left\{-\frac{1}{2}\|Y - X\beta\|_2^2 \right\} \nonumber\\
	&\propto \exp\left\{-\frac{1}{2}\left\|HY - X_1 \beta_1^*\right\|_2^2 \right\} \cdot \exp\left\{-\frac{1}{2}\|(I_n-H)Y - (I_n-H)X_{-1}\beta_{-1}\|_2^2 \right\}, \label{eq:likelihood_decomp}
\end{align}
where
\begin{equation}\label{eq:beta_1*}
\beta_1^* := \beta_1 + \sum_{i = 2}^p\gamma_i \beta_i.
\end{equation}
The first exponential contains all the $\beta_1$ terms via the transformed parameter $\beta_1^*$, while the second term involves only $\beta_{-1}$. Using the above bijective reparametrization $\beta = (\beta_1^*,\beta_{-1})$, the likelihood decouples $\beta_1^*$ from $\beta_{-1}$, with the dependence of $\beta_1$ and $\beta_{-1}$ fully captured through \eqref{eq:beta_1*}. This transformation can be viewed as an orthogonalization to approximately decorrelate the parameter of interest $\beta_1$ from the nuisance parameter $\beta_{-1}$ under the data. Note that if $X_1$ is orthogonal to $X_{-1}$, as might occur in experimental design setups, then $\beta_1^* = \beta_1$ since $X_{-1} \beta_{-1}$ carries no information about $\beta_1$. For the Bayesian, if the priors on $\beta_1^*$ and $\beta_{-1}$ are independent, then the corresponding posteriors will also be independent, allowing a particularly simple characterisation of the posterior. For this reason, we place independent priors on $\beta_1^*$ and $\beta_{-1}$, with $\beta_1^*$ sampled from a continuous distribution on $\R$. This induces a conditional (non-independent) prior on $\beta_1|\beta_{-1}$, but it is often helpful to think in terms of the decorrelated parameterization $(\beta_1^*,\beta_{-1})$. For the sparse prior on $\beta_{-1}$, we use the following model selection priors \citep{CS-HV2015,Ray_Szabo_2020}.
 \begin{definition}[Model selection prior] \label{def:sparse_prior}
 For $d\geq 1$, $\nu$ a probability distribution on $\{0, \dots, d\}$ and $\lambda > 0$, the \textit{model selection} prior on $u \in \R^d$, denoted $MS_d(\nu, \lambda)$, is determined in the following hierarchical manner:
 	\begin{enumerate}
 		\item the sparsity $s \in \{0,1,\dots,d\}$  of $u$ is drawn according to $s\sim \nu$,
 		\item the active set $S$ given $|S|=s$ of $u$ is drawn uniformly from the ${d \choose s}$ subsets of $\{1,\dots,d\}$ of size $s$,
 		\item $u_i | S \overset{ind}{\sim} \begin{cases}
 			\,\text{Lap}(\lambda),& i \in S, \\
 			\,\delta_0,& i \notin S,
 		\end{cases}$ 
 	\end{enumerate}
 	 where $\delta_0$ denotes the Dirac mass at zero. We require that $\nu$ satisfies
   \begin{align}\label{assum::prior_1}
	A_1d^{-A_3} \nu(s-1) &\leq \nu(s) \leq A_2 d^{-A_4} \nu(s-1), \qquad \text{for }s=1,\dots, d, 
\end{align}
for some constants $A_1, A_2, A_3, A_4 > 0$. 
 \end{definition}
 For our practical implementation, we use the spike-and-slab prior $\beta_i \sim^{iid} q \textrm{Lap}(\lambda) + (1-q)\delta_0$, where $q \in [0, 1]$ is the prior inclusion probability of each $\beta_i$,
which fits within Definition \ref{def:sparse_prior} with $\nu \sim \textrm{Binomial}(p, q)$. Note that taking lighter than exponential tails can lead to overshrinkage and poor performance \citep{Castillo_2012}; other heavy tailed priors, such as Cauchy or non-local priors \citep{jr10}, could be considered, but we focus here on Laplace slabs for simplicity. It remains to specify the conditional prior on $\beta_1|\beta_{-1}$.

 \begin{definition}\label{def_general_prior} For $\nu$ a probability distribution on $\{0, \dots, p-1\}$ , $\lambda>0$ and $g$ a positive density with respect to Lebesgue measure $\Lambda=\Lambda_1$, consider the hierarchical prior distribution $\Pi$ on $\beta \in \R^p$:
 \begin{equation}\label{general_prior}
 	\begin{aligned}
 		\beta_{-1} &\sim MS_{p-1}(\nu, \lambda)\\
 		 \beta_1 | \beta_{-1} & \sim g\left(\cdot + \sum_{i = 2}^p \gamma_i \beta_i\right) d\Lambda
 	\end{aligned}
 \end{equation}
(the notation $Y\sim h(\cdot)d\Lambda$ means that $Y$ has density $h$ with respect to $\Lambda$).  
 \end{definition}
This prior can equivalently be written as taking $\beta_{1}^* \sim g$ and $\beta_{-1} \sim MS_{p-1}(\nu,\lambda)$ \textit{independent}. The above prior is similar to that considered in \cite{DY2019}, though with two main differences. First, \cite{DY2019} places a different model selection prior on $\beta_{-1}$, notably the abstract prior introduced in \cite{GvdVZ_2020}. While this prior is amenable to very refined theoretical results, it is challenging to implement in practice. In contrast, we consider priors more widely used in practice, notably the spike and slab in our implementation.
     Second, we do not require the prior density $g$ of $\beta_1^*$ to be Gaussian, and will see below that choosing a distribution with heavier tails performs better numerically and theoretically. 
We consider three concrete examples of priors $g$ for $\beta_1|\beta_{-1}$ in \eqref{general_prior}, including an improper one, with varying tail decays. 

 \begin{example}[Laplace prior]\label{laplace_prior}
 	For $\sigma_n>0$, take $g=g_n$ the $\text{Lap}(1/\sigma_n)$ distribution. Then the prior reduces to $\beta_{-1} \sim MS_{p-1}(\nu, \lambda)$, $\beta_1 | \beta_{-1} \sim \text{Lap}(1/\sigma_n) -\sum_{i = 2}^p\gamma_i\beta_i$.  
 \end{example}
 \begin{example}[Gaussian Prior]\label{gaussian_prior}
 	For $\sigma_n^2>0$, take $g=g_n$ the $\mathcal{N}(0,\sigma_n^2)$ distribution. Then the prior reduces to $\beta_{-1} \sim MS_{p-1}(\nu, \lambda)$, $\beta_1 | \beta_{-1} \sim \mathcal{N}(-\sum_{i = 2}^p\gamma_i\beta_i, \sigma_n^2)$.
 \end{example}
 \begin{example}[Improper prior]\label{improper_prior} Consider the prior measure $\Lambda \otimes MS_{p-1}(\nu, \lambda)$, which can be seen as the prior \eqref{general_prior} with $g\propto 1$.
 \end{example}
 We find that the Laplace and improper priors work best in practice, having heavier tails, and would recommend using these by default. Note that the improper $g \propto 1$ does not require hyperparameter tuning for $\sigma_n$, which the Gaussian prior is especially sensitive to.
Our approach requires sufficiently good estimation of the nuisance parameter $\beta_{-1}$ and one could replace the model selection prior in \eqref{general_prior} by other priors that do so in sparse settings, such as the horseshoe \citep{carvalho2010horseshoe}.

Since both the likelihood \eqref{eq:likelihood_decomp} and prior \eqref{general_prior} factorize in $(\beta_1^*,\beta_{-1})$, so too does the resulting posterior, which we next characterize. Since $H$ is the projection matrix onto $\textrm{span}(X_1)$, $I_n-H$ is the orthogonal projection onto $\textrm{span}(X_1)^\perp$, an $(n-1)-$dimensional subspace of $\R^n$. To avoid degeneracy issues, we wish to consider such a projection mapping into $\R^{n-1}$ rather than $\R^n$, and we thus define $P \in \R^{n \times (n-1)}$ as any matrix whose columns form an orthonormal basis of $\textrm{span}(X_1)^\perp$, for which $PP^T = (I_n-H)$ and $P^TP = I_{n-1}$. We then define
\begin{equation}\label{eq:preprocess}
\cW = P^TX_{-1} \in \R^{(n-1) \times (p-1)}, \qquad \cY=P^TY \in \R^{n-1}, \qquad \check{\eps}=P^T\eps,
\end{equation}
where premultiplying by $P^T$ is essentially premultiplying by $(I_n-H)$ but with the image of the map in $\R^{n-1}$ rather than an $(n-1)-$dimensional subspace of $\R^n$. 
The logarithm of the second term in the likelihood decomposition \eqref{eq:likelihood_decomp} can then be written in terms of $\cY$ and $\cW$ as
$$-\tfrac{1}{2} \|(I_n-H)Y - (I_n-H)X_{-1}\beta_{-1}\|_2^2 = -\tfrac{1}{2}\|P(\cY - \cW \beta_{-1})\|^2_2 = -\tfrac{1}{2}\|\check{Y} - \cW \beta_{-1}\|_2^2,$$
i.e. the log-likelihood of the linear model $\cY = \cW \beta_{-1} + \check{\eps}$ with $\check{\eps} \sim \mathcal{N}_{n-1}(0, P^TP) = \mathcal{N}_{n-1}(0, I_{n-1})$ based on the transformed response $\cY$ and features $\cW$. To compute the posterior for $\beta_{-1}$, we thus only need a preprocessing step consisting of premultiplying the data by $P^T$ before running Bayesian linear regression as usual.
 \begin{lemma}[Posterior distribution when $k=1$]\label{lem:posterior_form}
 	Let $\Pi$ be the prior \eqref{general_prior}. Then under the posterior distribution, $\beta_{-1}$ and $\beta_1^*$ are independent, and their distributions take the form 
 	\begin{align}
 		d\pi(\beta_{-1} | Y) & \propto {e^{-\frac{1}{2} \| \cY -\cW\beta_{-1}\|^2_2}} dMS_{p-1}(\nu, \lambda), \label{eq:density_beta_minus_1_1D} \\
 		d\pi(\beta_{1}^* |Y) & \propto e^{-\frac{1}{2}\|X_1\|_2^2\left( \beta_{1}^* - \frac{X_1^TY}{\|X_1\|_2^2} \right) ^2} g(\beta_1^* ) d\beta_1^*. \label{density_beta_star_1D}
 	\end{align} 
 	In particular, in Examples \ref{gaussian_prior} and \ref{improper_prior},  the density \eqref{density_beta_star_1D} takes the form
  $$\beta_1^* | Y \sim \mathcal{N}\left(\frac{\sigma_n^2}{\|X_1\|_2^2\sigma_n^2 +1}X_1^TY, \frac{\sigma_n^2}{\|X_1\|_2^2\sigma_n^2 +1} \right) \qquad \text{and} \qquad \beta_1^* | Y  \sim \mathcal{N}\left(\frac{1}{\|X_1\|_2^2}X_1^TY, \frac{1}{\|X_1\|_2^2} \right),$$
  respectively. 
  Finally, under the frequentist assumption $Y \sim \mathcal{N}_n(X\beta^0,I_n)$, we have $\cY \sim\mathcal{N}_{n-1}(\cW\beta^0_{-1} , I_{n-1})$.
 \end{lemma}
 Lemma \ref{lem:posterior_form} shows that the posterior computations of $\beta_{-1}$ and $\beta_1^*$ decouple, with that of $\beta_{-1}$ following the linear regression model $\cY = \cW \beta^0_{-1} + \check{\eps}$ with model selection prior $MS_{p-1}(\nu, \lambda)$, while $\beta_{1}^*$ follows from a one-dimensional regression model with prior $g$. The first could in principle be tackled using computational tools for sparse priors, such as modern MCMC methods \citep{Martin2017,Griffin2021}, while the second is a low-dimensional problem that can be solved for instance analytically or by MCMC.
Lemma \ref{lem:posterior_form} thus provides an efficient way to sample from the posterior of $\beta_1$: (i) sample $\beta_{-1}$ according to \eqref{eq:density_beta_minus_1_1D}, (ii) sample $\beta_1^*$ according to \eqref{density_beta_star_1D}, (iii) compute $\beta_1 = \beta_1^* -\sum_{i = 2}^p\gamma_i\beta_i$.

However, the discrete nature of the model selection prior can make step $(i)$ very expensive computationally, requiring $O(2^{p})$ integrations for each of the possible active sets of $\beta_{-1}$. The discrete and multimodal nature of the posterior also makes posterior sampling challenging in high-dimensions, with MCMC methods known to have problems mixing for typical problem sizes of interest \citep{Griffin2021}. This motivates us to use a variational approximation to the posterior distribution, which provides a computationally scalable alternative. 
 
\subsection{Variational approximation and sampling scheme for $k=1$}\label{sec:variational_approx_1d}
To avoid the computational difficulty of posterior sampling for $\beta_{-1}$, we instead consider a scalable mean-field variational Bayes (MF VB) approximation, under which the model parameters are independent and each follow a spike and slab distribution. Consider the following class of distributions on $\R^{p-k}$:
\begin{align}\label{eq:variational_class}
	\mathcal{Q}_{-k} = \left\{Q_{\mu, \tau, q} =  \bigotimes_{i=k+1}^p \, q_i \, \mathcal{N}(\mu_i, \tau_i^2) + (1-q_i)\delta_0   : q_i \in [0,1], \, \mu_i \in \mathbb{R}, \, \tau_i \in \mathbb{R}^+ \right\},
\end{align}
and define the corresponding variational approximation to $\beta_{-1}$ within the class $\mathcal{Q}_{-1}$ by
\begin{equation}\label{eq:optimization_k=1}
\hat{Q}_{-1}= \argmin_{Q_{-1} \in \mathcal{Q}_{-1}} \operatorname{KL}(Q_{-1} || \Pi_{-1}(\cdot |Y)),
\end{equation}
where $\Pi_{-1}(\cdot | Y)$ is the marginal distribution \eqref{eq:density_beta_minus_1_1D} of $\beta_{-1}$ under the posterior (see below for the general case $k>1$). Such sparsity-inducing MF variational families have been used in a variety of settings \citep{logsdon2010,carbonetto2012,ormerod2017,Ray_Szabo_2020,RaySzaboClara2020,komodromos2022} and the optimization problem \eqref{eq:optimization_k=1} can be solved using coordinate ascent variational inference (CAVI), see \cite{Ray_Szabo_2020}. In particular, \eqref{eq:density_beta_minus_1_1D} in Lemma \ref{lem:posterior_form} shows that the optimization \eqref{eq:optimization_k=1} reduces to the problem of finding the closest element of $\mathcal{Q}_{-1}$ to the posterior (on $\R^{p-1}$) coming from the model selection prior $MS_{p-1}(\nu, \lambda)$ and likelihood from the transformed linear regression model $\cY = \cW \beta^0_{-1} + \check{\eps}$ defined in \eqref{eq:preprocess}. One can thus use the standard CAVI implementation derived in \cite{Ray_Szabo_2020} for linear regression by simply replacing the inputs $(X,Y)$ with $(\cW,\cY)$. In particular, this allows one to directly use any existing software implementation (e.g. \cite{sparsevb}) for this problem without modification after the straightforward preprocessing step \eqref{eq:preprocess}.

We  complete the posterior approximation by using the \textit{exact} posterior distribution for $\beta_1^*| Y$, given in \eqref{density_beta_star_1D}, yielding the overall VB posterior:
\begin{equation}\label{intuitive_approximation}
	\begin{split}
		\beta_{-1} \sim \hat{Q}_{-1}, \qquad  \beta_{1}^* &\sim \pi( \beta_1^*| Y), \qquad \beta_{-1}, \beta_{1}^* \textrm{  independent},\\
            \beta = ( \beta_1, \beta_{-1} ) &= (\beta_{1}^* - \sum_{i = 2}^p \gamma_i \beta_i, \beta_{-1}).
	\end{split} 
\end{equation}
We denote by $\hat{Q}$ the distribution for $\beta$ given by \eqref{intuitive_approximation}, and $\hat{Q}_1$ the marginal distribution of the first coordinate $\beta_1$. Once one has computed $\hat{Q}_{-1}$, one can sample $\beta_1 \sim \hat{Q}_1$ in three steps:
\begin{itemize}
    \item[(i)] sample $\beta_{-1}\sim \hat{Q}_{-1}$,
    \item[(ii)] sample $\beta_1^*\sim\pi(\beta_1^* | Y)$ according to \eqref{density_beta_star_1D}, independently of $\beta_{-1}$,
    \item[(iii)] compute $\beta_1 = \beta_1^* - \sum_{i\geq 2} \gamma_i \beta_i$.
\end{itemize}
The full algorithm for fitting this method and sampling $\beta_1$ is given in Algorithm \ref{alg:sample}. A discussion of its computational cost can be found in Section \ref{sec:computational_cost} of the supplement.

\begin{algorithm}
\caption{Generating $n_s$ samples of $\beta_1$ from the VB posterior $\hat{Q}_1$}
\label{alg:sample}
\begin{algorithmic}
\Require{$\mathcal{D}_n = (X, Y), \Pi = (MS_{p-1}
(\nu, \lambda), g), n_s$.}
    \State 
    Compute $\{u_1, \dots, u_{n-1}\}$, an orthonormal basis of $\textrm{span}(X_1)^\perp$ \\
    Define $P = (u_1, \dots, u_{n-1}) \in \R^{n \times (n-1)}$ with columns $u_1, \dots, u_{n-1}$
 \State 
 Compute $\cY = P^TY$, $\cW = P^TX_{-1}$, $\gamma_i = X_1^TX_i/\|X_1\|_2^2$ for $i = 2, \dots, p$.
\State Fit $\hat{Q}_{-1}$ in \eqref{eq:optimization_k=1} by performing CAVI based on $(\cW,\cY)$ and prior $MS_{p-1}(\nu, \lambda)$ [e.g. as in \cite{Ray_Szabo_2020}]
\For{$j = 1, \dots, n_s$}
\State Sample $\beta_{-1}^j \sim \hat{Q}_{-1}$
\State Sample $(\beta_1^*)^{j}\sim \pi(\beta_1^* | Y)$ according to \eqref{density_beta_star_1D}, independently of  $\beta_{-1}^j$
\State Compute $\beta_1^j = (\beta_1^*)^{j} - \sum_{i\geq 2} \gamma_i \beta_i^j$
\EndFor
\State
\Return{$\{\beta^j_1\}_{j = 1}^{n_s}$}
\end{algorithmic}
\end{algorithm}

The idea behind our method \eqref{intuitive_approximation} is to combine a fast MF approximation for $\beta_{-1}$ with the exact posterior distribution for $\beta_1^*$, which can be computed using low-dimensional Bayesian computational techniques. Since the posteriors for $\beta_1^*$ and $\beta_{-1}$ decorrelate, and are in fact independent for this specific prior choice by Lemma \ref{lem:posterior_form} above, not only does computation decouple but we may use the exact marginal posterior distribution \eqref{density_beta_star_1D} for $\beta_1^*$ (though not $\beta_1$). 
We then undo the transformation to obtain our parameter of interest $\beta_1 = \beta_{1}^* - \sum_{i = 2}^p \gamma_i \beta_i$, correctly capturing the posterior variance component of $\beta_{1}$ coming from $\beta_{1}^*$. Even though the MF approximation on $\beta_{-1}$ can (and does) underestimate the posterior variance of the second component, this is typically of smaller order leading to a significant improvement over full MF VB.
Indeed, using full MF VB for $(\beta_1,\beta_{-1})$ underestimates the posterior variance for $\beta_1$ leading to poor uncertainty quantification in correlated settings, see Section \ref{sect:simulations} below.

\subsection{The general case $k \geq 1$} \label{sec:dimk}
This methodology extends naturally to the $k$-dimensional subparameter $\beta_{1:k} = (\beta_1,\dots,\beta_k)^T \in \R^k$ at the cost of slightly more complicated notation. We assume that the columns $X_1, \dots, X_k$ are linearly independent, else the problem is not identifiable. Denoting $\Sigma_k := (X_{1:k}^T X_{1:k})^{-1}$, define the projection matrix $H_k = X_{1:k}\Sigma_k X_{1:k}^T$ onto $\textrm{span}(X_1, \dots, X_k)$, and a matrix $P_k \in \R^{n \times (n-k)}$, whose columns consist of an orthonormal basis of $\textrm{span}(X_1, \dots, X_k)^\perp$. The likelihood now decomposes as
\begin{equation*}  
    \mathcal{L}_n(\beta, Y) \propto \exp\left\{-\frac{1}{2}\left\|H_kY - X_{1:k} \beta_{1:k}^*\right\|_2^2 \right\} \exp\left\{-\frac{1}{2}\|\cY_k - \cW_k \beta_{-k}\|_2^2 \right\}, \label{eq:likelihood_decomp_kD}
\end{equation*}
where $\beta^*_{1:k} := \beta_{1:k} + \Sigma_k X_{1:k}^T X_{-k} \beta_{-k}$, $\cY_k = P_k^T Y\in \R^{n-k}$ and $\cW_k = P_k^T X_{-k}\in \R^{(n-k)\times (n-k)}$.
 
For $\nu$ a probability distribution on $\{0, \dots, p-k\}$ , $\lambda>0$ and $g$ a positive density with respect to the Lebesgue measure on $\R^k$, consider the following prior distribution for $\beta \in \R^p$:
 \begin{equation}\label{general_prior_kd}
 	\begin{aligned}
 		\beta_{-k} &\sim MS_{p-k}(\nu, \lambda)\\
 		\beta_{1:k} | \beta_{-k} &\sim g(\cdot + \Sigma_k X_{1:k}^T X_{-k} \beta_{-k}) d \Lambda.
 	\end{aligned}
 \end{equation}
We consider both the Laplace prior $g(x_1, \dots, x_k) \propto \exp \left\{-\frac{1}{\sigma_n}\sum_{i = 1}^k |x_i| \right\}$ and the improper prior $g\propto 1$, but not the multivariate Gaussian distribution for $g$ since it already performs less well in both theory and practice when $k = 1$. As in Lemma \ref{lem:posterior_form}, the prior \eqref{general_prior_kd} gives rise to a posterior under which $\beta_{1:k}^*$ and $\beta_{-k}$ are independent with distributions of the form 
\begin{align}
    d\pi(\beta_{-k} | Y) &\propto {e^{-\frac{1}{2} \| \cW_k\beta_{-k} - \cY_k \|^2_2}} dMS_{p-k}(\nu, \lambda)(\beta_{-k}), \label{eq:beta_minus_k_posterior} \\
    d\pi(\beta_{1:k}^* |Y) &\propto e^{-\frac{1}{2} (\beta_{1:k}^* - \Sigma_k X_{1:k}^TY )^T \Sigma_k^{-1} (\beta_{1:k}^* - \Sigma_k X_{1:k}^TY)} g(\beta_{1:k}^*)d\Lambda(\beta_{1:k}^*). \label{eq:beta_k_star_posterior}
\end{align}
As above, we approximate the distribution for $\beta_{-k}$ with the MF family $\mathcal{Q}_{-k}$ given in \eqref{eq:variational_class}:
$$\hat{Q}_{-k}= \argmin_{Q_{-k} \in \mathcal{Q}_{-k}} \operatorname{KL}(Q_{-k} || \Pi_{-k}(\cdot |Y)),$$
for $\Pi_{-k}(\cdot | Y)$ the posterior distribution of $\beta_{-k}$ defined by \eqref{eq:beta_minus_k_posterior}. We then consider the full variational approximation  for $\beta \in \R^p$:
\begin{equation*}\label{intuitive_approximation_k}
	\begin{aligned}
		\beta_{-k} \sim \hat{Q}_{-k}, \qquad \beta_{1:k}^* &\sim \pi(\beta_{1:k}^* | Y), \qquad \beta_{-k}, \beta_{1:k}^* \textrm{  independent},\\
           ( \beta_{1:k}, \beta_{-k} ) &= (\beta_{1:k}^* - \Sigma_k X_{1:k}^T X_{-k} \beta_{-k}, \beta_{-k}),
	\end{aligned} 
\end{equation*}
where $\pi(\beta_{1:k}^* | Y)$ is defined in \eqref{eq:beta_k_star_posterior}. In the following, we denote by $\hat{Q}$ the distribution on $\beta$ given by the last display and $\hat{Q}_{1:k}$ the marginal distribution of $\beta_{1:k}$. Computation of $\hat{Q}_{-k}$ and sampling from $\hat{Q}_{1:k}$ from this variational approximation follows analogously to the discussion for the one dimensional case $k=1$.

\section{Numerical studies}\label{sect:simulations}
We assess the numerical performance of our proposed VB method with the improper prior choice $g \propto 1$ (Example \ref{improper_prior}), which we refer to as I-SVB (I for improper), comparing its performance with $(i)$ two standard frequentist methods from \cite{Zhang2014} and \cite{Javanmard2014} (referred to as ZZ and JM, respectively), and $(ii)$ full mean-field variational inference with a spike-and-slab variational class \citep{Ray_Szabo_2020} (referred to as MF). Note that the frequentist methods deal with the more difficult problem of simultaneous inference of all $1$--dimensional coordinates, but we nonetheless include them as comparators since they are standard methods in the literature.

We found that the Laplace prior for $g$ performs similarly, but we focus on the improper case for simplicity, since this gives a particularly simple expression for the posterior of $\beta_1^* | Y$ (see Lemma \ref{lem:posterior_form}) and does not require hyperparameter tuning.
While one can achieve similar results using a Gaussian prior with large variance, the results are sensitive to this hyperparameter choice and are less robust. For a detailed comparison of the different prior choices, see Section \ref{sec:g_prior}. All of the code to produce these results can be found in our GitHub repo 
 at \url{https://github.com/lukemmtravis/Debiased-SVB/}. 
We test our method with random design matrices $X$ with i.i.d. rows $X_{i\cdot}\sim^{iid} \mathcal{N}_p(0,\Sigma)$ where, for some $\rho\in[0,1)$:
\begin{enumerate}
\item[(a)] $\Sigma=\Sigma_\rho$ the matrix with $(\Sigma_\rho)_{jk}=\rho$ if $j\neq k$ and $1$ otherwise;
\item[(b)] $\Sigma=\Sigma_\rho^{AR}$  the autoregressive matrix with $(\Sigma_\rho^{AR})_{jk}=\rho^{|j-k|}$ for $1\leq j , k\leq p$.
\end{enumerate}
To illustrate inference for multi-dimensional parameters, in Section \ref{sec:sims2} we focus on $\Sigma=\Sigma_\rho^{AR}$ for which `oracle' confidence sets are not parallel to the coordinate axes, and where improvement over standard mean-field VB is therefore particularly marked. Further simulations and a real covariate data example can be found in Section \ref{sec:additional_simulations} of the supplement.

\subsection{One-dimensional inference $k=1$} \label{sec:sims1}
We present here 6 different simulation settings, with additional simulations in Section \ref{sec:additional_simulations}. In each case, we generate the data from a linear regression model $Y = X\beta^0 + \sigma\eps$ as in \eqref{eq:linear_regression}. We estimate the noise level $\sigma^2$ with the simple estimator $\hat{\sigma}^2 = \|Y - X\hat{\beta}\|_2^2/(n - \hat{s} - 1)$, with $\hat{\beta}$ the cross-validated LASSO and $\hat{s}$ the sparsity of the estimator $\hat{\beta}$, and consider the rescaled data $\tilde{Y} = Y/\hat{\sigma}$ and $\tilde{X} = X/\hat{\sigma}$ as described at the start of Section \ref{sec:methodology}. Note that any suitable method for estimating $\sigma^2$ can be used instead.
 We then implement our VB method according to Algorithm \ref{alg:sample}, where the function \texttt{svb.fit} from the package \texttt{sparsevb} \cite{sparsevb} is used to fit $\hat{Q}_{-1}$ using CAVI after the preprocessing step \eqref{eq:preprocess}, and where we compute $\pi(\beta_1^*|Y)$ by conjugacy since $g \propto 1$.

For the full mean-field (MF) VB approach we approximate the posterior arising from placing the model selection $MS_p$ prior on $(\beta_1,\dots,\beta_p)^T$ by the variational family $\hat{Q}_{-0}$ given by \eqref{eq:variational_class}, again using the CAVI implementation in the \texttt{sparsevb} R-package \citep{sparsevb}. We implement the ZZ method by first performing cross-validated LASSO regression of $Y$ on $X$ with the \texttt{glmnet} package to obtain an initial estimator $\hat{\beta}^{(init)}$, then regress $X_1 = X_{-1}\gamma_{-1} + \eta$ to obtain $\hat{\gamma}_{-1}$ and the residual $z_1 = X_1 - X_{-1}\hat{\gamma}_{-1}$, and finally perform the bias correction by computing $\hat{\beta}_1 = \hat{\beta}^{(init)}_1 + z_1^T(Y - X\hat{\beta}^{(init)})/z_1^T x_1$. The JM method is implemented with the code attached to the supplementary material of \cite{Javanmard2014}, but adapted so that it only computes an estimate and confidence interval for one coordinate at a time (though recall that their method is designed for simultaneous inference of all coordinates).

We parameterise each scenario by the tuple $(n, p, s_0, \beta_1^0, \beta_j^0, \rho, \sigma^2)$, where $n$ is the number of observations, $p$ is the covariate dimension, $s_0$ is the true sparsity of $\beta^0$, $\beta_1^0$ is the first coordinate of $\beta^0$ (our target for inference), $\beta_0^j$ are the values of the other non-zero entries of $\beta^0$ and $\sigma^2$ is the unknown noise variance. In Table \ref{Tab:experiments_1D}, we consider case (a) where $\Sigma_{jk} = 1$ for $j=k$ and $\Sigma_{jk} = \rho$ for $j\neq k$. 
For each scenario, we simulate 500 datasets and compute $95\%-$credible or confidence intervals for each method. For the variational methods, we take 1000 VB posterior samples and compute credible intervals based on the empirical quantiles. Confidence intervals for the frequentist methods are computed as specified in their respective papers. For each method, we report the $(i)$ coverage (proportion of the intervals containing the truth); $(ii)$ mean absolute error (MAE) of the centering of the intervals; $(iii)$ interval length; and $(iv)$ mean computation time in seconds. Results are found in Table \ref{Tab:experiments_1D}, including standard deviations where relevant.

We see that I-SVB and MF perform very similarly when there is no correlation between features (scenario $(i)$), with both methods delivering approximately 95\% coverage and smaller MAE than the other methods. Furthermore, both methods deliver this coverage while giving intervals which are smaller than their frequentist counterparts.

However, when we add correlation to the features (scenarios $(ii)$-$(vi)$), the behaviour of I-SVB diverges from that of MF, with I-SVB generally maintaining higher coverage than the other methods in these scenarios. The coverage of I-SVB is sometimes higher than the intended 95\% level, 
but the I-SVB credible interval lengths are generally comparable or not much larger than ZZ, while the I-SVB centering has smaller bias and MAE. 
As expected, full MF ignores the correlation and thus provides much smaller credible sets in correlated settings, underestimating the posterior variance and leading to under-coverage. Indeed, the MF credible sets are the same length when the experiments are performed with zero correlation (see the additional simulations in Section \ref{sec:additional_simulations}). Aside from MF, all other methods adapt to the added difficulty of adding correlation by increasing their interval lengths.


Credible sets from both VB methods can be computed in comparable time to ZZ, and much faster than JM, which requires a relatively expensive optimization routine to estimate $\Sigma^{-1}$ that can, however, be used for simultaneous inference of coordinates.
 While we adapted the code for JM to compute confidence intervals for only a single coordinate to speed up computation, its implementation could likely be sped up by using a faster language than \texttt{R}. 

We provide additional scenarios in Section \ref{sec:additional_simulations} in the appendix, which match the above conclusions. These results already demonstrate for $k=1$ how I-SVB improves upon naive MF VB by providing reliable statistical inference while keeping the computational scalability of MF VB. The overall performance of I-SVB is also competitive with its frequentist counterparts.

\begin{table}[!htbp]
\centering
\footnotesize
\resizebox{\textwidth}{!}{%
\begin{tabular}{r|r|cccc}
\toprule
Scenario $(n, p, s_0, \beta_1^0, \beta_j^0, \rho, \sigma^2)$      & Method & Cov.  & MAE                        & Length                     & Time\\
\hline
$(i)$                                                             & I-SVB  & 0.952 & \textbf{0.082 $\pm$ 0.059} & 0.403 $\pm$ 0.034          & 0.270 $\pm$ 0.099\\
$(100, 1000, 3, \log n, \log n, 0, 1)$                            & MF     & 0.954 & \textbf{0.082 $\pm$ 0.059} & \textbf{0.396 $\pm$ 0.028} & 0.162 $\pm$ 0.069\\
                                                                  & ZZ     & 0.886 & 0.112 $\pm$ 0.081          & 0.473 $\pm$ 0.223          & 0.271 $\pm$ 0.100\\
                                                                  & JM     & 0.976 & 0.163 ± 0.103              & 0.850 ± 0.077              & 1.899 ± 0.252\\
\hline
$(ii)$                                                          & I-SVB  & 0.940 & \textbf{0.437 ± 0.343}     & \textbf{2.241 ± 0.329}              & 0.391 ± 0.132\\
$(100, 1000, 3, \log n, \log n, 0.5, 16)$                         & MF     & 0.710 & 0.516 ± 0.512              & 1.316 ± 0.299     & 0.320 ± 0.187\\
                                                                  & ZZ     & 0.836 & 0.647 ± 0.473              & 2.824 ± 2.277              & 0.399 ± 0.128\\
                                                                  & JM     & 0.834 & 0.926 ± 0.602              & 3.066 ± 0.731              & 1.480 ± 0.249\\
\hline
$(iii)$                                                          & I-SVB  & 0.992 & \textbf{0.051 $\pm$ 0.038} & \textbf{0.331 $\pm$ 0.024}          & 1.384 $\pm$ 0.238\\
$(400, 1500, 32, \mathcal{N}(0, 1), \mathcal{N}(0, 1),$  & MF     & 0.752 & 0.072 $\pm$ 0.064          & 0.196 $\pm$ 0.007 & 0.913 $\pm$ 0.159\\
$0.25, 1)$                                                  & ZZ     & 0.846 & 0.069 $\pm$ 0.051          & 0.254 $\pm$ 0.017          & 1.396 $\pm$ 0.232\\
                                                                  & JM     & 0.636 & 0.289 ± 0.217              & 0.655 ± 0.13               & 33.673 ± 15.227\\
\hline
$(iv)$                                                           & I-SVB  & 1.000 & \textbf{0.182 $\pm$ 0.131} & 1.872 $\pm$ 0.089          & 0.707 $\pm$ 0.159\\
$(200, 800, 10, \log n, \log n, 0.9, 1)$                          & MF     & 0.014 & 3.629 $\pm$ 2.252          & 0.276 $\pm$ 0.015 & 1.062 $\pm$ 0.356\\
                                                                  & ZZ     & 0.936 & 0.217 $\pm$ 0.171          & \textbf{1.055 $\pm$ 0.082}          & 0.626 $\pm$ 0.128\\
                                                                  & JM     & 0.260 & 1.693 ± 1.337              & 1.441 ± 0.123              & 9.927 ± 0.697\\
\hline
$(v)$  & ISVB & 0.96 & 0.094 ± 0.074 & 0.482 ± 0.046 & 0.378 ± 0.053\\
$(200,1000,5,0,0.5,0.5,1)$ & MF & 1.00 & \textbf{0.000 ± 0.000} & \textbf{0.272 ± 0.013} & 0.289 ± 0.054\\
 & ZZ & 1.00 & 0.072 ± 0.053 & 0.476 ± 0.057 & 0.468 ± 0.061\\
 & JM & 1.00 & 0.031 ± 0.025 & 0.668 ± 0.059 & 3.693 ± 1.120\\
\hline
$(vi)$ & ISVB & 0.994 & \textbf{0.050 ± 0.038} & 0.340 ± 0.027 & 1.214 ± 0.144\\
$(500,1000,10, \mathcal{U}(-1, 1),$ & MF & 0.708 & 0.077 ± 0.071 & 0.175 ± 0.006 & 0.685 ± 0.153\\
$ \mathcal{U}(-1, 1),0.5,1)$ & ZZ & 0.958 & 0.051 ± 0.040 & \textbf{0.273 ± 0.011} & 1.292 ± 0.157\\
 & JM & 0.672 & 0.238 ± 0.155 & 0.622 ± 0.203 & 38.844 ± 17.190\\
\bottomrule
\end{tabular}
}
\caption{Estimation and uncertainty quantification performance for $\beta_1$ using 4 methods. Highlighted in bold are the smallest interval length, subject to coverage being larger than 0.85, and the smallest MAE.}
\label{Tab:experiments_1D}
\end{table}

\subsection{Multi-dimensional inference $k \geq 1$} \label{sec:sims2}

We now consider inference for $\beta_{1:k}$ for general $k\geq 1$. We generate $n_s$ VB posterior samples $\{\beta_{1:k}^i\}_{i = 1}^{n_s}$ of $\beta_{1:k}$, using which we estimate the posterior mean and covariance by
\begin{align*}
    \bar{\beta}_{1:k} := \frac{1}{n_s} \sum_{i = 1}^{n_s} \beta_{1:k}^i, \qquad 
    \hat{\Theta}_{1:k} := \frac{1}{n_s} \sum_{i = 1}^{n_s} (\beta_{1:k}^i - \bar{\beta}_{1:k})(\beta_{1:k}^i - \bar{\beta}_{1:k})^T.
\end{align*}
We then compute an approximate $\gamma-$credible set
\begin{equation}\label{eq:cred_set_multi_d}
C_\gamma := \left\{v \in \mathbb{R}^k : (v_k - \bar{\beta}_{1:k})^T\hat{\Theta}_{1:k}^{-1}(v_k - \bar{\beta}_{1:k}) \leq \chi_{\gamma,k}\right\},    
\end{equation}
where $\chi_{\gamma, k}$ is defined by $P(\chi_k^2 \leq \chi_{\gamma, k}) = \gamma$ for $\chi_k^2$ a chi-squared distribution with $k$ degrees of freedom. The method of \cite{Javanmard2014} also extends to the multi-dimensional setting $k\geq 1$ and results in a set of the form \eqref{eq:cred_set_multi_d} above, but with $\bar{\beta}_{1:k}$ replaced by their estimator and $\hat{\Theta}_{1:k}$ replaced by $\frac{\hat{\sigma^2}}{n} (M\hat{\Sigma} M^T)_{1:k, 1:k}$, with $M$ the relevant estimate of $\Sigma^{-1}$. Approximate credible sets of the form \eqref{eq:cred_set_multi_d} can be motivated by the Bernstein--on Mises in Theorem \ref{thm:asymptotic_normality_variational_kD}.

\textbf{Two-dimensional visualization.} To help understand the behaviour of the three methods,
Figure \ref{fig:credible_regions_example} plots a realisation of the credible/confidence regions and estimates in three $2$-dimensional examples. We take $n = 200$, $p = 400$, and true parameter $\beta^0$ with sparsity $s_0 = 10$ and non-zero coordinates $\beta_i^0 = 5$. We set the first two coordinates of $\beta^0$ to be non-zero, and distribute the remaining 8 non-zero coordinates uniformly throughout $\beta^0$. The design has i.i.d. rows $X_{i\cdot} \sim^{iid} \mathcal{N}_p(0, \Sigma_\rho^{AR})$ with autoregressive covariance matrix $[\Sigma_\rho^{AR}]_{ij} = \rho^{|i-j|}$ with $\rho = 0,0.5,0.95$. 
For comparison, we also plot {\it Oracle} confidence sets corresponding to the low-dimensional linear regression model $Y = X_{S_0} \beta_{S_0} + \eps$, where $S_0$ is the (unknown) support of the true $\beta^0$. The oracle estimate of the full parameter $\beta$ is given by $\hat{\beta}^o_{S_0} := (X_{S_0}^T X_{S_0})^{-1} X_{S_0}^T Y$ on $S_0$, and a point estimate $\hat{\beta}^o_{S_0^C} = (0,\dots, 0)$ outside $S_0$. On $S_0$, the oracle covariance matrix is $\Theta_{{S_0},{S_0}}^o = (X_{S_0}^T X_{S_0})^{-1}$, while outside of $S_0$ the covariance matrix is 0.
The oracle sets are then of the form \eqref{eq:cred_set_multi_d} centered at $\hat{\beta}^o_{1:k}$ and with covariance matrix $\Sigma^o_{1:k, 1:k}$. Note that this is \textit{not} a valid statistical method since the oracle can not be computed in practice as $S_0$ is unknown, but it provides a useful benchmark since one essentially cannot do better than this.

In the first scenario with no correlation (left, $\rho = 0$), all the sets have a similar circular shape and contain the truth (pink). MF and I-SVB provide very similar estimates and sets as expected, since there is no underlying correlation. There is also little difference between these credible sets and the Oracle confidence region, suggesting they are performing very well in this scenario. JM produces a slightly larger confidence set due to the presence of an inflation factor in the application of their method \citep{Javanmard2014}. 

In correlated settings, we see the well-known issue \citep{Bishop2006,Blei2017} that MF VB both underestimates the marginal variances and fails to capture the correlation, leading to credible sets that are smaller and poorly oriented compared to the Oracle. In the second scenario (middle, $\rho=0.5$) the set still covers the truth, but in the third scenario (right, $\rho = 0.95$) it does a poor job of quantifying the uncertainty. In these scenarios, JM increases the size of the confidence sets to account for the added difficulty of the correlation, but it does not seem to pick up the correct covariance structure, leading to much larger confidence regions (though at least covering the truth). In contrast, the I-SVB credible sets seem to capture the correlation structure much more effectively, mirroring the oracle, resulting in sets that are smaller while also covering the truth. Figure \ref{fig:credible_regions_example} visually demonstrates that unlike MF VB, our method I-SVB can correctly pick up the posterior correlation structure of $\beta_{1:k}$ (it is close to the oracle) and provide good estimates of the posterior uncertainty, leading to much improved uncertainty quantification in correlated settings. In Section \ref{sec:heuristics}, we provide heuristics to explain why I-SVB performs similarly to the oracle in this strongly correlated AR scenario, having similar coverage and shape to the Oracle confidence set.

\begin{figure}
    \centering
    \includegraphics[width=0.85\linewidth]{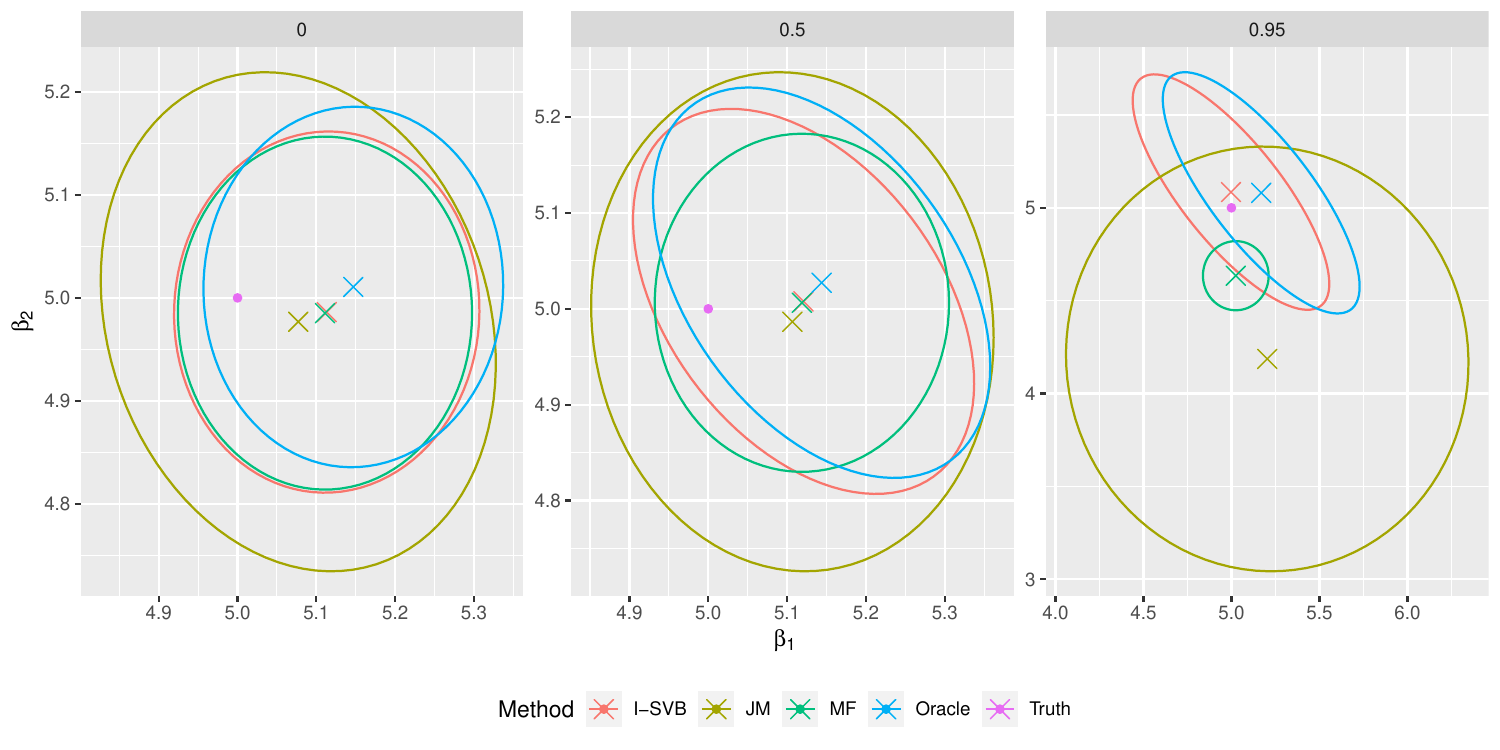}
    \caption{95\%-credible or confidence regions in three 2D scenarios with design matrix rows $X_{i\cdot}\sim^{iid} \mathcal{N}_p(0, \Sigma_\rho^{AR})$ for increasing values of $\rho$ given in the title facets. The ellipse interiors represent the credible or confidence regions, the crosses their centerings and the pink point is the true $(\beta_1, \beta_2)$. Each scenario has $n=200, p=400, k=2, s_0=10$ and $\beta^0_i=5$.}
    \label{fig:credible_regions_example}
\end{figure}

{\bf Numerical simulations for $k\geq 2$.} We now compare these methods more systemically, still with $\Sigma=\Sigma_\rho^{AR}$. We measure the size of the sets via their $k-$dimensional volume, which is proportional to $\sqrt{\prod_{j = 1}^k \lambda_j}$  with $\{\lambda_i\}_{i=1}^k$ the eigenvalues of the covariance matrix defining the regions (e.g. $\hat{\Theta}_{1:k}$ in \eqref{eq:cred_set_multi_d}). To facilitate comparison, we report the relative volume of each method compared to the Oracle when this volume is non-zero, so the Oracle has relative volume 1. Where the true low-parameter contains a 0, the Oracle set is a lower dimensional subspace of $\R^{k}$, and so has volume 0; in these cases we report the relative volume of each method to I-SVB. We also report the average $L_2-$error of the estimators (i.e. region centers).

Each scenario is parametrized by the tuple $(n, p, k, s_0, \beta^0_j, \rho)$, where all parameters apart from $k$ have the same meaning as before, but we are now interested in $k-$dimensional credible/confidence regions for $\beta_{1:k}$. We also specify the true value of $\beta_{1:k}^0$, and distribute the remaining non-zero coefficients uniformly in the nuisance parameter $\beta_{-k}^0$, where the values of the non-zero coordinates equal $\beta_j^0$.
We again simulate 500 realisations of a dataset and compute the credible or confidence sets for each method. Table \ref{Tab:results_multi_d} contains the results.

With no correlation (scenario $(i)$), MF and I-SVB perform similarly to the Oracle, while JM likewise achieves good coverage, but with a much larger relative volume to compensate for a larger $L_2-$error in the estimate. The JM method also takes longer to run due to the expensive estimation of $\Sigma^{-1}$. 

However, once we add feature correlations (scenarios $(ii) -(vi)$), the behaviour of the non-Oracle methods start to differ, similar to what we observe in Figure \ref{fig:credible_regions_example}. I-SVB maintains the target (or slightly conservative) coverage, while the coverage of MF and JM drops significantly. Remarkably, the volume of the I-SVB sets is at most 50\% larger than the Oracle, and in the higher dimensional scenarios with larger $n$ only about 30\% larger. The JM method struggles due to a larger bias in its centering, compensating this with a larger relative volume, though this does not always correct the coverage. For full MF, the credible regions are too small with their relative volume plummeting as we increase $k$ and $\rho$, leading to significant underestimates of the uncertainty. We also see that in scenarios $(iv)$ and $(vi)$, where some coordinates of $\beta_{1:k}^0$ are zero, I-SVB maintains good coverage while providing a volume which is significantly smaller than JM, suggesting that I-SVB is concentrating on a lower dimensional subspace of $\R^k$ as the Oracle does. 
In summary, I-SVB appears to provide good estimation and (possibly conservative) coverage for $\beta_{1:k}$, while also accurately capturing the underlying correlation structure, leading to smaller and more informative credible sets.

\begin{table}[!htbp]
\centering
\footnotesize
\resizebox{\textwidth}{!}{%
\begin{tabular}{r|r|cccc}
\toprule
Scenario   $(n, p, k, s_0, \beta^0_j, \rho)$  & Method & Cov.  & $L_2-$error            & Rel. Volume            & Time\\
\hline
$(i)$                                         & I-SVB  & 0.960 & \textbf{0.091 ± 0.048} & 1.007 ± 0.081          & 0.301 ± 0.107\\
$(200, 400, 2, 10, \log n, 0)$                & MF     & 0.946 & \textbf{0.091 ± 0.048} & \textbf{0.951 ± 0.071} & 0.162 ± 0.080\\
$\beta^0_{1:2} = (\log n, \log n)$            & JM     & 0.952 & 0.112 ± 0.065          & 1.843 ± 0.314          & 1.012 ± 0.163\\
                                              & Oracle & 0.948 & 0.096 ± 0.065          & 1.000 ± 0.073          & - \\
\hline
$(ii)$                                        & I-SVB  & 0.966 & 0.127 ± 0.065          & \textbf{1.517 ± 0.118} & 0.206 ± 0.040\\
$(200, 400, 2, 10, \log n, 0.5)$              & MF     & 0.790 & \textbf{0.125 ± 0.064} & 0.532 ± 0.043          & 0.309 ± 0.071\\
$\beta^0_{1:2} = (\log n, \log n)$            & JM     & 0.742 & 0.342 ± 0.527          & 2.975 ± 0.402          & 6.526 ± 0.902\\
                                              & Oracle & 0.950 & 0.129 ± 0.068          & 1.000 ± 0.069          & -\\
\hline
$(iii)$                                       & I-SVB  & 0.976 & 0.130 ± 0.046          & \textbf{1.397 ± 0.122} & 0.651 ± 0.106\\
$(400, 800, 4, 10, \log n, 0.5)$              & MF     & 0.734 & \textbf{0.129 ± 0.046} & 0.305 ± 0.029          & 1.015 ± 1.652\\
$\beta^0_{1:4} = (\log n, \dots, \log n)$     & JM     & 0.430 & 0.308 ± 0.328          & 4.194 ± 0.979          & 32.808 ± 7.427\\
                                              & Oracle & 0.972 & 0.123 ± 0.047          & 1.000 ± 0.080          & -\\
\hline
$(iv)$                                        & I-SVB  & 0.970 & \textbf{0.110 ± 0.040} & \textbf{1.000 ± 0.091} & 0.982 ± 0.119\\
$(400, 800, 4, 10, \log n, 0.25)$             & MF     & 0.842 & 0.342 ± 1.115          & 0.394 ± 0.028          & 2.244 ± 0.423\\
$\beta^0_{1:4} = (0, 0, \log n, \log n)$      & JM     & 0.500 & 0.498 ± 0.535          & 8.727 ± 9.590          & 31.551 ± 2.089\\
                                              & Oracle & 0.990 & 0.072 ± 0.038          & 0.000 ± 0.000          & -\\
\hline
$(v)$                                         & I-SVB  & 0.966 & \textbf{0.163 ± 0.052} & \textbf{1.325 ± 0.177} & 1.591 ± 17.545\\
$(400, 1000, 6, 12, \log n, 0.5)$             & MF     & 0.616 & 0.246 ± 0.859          & 0.157 ± 0.016          & 7.984 ± 63.167\\
$\beta^0_{1:6} = (\log n, \dots, \log n)$     & JM     & 0.144 & 0.474 ± 0.551          & 7.304 ± 1.964          & 32.330 ± 2.805\\
                                              & Oracle & 0.936 & 0.166 ± 0.054          & 1.000 ± 0.088          & -\\
\hline
$(vi)$                                        & I-SVB  & 0.950 & 0.366 ± 0.111          & \textbf{1.000 ± 0.157} & 0.868 ± 0.148\\
$(400, 1000, 6, 12, \mathcal{U}(-5, 5), 0.9)$ & MF     & 0.576 & 0.175 ± 0.136          & 0.001 ± 0.000          & 0.531 ± 0.196\\
$\beta^0_{1:6} = (0, \dots, 0)$               & JM     & 0.180 & \textbf{0.148 ± 0.090} & 265.69 ± 5538.8        & 29.079 ± 11.735\\
                                              & Oracle & -     & -                      & -                      & -\\
\bottomrule
\end{tabular}
}
\caption{Estimation and uncertainty quantification performance for $\beta_{1:k}$ using 3 methods and the Oracle approach. Highlighted in bold are the smallest relative volume (compared to the Oracle) of the sets, subject to coverage being larger than 0.85, and the smallest mean $L_2$-error. 
In (vi) the Oracle reduces to the point estimate at 0, hence is omitted.}
\label{Tab:results_multi_d}
\end{table}

\textbf{Simulation summary.} Our proposed I-SVB method provides reliable statistical inference and uncertainty quantification for $\beta_{1:k}$, significantly improving upon standard mean-field variational Bayes. I-SVB (1) avoids the typical undercoverage of standard MF VB and (2) is available for multi-dimensional parameters $k\geq 2$, where in two dimensions it even comes close to the optimal oracle ellipsoidal `shape' expected when the covariance is of AR type.  I-SVB is competitive with several well-established frequentist methods: in our study, the best other method is ZZ for $k=1$ dimensional targets, and JM for $k\ge 2$ dimensional targets, compared to which I-SVB performs as follows: 
\begin{itemize}
\item for one dimensional targets $(k=1)$ and a $\Sigma_\rho$ correlation matrix, I-SVB is overall best for moderate or large signals, while slightly conservative compared to ZZ for low signals with slightly larger credible sets; 
\item for $k\ge 2$, I-SVB is overall best for  $\Sigma=\Sigma_{AR}$ and $\Sigma=\Sigma_\rho$ (see Section \ref{sec:additional_simulations}, Figures \ref{fig:credible_regions_full_corr}--\ref{fig:cov_structure_full_corr} in the appendix). For covariance $\Sigma_{AR}$, the credible set shape nearly matches the oracle.  
\end{itemize}

\section{Theoretical guarantees}\label{sec:BvM_VB_1d}

We now establish asymptotic theoretical guarantees and optimality of our method when the features $X_i$ are not too correlated, in a sense made precise in \eqref{assum::design_matrix} below. 
We work in the frequentist framework, where we assume there is a true parameter $\beta^0$ generating the data according to model \eqref{eq:linear_regression} with $\sigma =1$ for simplicity, and derive a Bernstein-von Mises (BvM) theorem for our VB posterior for $\beta_{1}$ in the asymptotic regime $n, p \rightarrow \infty$.  Similar results for $\beta_{1:k}$ with $k=k_n \geq 2$ possibly increasing can be found in Section \ref{sec:proofs} in the supplement.


Write $\mathcal{L}_{\hat{Q}}(\|X_1\|_2(\beta_{1} - \hat{\beta}_1))$ for the marginal variational posterior of $\|X_1\|_2(\beta_{1} - \hat{\beta}_1)$,  where $\hat{\beta}_1$ is any random sequence satisfying
\begin{equation}\label{eq:efficient_centering}
\hat{\beta}_1 = \beta^0_1 + \frac{X_1^T \eps}{\|X_1\|_2^2} + o_{P_0}\left(\frac{1}{\|X_1\|_2}\right).
\end{equation}
We say the \textit{semiparametric Bernstein-von Mises (BvM)} holds for $\hat{Q}_1$ if
\begin{align}\label{def:semiparam_BvM_1d}
	\left\| \mathcal{L}_{\hat{Q}}(\|X_1\|_2(\beta_{1} - \hat{\beta}_1))- \mathcal{N}(0,1) \right\|_{TV} \xrightarrow{P_0} 0,
\end{align}
as $n\to\infty$, where $\|\cdot\|_{TV}$ is the total variation distance. Several frequentist estimators have been shown to satisfy the expansion \eqref{eq:efficient_centering} \citep{Zhang2014,JM2018,DY2019}, which mirrors the classic notion of semiparametric efficiency in settings with low correlation among columns of $X$. For instance, for random $X$ with i.i.d. $\mathcal{N}(0,1)$ entries, $\|X_1\|_2/\sqrt{n}\to 1$ with high probability and it is shown in \cite{jvdg18} that $1/\sqrt{n}$ is the `best possible' variance (in either Cramer-Rao-type or Le Cam sense) in this setting for estimating $\beta_1$. In settings with {stronger correlation in $X$, still in the random design case with $X_{i\cdot}\sim\mathcal{N}(0,\Sigma)$ i.i.d.,  it can be shown (\cite{jvdg18}, Section 8) that under some conditions the `best possible' variance is more generally $(\Sigma^{-1})_{11}$. Proving theory for Bayesian methods in more correlated settings is a challenging open problem that is beyond the scope of this paper, but is an interesting avenue for future work.

Equation \eqref{def:semiparam_BvM_1d} states that the marginal VB posterior for $\beta_1$ is asymptotically Gaussian, centered at an estimator verifying \eqref{eq:efficient_centering} and has variance equal to the frequentist variance of $\hat{\beta}_1$ (in contrast to full MF VB). An especially important implication is that the resulting credible intervals for $\beta_{1}$ are asymptotic confidence intervals of the right level. This provides a theoretical guarantee that both estimation and uncertainty quantification for $\beta_{1}$ based on our VB approach is not only computationally scalable, but also (asymptotically) reliable. 
While an estimator satisfying \eqref{eq:efficient_centering} can in principle be used to construct an asymptotic confidence interval by inverting the limit theorem, Bayesian methods can have practical advantages in some settings, such as better performance in small or moderate data sizes, the ability to incorporate prior knowledge and automatic uncertainty quantification methodology via credible sets.

As discussed above, for the sparse nuisance term $\beta_{-1}$, computation is achieved by considering the transformed linear regression model $\cY = \cW \beta^0_{-1} + \check{\eps}$ defined in \eqref{eq:preprocess}. For our theory, we must therefore translate the conditions required for sparse (VB) posteriors to behave well from the usual regression setting $(X,Y)$ to $(\cW,\cY)$. We first require a mild condition on the hyperparameter $\lambda$ in the slab distribution of the model selection prior in Definition \ref{def:sparse_prior}:
\begin{align}\label{assum::prior_2}
	\frac{\| \cW \|}{p-1} \leq &\lambda \leq 2 \bar{\lambda}, \qquad \bar{\lambda}=2\|\cW\|\sqrt{\log (p-1)}.
\end{align}
This assumption is standard \citep{CS-HV2015,Ray_Szabo_2020} and rules out large values of $\lambda$ which may shrink coordinates $\beta_i$ in the slab towards zero, which we want to avoid since our prior instead induces sparsity via $\nu$. In practice, one would usually to take $\lambda$ fixed or $\lambda \to 0$.

The sparse parameter $\beta$ in \eqref{eq:linear_regression} is not identifiable without assumptions on the design matrix $X$. We make analogous assumptions on $\cW$, and discuss how these relate to conditions on the original design $X$ in Section \ref{sec::Additional_result}. The following assumptions are standard in the high-dimensional statistics literature \citep{Buhlmann2011} and their exact formulation are taken from \cite{CS-HV2015}, see Section 2.2 of \cite{CS-HV2015} for further discussion.

\begin{definition}[Compatibility]\label{def:compatibility_number}
	For a design matrix $X \in \R^{n \times p}$, a model $S \subset \{1, \dots, p\}$ has compatibility number
	$$
	\phi^X(S) := \inf \left\{\frac{\|X\beta\|_2 |S|^{1/2}}{\|X\|\|\beta_S\|_1} : \beta \in \R^p, \|\beta_{S^c}\|_1 \leq 7 \|\beta_S \|_1, \beta_S \neq 0  \right\}.
	$$
\end{definition}

\begin{definition}
    The smallest sparse scaled singular value of dimension $s$ is
    $$
    \tilde{\phi}^X(s) := \inf \left\{\frac{\|X\beta\|_2}{\|X\|\|\beta\|_2} : \beta \in \R^p, 0 \neq |S_\beta| \leq s \right\}.
    $$
\end{definition}
Positive lower bounds on the above quantities imply a kind of invertibility of $X$ over (approximately) sparse vectors. The number 7 in Definition \ref{def:compatibility_number} is not important and is taken in Definition 2.1 of \cite{CS-HV2015} to provide a specific numerical value; since we use one of their results, we employ the same convention. Such compatibility type constants are bounded away from zero for many standard designs, such as diagonal matrices, orthogonal designs, i.i.d. (including Gaussian) random matrices and matrices satisfying the `strong irrepresentability condition', see \cite{CS-HV2015,Ray_Szabo_2020}. We shall require these quantities to be bounded away from zero for $X$ replaced by $\cW$ and $s$ a multiple of the true model size. For $M > 0$, define
\begin{align*}
    \tilde{\psi}^X_M(S) &:= \tilde{\phi}\left(\left(2 + \frac{4M}{A_4} \left( 1 + \frac{16}{\phi^X(S)^2} \frac{\lambda}{\bar{\lambda}} \right) \right) |S| \right),
\end{align*}
where $\lambda,\bar{\lambda}$ satisfy \eqref{assum::prior_2} and $A_4$ is the prior constant in \eqref{assum::prior_1}. In the typical case $\lambda \ll \bar\lambda$, the constant in the last display is asymptotically bounded from below by $\tilde{\phi}((2+4M/A_4)|S|)$. We require that for some $\rho_n\rightarrow \infty$ and $c_0 > 0$, the $s_0$-sparse (sequence of) vectors $\beta^0 \in \R^p$ satisfy
\begin{align}\label{assum::beta_0}
\phi^{\check{W}}(S_{\beta^0_{-1}}) \geq c_0, \qquad \tilde{\psi}^{\check{W}}_{\rho_n}(S_{\beta^0_{-1}}) \geq c_0, \qquad s_0=o(n)
\end{align}
as $n,p\to\infty$. Such conditions are typically implied by the same conditions on $X$ instead of $\cW$, see Lemma \ref{lem:compatibility_number}. We now state our main theoretical result.


\begin{theorem}\label{thm:asymptotic_normality_variational_1D} Let $\Pi$ be the prior \eqref{general_prior} satisfying \eqref{assum::prior_1}  and\eqref{assum::prior_2}. Suppose that $\beta^0 \in \R^p$ satisfies \eqref{assum::beta_0}  for some $\rho_n\to \infty$, that $\lambda = O(\|\cW\|\sqrt{\log  (p-1)}/s_0)$  and $\log g$ is $c$-Lipschitz with $c=o(\|X_1\|_2)$. If the design matrix satisfies
\begin{align}\label{assum::design_matrix}
	\frac{ \|X_1\|_2 \max_{i = 2, \dots, p} |\gamma_i|}{\max_{i=2, \dots, p}\|(I-H)X_{i}\|_2} \rho_n s_0 \sqrt{\log p} \rightarrow 0,
\end{align}
then the semiparametric BvM holds for the VB posterior $\hat{Q}_{1}$.	
\end{theorem}

\begin{remark}[Misspecification of the noise distribution]\label{rem:noise}
A Gaussian error distribution is assumed in model \eqref{eq:linear_regression} for concreteness and can be relaxed. For our theoretical results, an inspection of the proofs shows that it suffices that there exists a constant $C>0$ such that
\begin{align}\label{eq:noise_miss}
P_0 \left( \| \cW^T(\cY-\cW\beta_{-1}^0)\|_\infty > C \|\cW\| \sqrt{\log(p-1)} \right) \to 0,
\end{align} 
which holds for much more general noise distributions than Gaussian. This condition, as applied to the transformed regression model $\cY = \cW \beta_{-1} + \check{\eps}$, is commonly used when studying the LASSO, see e.g. \cite{Buhlmann2011}. Our theoretical results are thus robust to noise misspecification.

However, while the BvM \eqref{def:semiparam_BvM_1d} may still hold when the noise $\eps$ is misspecified, one should be careful about its implications since the asymptotic distribution of $\hat{\beta}_1$ satisfying \eqref{eq:efficient_centering} may no longer be $\mathcal{N}(\beta_1^0,1/\|X_1\|_2^2)$. In this case, the posterior may not correctly capture the uncertainty of $\hat{\beta}_1$, and credible sets can have too small or large coverage depending on the distribution of $\eta$. This mirrors the parametric case, see Example 2.1 in \cite{kleijn2012}.
\end{remark}


\begin{corollary}\label{cor:laplace_improper} The Laplace prior (Example \ref{laplace_prior}) with $1/\sigma_n = o(\|X_1\|_2)$ and the improper prior (Example \ref{improper_prior}) satisfy the condition on the density $g$ in Theorem \ref{thm:asymptotic_normality_variational_1D}. Thus, under the other conditions of Theorem  \ref{thm:asymptotic_normality_variational_1D}, the semiparametric BvM holds for $\hat{Q}_{1}$.
\end{corollary} 

Recalling that $\gamma_i=X_1^TX_i/\|X_1\|_2^2$, condition \eqref{assum::design_matrix} essentially requires that the correlations between columns $X_1^TX_i$ are not too large. The role of the  sequence $\rho_n$ is purely technical (it arises for the same technical reasons as in \cite{Ray_Szabo_2020}) and it can be taken to be any slowly diverging sequence, e.g. $\rho_n=\log{n}$. Condition \eqref{assum::design_matrix} is verified in Section \ref{sec::Additional_result} for $X_{ij}\sim^{iid} \mathcal{N}(0,1)$ and could similarly be verified in the other random design settings of Section \ref{sect:simulations}, provided the correlation parameter vanishes fast enough. We provide analogous results in Theorem \ref{thm:asymptotic_normality_variational_kD} in the supplement for the multidimensional growing case $k=k_n \geq 1$. In particular, if the sparsity $s_0$ is very small, one can even take $k$ almost of the order $n$ when using an improper prior for $g$, covering large parameter subsets of interest, see Appendix \ref{sec:proofs}.

While our theoretical guarantees are established in the lower correlation regime, the above simulations show that the method continues to perform well in more strongly correlated settings, see Section \ref{sect:simulations}. Under stronger correlation, the VB posterior performs both a significant bias and variance correction, which is needed for accurate statistical inference and uncertainty quantification. In Section \ref{sec:heuristics}, we provide heuristic computations explaining the behaviour of our method in the correlated settings with random Gaussian design considered numerically in Section \ref{sect:simulations}. Rigourously extending these results to general correlated settings will require new proof ideas and is beyond the scope of this work.

For Gaussian priors, $\log g$ is not Lipschitz and has lighter tails, requiring extra assumptions.

\begin{lemma}\label{lem:asymptotic_normality_gaussian_variational}
Let $\Pi$ be the prior \eqref{general_prior} with $g\sim \mathcal{N}(0,\sigma_n^2)$ (Example \ref{gaussian_prior}) satisfying \eqref{assum::prior_1}  and\eqref{assum::prior_2}.
	Suppose that $\beta^0 \in \R^p$ satisfies \eqref{assum::beta_0},  $\lambda = O(\|\cW\|\sqrt{\log  (p-1)}/s_0)$ and the design matrix satisfies \eqref{assum::design_matrix}. If in addition
 \begin{align}\label{assum:asymp_nor_gauss}
     \sigma_n\|X_1\|_2 \rightarrow \infty, \qquad \frac{|\beta^0_1|}{\sigma_n^2\|X_1\|_2} \rightarrow 0 \qquad \text{and} \quad \frac{\max_{i=2,\dots,p} |\gamma_i|||\beta^0_{-1}||_1}{\sigma_n^2\|X_1\|_2} \rightarrow 0,
 \end{align}
	then the semiparametric BvM holds for $\hat{Q}_{1}$.
\end{lemma} 

Condition \eqref{assum:asymp_nor_gauss} requires that $\sigma_n$ be large enough depending on the size of the unknown true $\beta^0$. A similar condition is required in \cite{DY2019} since she also uses a Gaussian prior. In practice, the VB method with Gaussian $g$ can perform poorly if this condition is not satisfied, which requires careful hyperparameter tuning to work, unlike the Laplace and improper case.
This is important in practice since this assumption depends on the unknown truth.

\section{Discussion}\label{sec:discussion}

We propose a scalable variational Bayes approach for inference of a low-dimensional parameter in sparse high-dimensional linear regression. Our approach combines a simple preprocessing step with a variational class tailored to this problem, allowing one to plug-in existing computational tools for fast and simple computation. We demonstrate that our I-SVB method significantly improves on standard mean-field variational inference for this task, and performs favourably compared to two frequentist counterparts in terms of both estimation accuracy and uncertainty quantification. The I-SVB method demonstrates robustness to correlated features, various forms of signal, and extends naturally to multidimensional uncertainty quantification. Furthermore, I-SVB produces multidimensional credible regions which
are close to optimal `oracle'-type confidence sets across a variety of scenarios. Our empirical results are supported by theoretical frequentist guarantees.

\textbf{Acknowledgements.} We wish to thank the Imperial College London-CNRS PhD Joint Programme for funding to support this collaboration, including funding ALH's PhD position. IC's work is partly supported by ANR grant project ANR-23-CE40-0018-01 (BACKUP) and an Institut Universitaire de France fellowship.


\bigskip 
\begin{center}
{\large\bf SUPPLEMENTARY MATERIAL}
\end{center}

\begin{description}
\item [Supplementary material.] In Section \ref{sec:additional_simulations}, we present additional numerical results, including a real-world covariate data example, further simulations and investigations into methodological choices for our approach (e.g. prior choice for $g$). We also provide heuristic calculations explaining the performance of our method in more strongly correlated settings (Section \ref{sec:heuristics}), a semiparametric BvM for the case $k\geq 2$ and all proofs (Section \ref{sec:proofs}), discussion of the computational cost of our algorithm (Section \ref{sec:computational_cost}), discussion of the conditions for our theoretical results (Section \ref{sec::Additional_result}) and further background (Section \ref{sec:additional_VB}).
\end{description}

\newpage

\appendix

 \section{Additional numerical results} \label{sec:additional_simulations}
We present here additional simulations, including a real world covariate data example (Section \ref{sec:realdata}), a comparison of the different prior choices for $g$ (Section \ref{sec:g_prior}), additional simulations and visualizations to explain the behaviour of our method (Section \ref{sec:additional_sims}) and what happens if one does not randomize the nuisance parameter $\beta_{-1}$ (Section \ref{sec:no_randomization}).

\subsection{Real covariate example: riboflavin dataset} \label{sec:realdata}
Consider real covariates from the {\texttt riboflavin} dataset \citep{riboflavinData} with simulated responses. The data is publicly available from the journal's webpage  (\href{https://www.annualreviews.org/content/journals/10.1146/annurev-statistics-022513-115545#supplementary_data}{link}). The dimension of the covariate data $X$ is $n = 71$ and $p = 4088$, and the pairwise correlation between features ranges from close to -1 to close to +1. We sample 500 realisations of the dataset $Y = X\beta^0 + \epsilon$ (holding $X$ and $\beta^0$ fixed) computing the corresponding $95\%$-credible/confidence regions for each method. We take $k=1$ and consider as target parameter $\beta_j$, i.e. we do not necessarily take $j=1$. Since the design is fixed, the different covariate columns of $X$ have very different structures and hence the performance of each method depends on this choice. To provide a fair comparison, we sample a random subset of the coordinates of size 100, compute estimates of the coverage for each method on each coordinate using 500 realisations of the data, and report the mean performance over the 100 coordinates. Results are given in Table \ref{Tab:real_data}, with the Oracle included for reference.

In the first example, we set both $\beta_j^0 = \log n$ and the other active coordinates $\beta_k^0 = \log n$. In this case, we see that I-SVB and ZZ both offer good coverage, with the I-SVB coverage slightly below the intended level. Of the non-oracle methods, I-SVB offers the smallest MAE, and the credible intervals produced by this method are nearly the same size as the Oracle intervals, showing good uncertainty quantification. The mean-field method struggles in this scenario, as the complex correlation structure between features throws off the modelling of the nuisance parameter and the credible intervals are too small. The ZZ method offers high coverage, but the confidence sets produced by this method are very wide, and the estimate itself is more biased than for I-SVB. The JM method struggles due to a large bias and the confidence intervals have low coverage despite being quite large.

In the second example, we set the coordinate of interest $\beta_j^0 = 0$, with the other active coordinates given by $\beta_k^0 = \log n$. In this case, the I-SVB coverage drops, though not too dramatically. The MF method performs very well in this scenario, which is understandable since  the prior contains a spike at 0 and it can often detect that the truth is zero; the MAE is very small and the credible intervals are the smallest of the non-oracle methods. The ZZ method once again yields very high coverage, but it is again conservative due to the very long intervals. Likewise, the JM method offers high coverage, but with long credible intervals.\\

\begin{table}[htbp]
\centering
\footnotesize
\resizebox{\textwidth}{!}{
\begin{tabular}{l|r|cccc}
\toprule
Scenario  $(n, p, s_0 , \beta_j^0 , \beta_k^0)$ & Method & Cov.  & MAE           & Length        & Time\\
\midrule
{\texttt riboflavin} $(i)$              & I-SVB  & 0.905 & \textbf{0.159 ± 0.127} & \textbf{0.614 ± 0.021} & 2.129 ± 0.275\\
$(71 , 4088 , 5 , \log n , \log n )$ & MF     & 0.262 & 2.078 ± 0.931 & 0.463 ± 0.002 & 1.828 ± 0.422\\
                                        & ZZ     & 0.971 & 0.422 ± 0.250 & 4.691 ± 1.465 & 0.656 ± 0.162\\
                                        & JM & 0.300 & 2.430 ± 0.135 & 1.907 ± 0.103 & 22.573 ± 3.304\\
                                        & Oracle & 0.944 & 0.125 ± 0.093 & 0.613 ± 0.000 & -\\
\hline
{\texttt riboflavin} $(ii)$             & I-SVB  & 0.857 & 0.196 ± 0.154 & 0.646 ± 0.022 & 3.209 ± 0.511\\
$(71 , 4088 , 5, 0, \log n)$         & MF     & 0.976 & \textbf{0.042 ± 0.032} & \textbf{0.450 ± 0.005} & 4.884 ± 7.862\\
                                        & ZZ     & 1.000 & 0.174 ± 0.131 & 4.664 ± 1.380 & 2.650 ± 4.094\\
                                        & JM & 0.992 & 0.078 ± 0.033 & 1.408 ± 0.15 & 14.177 ± 2.483\\
                                        & Oracle & 1.000 & 0.000 ± 0.000 & 0.000 ± 0.000 & -\\
\bottomrule
\end{tabular}
}
\caption{Estimates of the coverage, MAE and length of the credible intervals produced using simulated data with the covariates from the {\texttt riboflavin} dataset. Highlighted in bold are the smallest interval length, subject to coverage being larger than 0.85, and the smallest MAE.}
\label{Tab:real_data}
\end{table}

\subsection{Comparing the priors on $g$}\label{sec:g_prior}

We compare the performance of our method \eqref{general_prior} using three different priors for $g$, namely the Laplace (Example \ref{laplace_prior}), Gaussian (Example \ref{gaussian_prior}) and improper (Example \ref{improper_prior}). We denote these three methods by L-SVB-$\sigma$, G-SVB-$\sigma$ and I-SVB, respectively, where $\sigma$ represents the prior standard deviation. For both the Laplace and Gaussian, we use one prior with a relatively light tail ($\sigma = 1$) and one with a heavier tail ($\sigma = 4$).
We again parameterise each scenario by the tuple $(n, p, s_0, \beta_1^0, \beta_j^0, \rho, \sigma^2)$ with $n$ observations, $\beta^0 \in \R^p$ having sparsity $s_0$, and true value $\beta_1^0$, while the other non-zero entries of $\beta^0$ are given by $\beta^0_j$. We take rows $X_{i,\cdot}\sim^{iid} \mathcal{N}_p(0,\Sigma_\rho)$, where the diagonal entries of $\Sigma_\rho$ are 1 and the off-diagonal entries are given by $\rho$ (so that $\rho$ represents the correlation between any pair of features). 

We simulate 500 datasets and compute $95\%-$credible intervals for the different choices of $g$, using 1000 samples from the variational posterior to compute empirical quantiles. We report the $(i)$ coverage (proportion of intervals containing the thruth); $(ii)$ mean absolute error of the interval centerings; $(iii)$ mean interval lengths; and $(iv)$ mean computation time. Results are found in Table \ref{Tab:experiments_svb_comp}.

The main observation is that most of the calibrations perform similarly to one another. Indeed, the heavy tailed methods (every method apart from G-SVB-1) exhibit similar performance in every scenario. Indeed, it is only for very large signals (scenarios $(ii)$, $(iv)$ and $(v)$) that G-SVB-1 drops in performance due to an increase in bias. Otherwise, the length and bias of the credible intervals are consistent within each scenario for every calibration, and they all take a similar time to fit. This shows that our method is not sensitive to the choice of $g$, as long as sufficently heavy tails are used. Note that using an improper $g \propto 1$ avoids the need to choose or tune a variance hyperparameter $\sigma$.

\begin{table}[htbp]
\centering
\scriptsize
\resizebox{\textwidth}{!}{
\begin{tabular}{r|r|cccc}
\toprule
Scenario $(n, p, s_0, \beta_1^0, \beta_k^0, \rho, \sigma^2)$ & Method & Cov.  & MAE           & Length        & Time\\
\midrule
$(i)$                                                        & I-SVB   & 0.934 & 0.078 ± 0.064 & 0.401 ± 0.033 & 0.257 ± 0.065\\
$(100, 1000, 3, \log n, \log n, 0, 1)$                       & G-SVB-1 & 0.926 & 0.086 ± 0.069 & 0.400 ± 0.032 & 0.224 ± 0.053\\
                                                             & G-SVB-4 & 0.940 & 0.078 ± 0.064 & 0.402 ± 0.033 & 0.235 ± 0.067\\
                                                             & L-SVB-1 & 0.936 & 0.078 ± 0.065 & 0.402 ± 0.032 & 0.230 ± 0.058\\
                                                             & L-SVB-4 & 0.938 & 0.078 ± 0.064 & 0.401 ± 0.033 & 0.228 ± 0.057\\
\hline
$(ii)$                                                       & I-SVB   & 0.942 & 0.083 ± 0.064 & 0.400 ± 0.032 & 0.223 ± 0.072\\
$(100, 1000, 3, 2\log n, 2\log n, 0, 1)$                     & G-SVB-1 & 0.844 & 0.112 ± 0.081 & 0.398 ± 0.031 & 0.203 ± 0.048\\
                                                             & G-SVB-4 & 0.948 & 0.083 ± 0.063 & 0.401 ± 0.032 & 0.201 ± 0.046\\
                                                             & L-SVB-1 & 0.942 & 0.083 ± 0.064 & 0.400 ± 0.032 & 0.206 ± 0.056\\
                                                             & L-SVB-4 & 0.942 & 0.083 ± 0.063 & 0.400 ± 0.031 & 0.205 ± 0.052\\
\hline
$(iii)$                                                      & I-SVB   & 0.978 & 0.074 ± 0.054 & 0.414 ± 0.031 & 0.340 ± 0.059\\
$(200, 800, 3, \log n, \log n, 0.5, 1)$                      & G-SVB-1 & 0.964 & 0.082 ± 0.060 & 0.414 ± 0.032 & 0.338 ± 0.063\\
                                                             & G-SVB-4 & 0.978 & 0.074 ± 0.054 & 0.414 ± 0.031 & 0.336 ± 0.061\\
                                                             & L-SVB-1 & 0.980 & 0.074 ± 0.054 & 0.413 ± 0.033 & 0.330 ± 0.054\\
                                                             & L-SVB-4 & 0.982 & 0.073 ± 0.054 & 0.414 ± 0.032 & 0.337 ± 0.064\\
\hline
$(iv)$                                                       & I-SVB   & 0.980 & 0.071 ± 0.053 & 0.414 ± 0.031 & 0.340 ± 0.058\\
$(200, 800, 3, 2\log n, 2\log n, 0.5, 1)$                    & G-SVB-1 & 0.908 & 0.101 ± 0.069 & 0.413 ± 0.030 & 0.336 ± 0.061\\
                                                             & G-SVB-4 & 0.978 & 0.070 ± 0.053 & 0.415 ± 0.031 & 0.339 ± 0.067\\
                                                             & L-SVB-1 & 0.986 & 0.071 ± 0.053 & 0.414 ± 0.030 & 0.338 ± 0.064\\
                                                             & L-SVB-4 & 0.982 & 0.071 ± 0.053 & 0.413 ± 0.031 & 0.338 ± 0.072\\
\hline
$(v)$                                                        & I-SVB & 0.974 & 0.075 ± 0.056 & 0.412 ± 0.028 & 0.352 ± 0.094\\
$(200, 800, 10, \log n^2, \log n^2, 0.5, 1)$                 & G-SVB-1 & 0.232 & 0.275 ± 0.095 & 0.413 ± 0.029 & 0.347 ± 0.108\\
                                                             & G-SVB-4 & 0.974 & 0.074 ± 0.057 & 0.413 ± 0.028 & 0.339 ± 0.086\\
                                                             & L-SVB-1 & 0.974 & 0.074 ± 0.056 & 0.412 ± 0.029 & 0.338 ± 0.090\\
                                                             & L-SVB-4 & 0.972 & 0.074 ± 0.056 & 0.411 ± 0.029 & 0.340 ± 0.095\\
\hline
$(vi)$                                                       & I-SVB & 0.954 & 0.059 ± 0.047 & 0.298 ± 0.018 & 0.417 ± 0.142\\
$(200, 800, 10, 2\log n, 2\log n, 0.9, 1)$                   & G-SVB-1 & 0.938 & 0.063 ± 0.046 & 0.298 ± 0.019 & 0.410 ± 0.140\\
                                                             & G-SVB-4 & 0.952 & 0.059 ± 0.046 & 0.298 ± 0.018 & 0.398 ± 0.130\\
                                                             & L-SVB-1 & 0.948 & 0.060 ± 0.046 & 0.298 ± 0.019 & 0.400 ± 0.142\\
                                                             & L-SVB-4 & 0.954 & 0.059 ± 0.046 & 0.298 ± 0.019 & 0.397 ± 0.141\\
\bottomrule
\end{tabular}
}
\caption{Uncertainty quantification for each debiased SVB method in 6 different scenarios.}
\label{Tab:experiments_svb_comp}
\end{table}

\subsection{Additional simulations}\label{sec:additional_sims}
Tables \ref{Tab:experiments_1d_supplement_1} and \ref{Tab:experiments_1d_supplement_2} show the performance of the various methods in additional scenarios, where again the correlation structure between features is given by $\Sigma_\rho$ described in Section \ref{sect:simulations}. The discussion of these results is largely the same as that presented in Section \ref{sect:simulations}.

\begin{table}[htbp]
\centering
\footnotesize
\resizebox{\textwidth}{!}{
\begin{tabular}{r|r|cccc}
\toprule
Scenario $(n, p, s_0, \beta_1^0, \beta_j^0, \rho, \sigma^2)$  & Method & Cov.  & MAE                        & Length                     & Time\\
\hline
$(vii)$                                                       & I-SVB  & 0.964 & \textbf{0.056 $\pm$ 0.041} & 0.281 $\pm$ 0.016          & 0.370 $\pm$ 0.127\\
$(200, 800, 3, \log n, \log n, 0, 1)$                         & MF     & 0.968 & \textbf{0.056 $\pm$ 0.041} & \textbf{0.279 $\pm$ 0.013} & 0.223 $\pm$ 0.092\\
                                                              & ZZ     & 0.952 & 0.062 $\pm$ 0.046          & 0.293 $\pm$ 0.039          & 0.479 $\pm$ 0.153\\
                                                              & JM     & 1.000 & 0.070 ± 0.053              & 0.483 ± 0.025              & 2.451 ± 0.203\\
\hline
$(viii)$                                                      & I-SVB  & 0.946 & \textbf{0.056 $\pm$ 0.045} & 0.286 $\pm$ 0.017          & 0.397 $\pm$ 0.135\\
$(200, 800, 10, \log n, \log n, 0, 1)$                        & MF     & 0.928 & \textbf{0.056 $\pm$ 0.045} & \textbf{0.279 $\pm$ 0.014} & 0.250 $\pm$ 0.119\\
                                                              & ZZ     & 0.880 & 0.074 $\pm$ 0.059          & 0.295 $\pm$ 0.061          & 0.511 $\pm$ 0.161\\
                                                              & JM     & 1.000 & 0.078 ± 0.058              & 0.498 ± 0.029              & 2.547 ± 0.291\\
\hline
$(ix)$                                                        & I-SVB  & 0.926 & \textbf{0.045 $\pm$ 0.034} & \textbf{0.203 $\pm$ 0.009} & 1.299 $\pm$ 0.238\\
$(400, 1500, 32, \mathcal{N}(0, 1), \mathcal{N}(0, 1), 0, 1)$ & MF     & 0.818 & 0.059 $\pm$ 0.052          & 0.196 $\pm$ 0.007          & 0.688 $\pm$ 0.137\\
                                                              & ZZ     & 0.822 & 0.059 $\pm$ 0.045          & 0.201 $\pm$ 0.012          & 1.396 $\pm$ 0.259\\
                                                              & JM     & 0.724 & 0.145 ± 0.105              & 0.369 ± 0.015              & 17.77 ± 1.712\\
\hline
$(x)$                                                         & I-SVB  & 0.982 & 0.101 $\pm$ 0.075          & 0.595 $\pm$ 0.054          & 0.295 $\pm$ 0.091\\
$(100, 1000, 3, \log n, \log n, 0.5, 1)$                      & MF     & 0.896 & \textbf{0.097 $\pm$ 0.071} & \textbf{0.395 $\pm$ 0.028} & 0.273 $\pm$ 0.116\\
                                                              & ZZ     & 0.948 & 0.152 $\pm$ 0.115          & 0.845 $\pm$ 0.603          & 0.324 $\pm$ 0.122\\
                                                              & JM     & 0.960 & 0.231 ± 0.149              & 0.962 ± 0.11               & 1.718 ± 0.202\\
\hline
$(xi)$                                                        & I-SVB  & 0.986 & 0.069 $\pm$ 0.049          & 0.413 $\pm$ 0.032          & 0.359 $\pm$ 0.089\\
$(200, 800, 3, \log n, \log n, 0.5, 1)$                       & MF     & 0.914 & \textbf{0.067 $\pm$ 0.049} & \textbf{0.278 $\pm$ 0.014} & 0.261 $\pm$ 0.084\\
                                                              & ZZ     & 0.978 & 0.084 $\pm$ 0.063          & 0.475 $\pm$ 0.078          & 0.530 $\pm$ 0.130\\
                                                              & JM     & 0.976 & 0.131 ± 0.085              & 0.630 ± 0.063              & 3.863 ± 0.49\\
\bottomrule
\end{tabular}
}
\caption{Assessing the performance of the uncertainty quantification provided by each method in 5 additional scenarios. Highlighted in bold are the smallest interval length, subject to coverage being larger than 0.85, and the smallest MAE.}
\label{Tab:experiments_1d_supplement_1}
\end{table}

\begin{table}[htbp]
\centering
\footnotesize
\resizebox{\textwidth}{!}{
\begin{tabular}{r|r|cccc}
\toprule
Scenario  $(n , p    , s_0 , \beta_1^0 , \beta_j^0 , \rho, \sigma^2)$ & Method & Cov.  & MAE                    & Length                 & Time\\
\midrule
$(xii)$                                                               & I-SVB  & 0.830 & \textbf{0.353 ± 0.272} & 1.349 ± 0.391          & 0.339 ± 0.112\\
$(100, 1000, 3, \log n, \log n, 0, 16)$                               & MF     & 0.802 & 0.391 ± 0.290          & 1.302 ± 0.399          & 0.228 ± 0.103\\
                                                                      & ZZ     & 0.766 & 0.468 ± 0.338          & 1.579 ± 1.185          & 0.296 ± 0.081\\
                                                                      & JM     & 0.924 & 0.596 ± 0.342          & \textbf{2.870 ± 0.854} & 1.636 ± 0.219\\
\hline
$(xiii)$                                                              & I-SVB  & 0.956 & 0.054 ± 0.042          & 0.281 ± 0.017          & 0.349 ± 0.083\\
$(200, 1000, 5, 0, \log n, 0, 1)$                                     & MF     & 1.000 & \textbf{0.000 ± 0.000} & \textbf{0.272 ± 0.014} & 0.211 ± 0.066\\
                                                                      & ZZ     & 0.976 & 0.049 ± 0.038          & 0.295 ± 0.048          & 0.434 ± 0.103\\
                                                                      & JM     & 1.000 & 0.041 ± 0.035          & 0.511 ± 0.028          & 2.980 ± 0.243\\
\hline
$(xiv)$                                                               & I-SVB  & 0.928 & \textbf{0.058 ± 0.045} & 0.283 ± 0.017          & 0.389 ± 0.092\\
$(200, 1000, 10, 1, \log n, 0, 1)$                                    & MF     & 0.930 & \textbf{0.058 ± 0.045} & \textbf{0.277 ± 0.014} & 0.238 ± 0.075\\
                                                                      & ZZ     & 0.854 & 0.082 ± 0.062          & 0.292 ± 0.036          & 0.469 ± 0.108\\
                                                                      & JM     & 0.974 & 0.092 ± 0.069          & 0.531 ± 0.032          & 3.237 ± 0.246\\
\hline
$(xv)$                                                                & I-SVB  & 0.962 & \textbf{0.037 ± 0.025} & 0.177 ± 0.007          & 1.157 ± 0.245\\
$(500, 1000, 10, \mathcal{U}(-5, 5), \mathcal{U}(-5, 5), 0, 1)$       & MF     & 0.952 & 0.039 ± 0.030          & \textbf{0.176 ± 0.005} & 0.434 ± 0.140\\
                                                                      & ZZ     & 0.946 & 0.041 ± 0.027          & 0.178 ± 0.008          & 1.407 ± 0.343\\
                                                                      & JM     & 0.574 & 0.136 ± 0.113          & 0.267 ± 0.011          & 25.371 ± 7.692\\
\hline
$(xvi)$                                                               & I-SVB  & 0.986 & 0.101 ± 0.304          & 0.598 ± 0.124          & 0.978 ± 0.242\\
$(200, 1000, 10, 1, \log n, 0.5, 1)$                                  & MF     & 0.638 & 0.314 ± 0.466          & 0.276 ± 0.014          & 2.032 ± 0.587\\
                                                                      & ZZ     & 0.952 & \textbf{0.095 ± 0.074} & \textbf{0.482 ± 0.087} & 0.610 ± 0.079\\
                                                                      & JM     & 0.312 & 1.813 ± 1.359          & 1.496 ± 0.13           & 14.73 ± 1.267\\
\bottomrule
\end{tabular}
}
\caption{Assessing the performance of the uncertainty quantification provided by each method in 5 additional scenarios. Highlighted in bold are the smallest interval length, subject to coverage being larger than 0.85, and the smallest MAE.}
\label{Tab:experiments_1d_supplement_2}
\end{table}

\newpage

\textbf{Further two-dimensional visualizations.} To better understand why I-SVB seems to approach the Oracle covariance under Gaussian feature design with rows $X_{i\cdot} \sim^{iid} \mathcal{N}_p(0, \Sigma_\rho^{AR})$ in Section \ref{sect:simulations}, we  plot the components of the covariance of $\beta_{1:2}$ under the variational posterior. Recall that in 2 dimensions, the I-SVB posterior is characterised by $\beta_{1:2} = \beta_{1:2}^* - (X_{1:2}^T X_{1:2})^{-1}X_{1:2}^T X_{-2} \beta_{-2}$, where the distributions of $\beta_{1:2}^*$ and $\beta_{-2}$ are independent. In Figure \ref{fig:cov_structure_AR}, we plot the ellipses $\{v\in \R^2 : v^T M^{-1} v = \chi_{0.95, 2} \}$ with $M$ given by $(i)$ the I-SVB covariance of $\beta_{1:2}$, $(ii)$ the I-SVB covariance of $\beta_{1:2}^*$, $(iii)$ the I-SVB covariance of $(X_{1:2}^T X_{1:2})^{-1}X_{1:2}^T X_{-2} \beta_{-2}$, and $(iv)$ the Oracle covariance of $\beta_{1:2}$, in order to show the covariance structures of these terms.

We first see that the covariances of $\beta_{1:2}$ and $\beta_{1:2}^*$ are notably different, so that the `debiasing quantity' $(iii)$ performs a significant covariance correction, which can been seen in its own covariance structure (purple ellipse). The covariance of $\beta_{1:2}^*$ (green) is too small and not correctly aligned with the oracle structure (blue). The debiasing correctly captures the uncertainty due to the nuissance parameter $\beta_{-2}$, leading to a covariance structure for $\beta_{1:2}$ that is closer to the oracle covariance. Note this is not a simple inflation factor, since this debiasing occurs by varying amounts and in different directions (not necessarily parallel to the coordinate axes), and becomes more pronounced as the feature correlation increases (purple). Our method thus performs both a bias and covariance correction.

\begin{figure}
    \centering
    {\bf Covariance structure of the components of $\beta_{1:2}$}
    \includegraphics[width=0.9\linewidth]{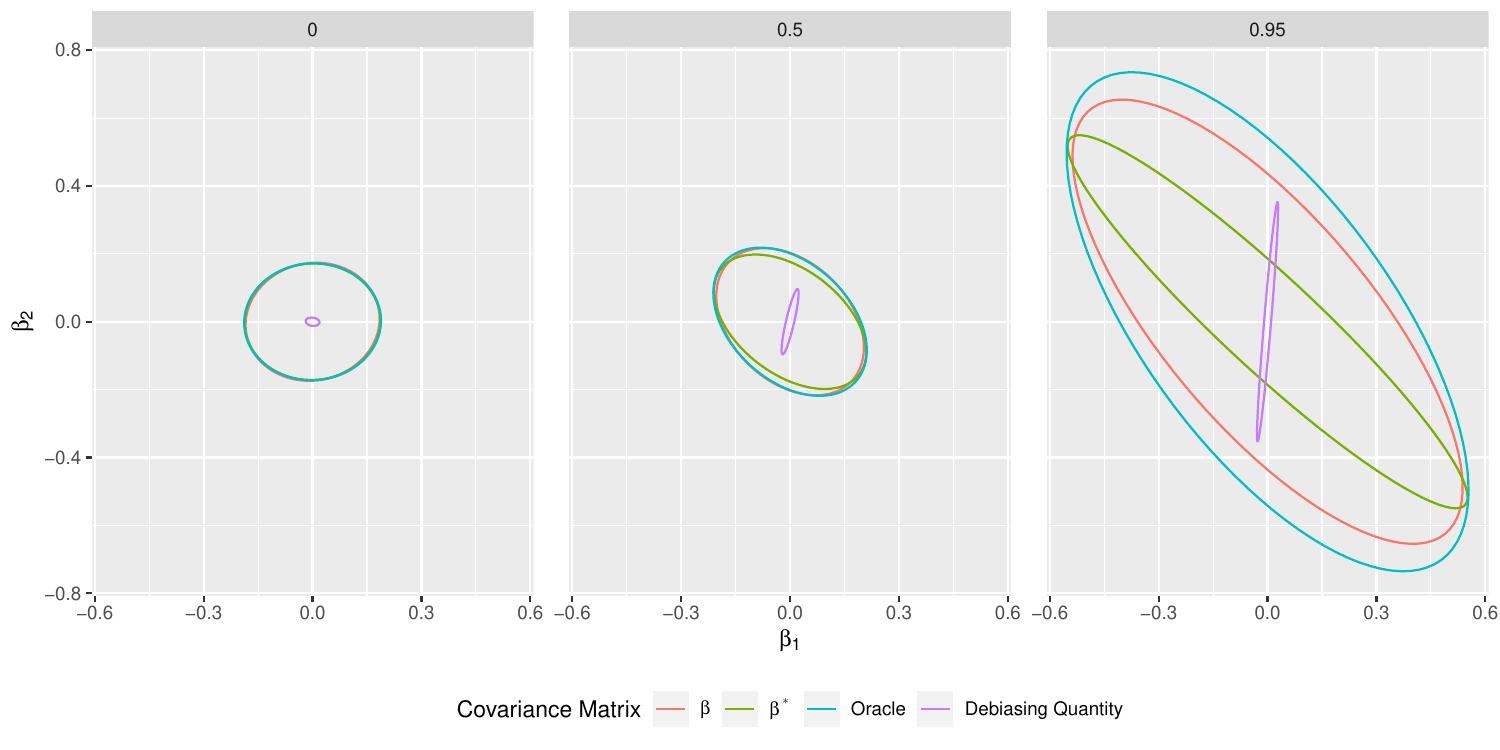}
    \caption{Realisations of the sets $\{v\in \R^2 : v^T M^{-1} v = \chi_{0.95, 2} \}$ with $M$ given by the covariance matrices of $\beta_{1:2}$, $\beta_{1:2}^*$, the Oracle, and the debiasing quantity ($(X_{1:2}^TX_{1:2})^{-1}X_{1:2}^TX_{-2}\beta_{-2}$). The design matrix rows are distributed as $X_{i\cdot}\sim^{iid} \mathcal{N}_p(0, \Sigma_\rho^{AR})$ for increasing values of $\rho$ given in the title of each facet. Each scenario has $n=200, p=400, k=2, s_0=10$ and $\beta^0_i=5$.}
    \label{fig:cov_structure_AR}
\end{figure}

The analogous plots for the more heavily correlated feature setting with design $X_{i\cdot} \sim \mathcal{N}_p(0, \Sigma_\rho)$, where $[\Sigma_\rho]_{ij} = \rho$ if $i \neq j$ and is 1 otherwise, are shown in Figures \ref{fig:credible_regions_full_corr} and \ref{fig:cov_structure_full_corr}. Figure \ref{fig:cov_structure_full_corr} shows that in this more correlated scenario, it is even more important to account for the covariance structure of the debiasing, as the ellipse from $\beta_{1:2}^*$ is far too small; once the debiasing has taken place (which depends on $\beta_{-k} \sim \hat{Q}_{-k}$) the covariance structure of $\beta_{1:2}$ is better, though perhaps conservative.

\begin{figure}
    \centering
    {\bf Realisations of 2D credible regions}
    \includegraphics[width=0.9\linewidth]{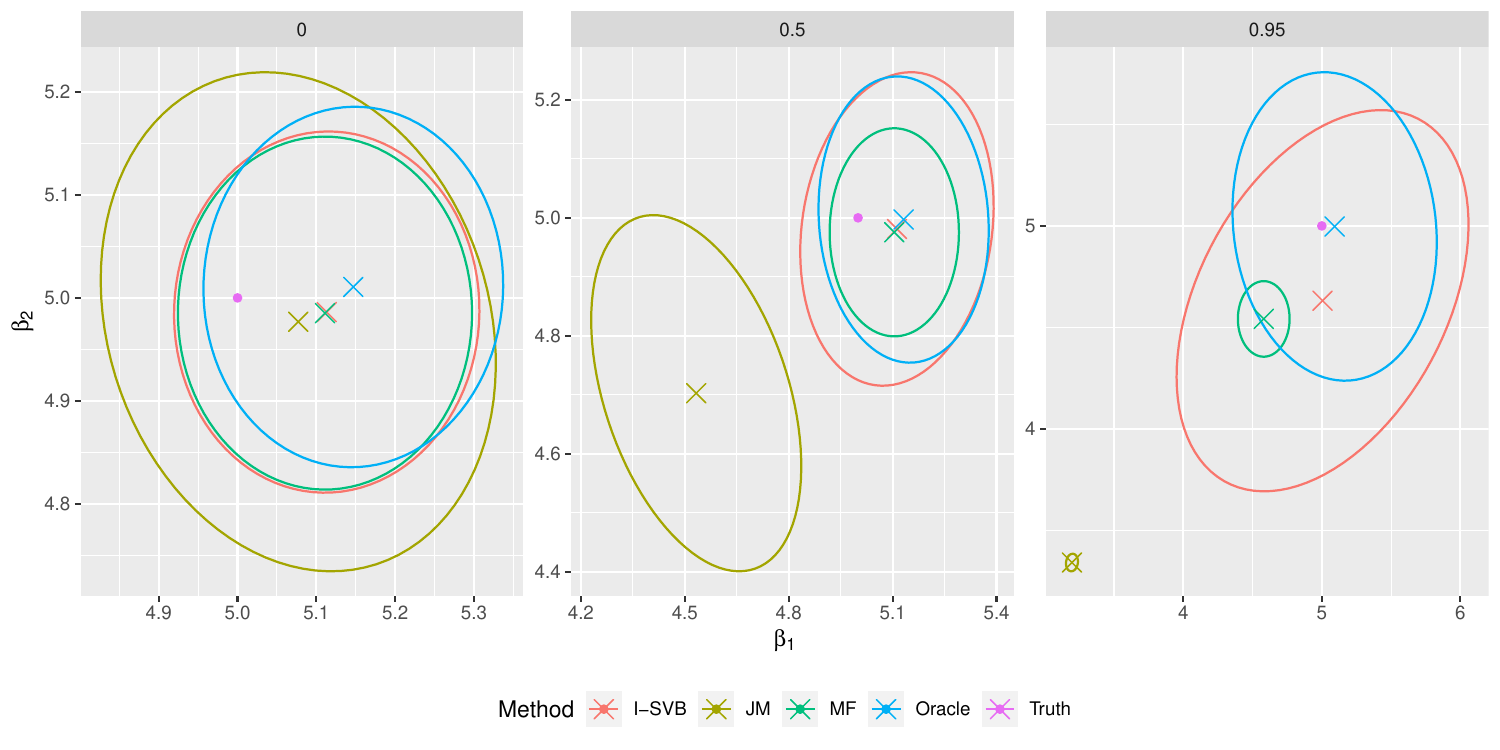}
    \caption{95\%-credible or confidence regions in three scenarios with design matrix rows $X_{i\cdot}\sim^{iid} \mathcal{N}_p(0, \Sigma_\rho)$ for increasing values of $\rho$ given in the title of each facet. The interior of the ellipses represent the credible or confidence regions, the crosses mark their centering and the pink point is the true value of $(\beta_1, \beta_2)$. Each scenario has $n=200, p=400, k=2, s_0=10$ and $\beta^0_i=5$.}
    \label{fig:credible_regions_full_corr}
\end{figure}

\begin{figure}
    \centering
    {\bf Covariance structure of the components of $\beta_{1:2}$}
    \includegraphics[width=0.9\linewidth]{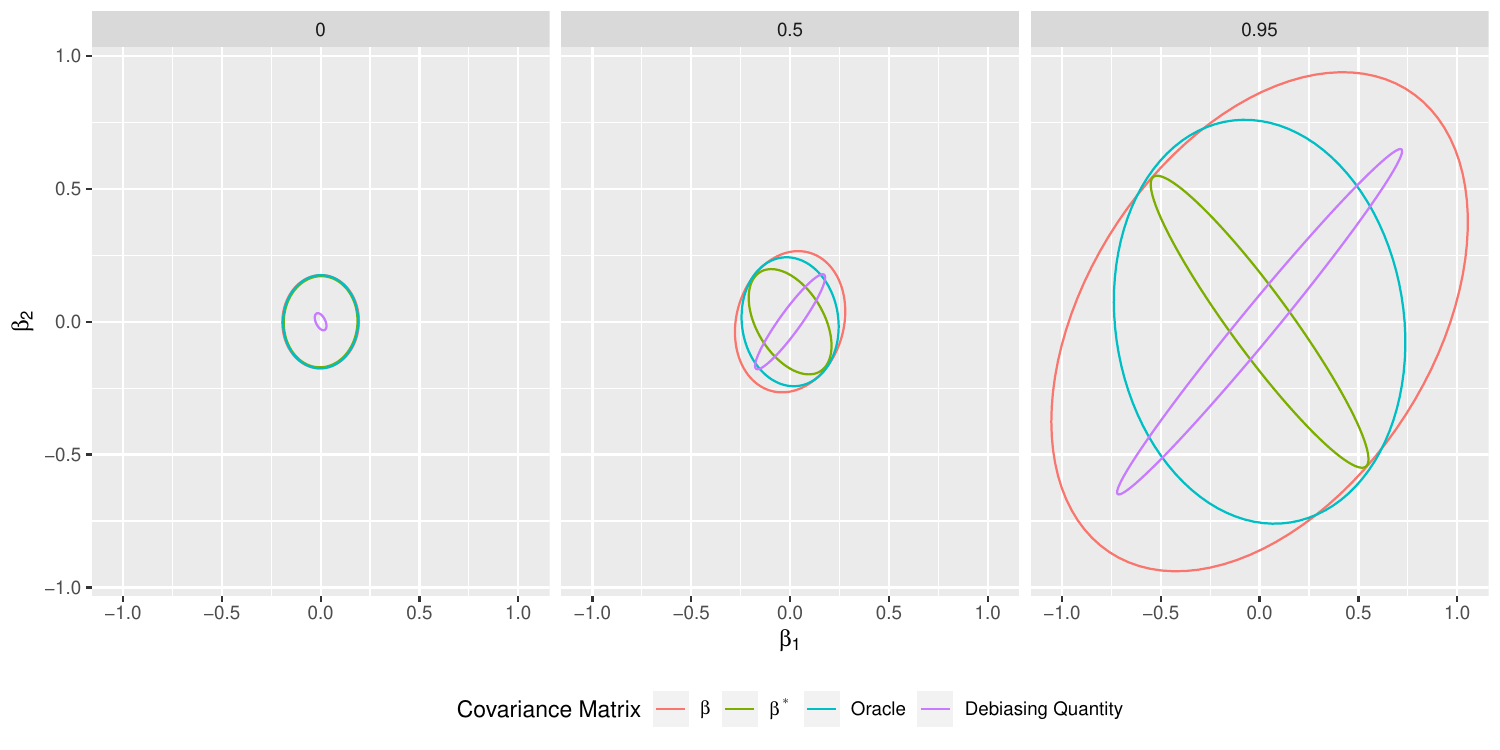}
    \caption{Structure of the 95\%-credible regions given by the covariance matrices of $\beta_{1:2}$, $\beta_{1:2}^*$, the Oracle, and the debiasing quantity ($(X_{1:2}^TX_{1:2})^{-1}X_{1:2}^TX_{-2}\beta_{-2}$) with design matrix rows $X_{i\cdot}\sim^{iid} \mathcal{N}_p(0, \Sigma_\rho)$ for increasing values of $\rho$ given in the title of each facet. Each scenario has $n=200, p=400, k=2, s_0=10$ and $\beta^0_i=5$.}
    \label{fig:cov_structure_full_corr}
\end{figure}

\subsection{The role of randomizing $\beta_{-k}$}\label{sec:no_randomization}

The asymptotic theory in Section \ref{sec:BvM_VB_1d} suggests that in low-correlation settings, the variance of $\beta_{-1}$ is of smaller order than that of $\beta_1^*$ and hence one may not need to randomize $\beta_{-1}$. We show here that even in such low-correlation settings, it is still beneficial to include the VB posterior variance of $\beta_{-1}$ since this leads to better finite-sample performance. Note that in moderate or high correlation settings, the discussion in Sections \ref{sec:additional_sims} and \ref{sec:heuristics} shows that this variance must be accounted for, and our method performs a non-trivial variance correction that is crucial for good uncertainty quantification. We therefore always recommend to randomize $\beta_{-k}$ according to the variational family $\hat{Q}_{-k}$.

Turning to the case $k=1$, recall that we sample $\beta_{-1} \sim \hat{Q}_{-1}$ and $\beta_1^* \sim \pi(\beta_1^* | Y)$ independently, and then compute $\beta_1 = \beta_1^* - \sum_{i \geq 2} \gamma_i \beta_i$. Here we explore what happens when instead of sampling $\beta_{-1} \sim \hat{Q}_{-1}$, we instead just use the variational posterior mean of $\beta_{-1}$ given by $\bar{\beta}_i = q_i \mu_i$. For $(n, p, s_0, \beta_1^0, \beta_j^0, \rho) = (100, 200, 10, \log n, \log n, 0)$, Figure \ref{fig:using_vb_mean} plots a histogram of 100,000 VB posterior samples both with $\beta_{-1} \sim \hat{Q}_{-1}$ and with $\beta_{-1}$ set to be the VB posterior mean. As expected, adding the (independent) variance component from $\beta_{-1}\sim \hat{Q}_{-1}$ noticeably increases the VB posterior uncertainty. This leads to slightly worse performance, as evidenced in Table \ref{Tab:using_vb_mean}, which shows estimates of the coverage over 1000 realisations of the dataset. Using the VB mean yields slight undercoverage with only minimal time savings in this fully uncorrelated setting.

\begin{figure}
    \centering
    \includegraphics[width=0.75\linewidth]{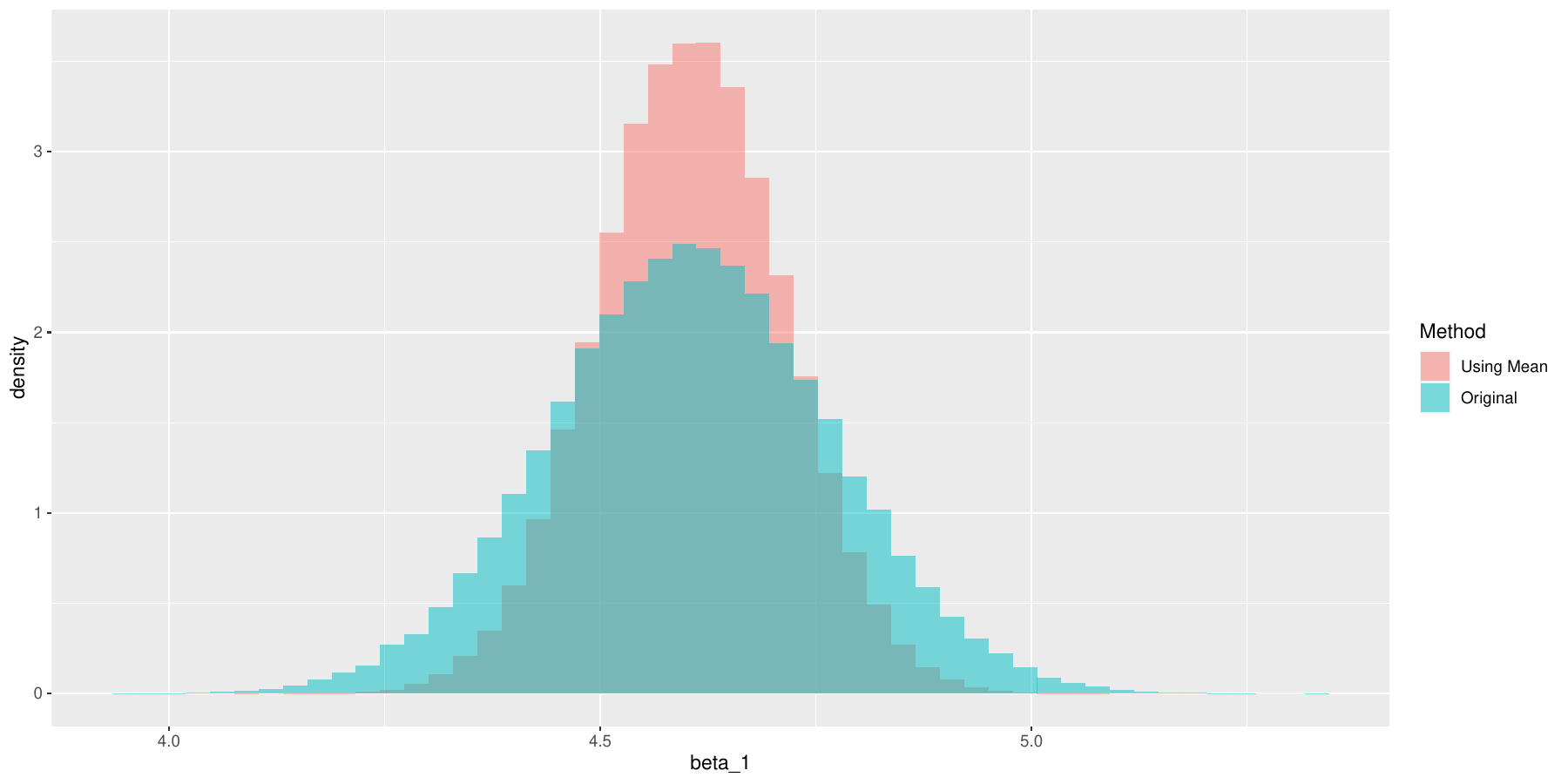}
    \caption{Comparison of the VB posteriors when $\beta_{-1}$ is drawn according to $\hat{Q}_{-1}$ and when it is set equal to the VB posterior mean (with no randomization).}
    \label{fig:using_vb_mean}
\end{figure}


\begin{table}[htbp]
\centering
\footnotesize
\resizebox{\textwidth}{!}{
\begin{tabular}{r|r|cccc}
\toprule
Scenario $(n, p, s_0, \beta_1^0, \beta_j^0, \rho)$ & Method & Cov. & MAE & Length & Time\\
\hline
$(100, 200, 10, \log n, \log n, 0)$ & Normal        & 0.953 & 0.069 & 0.286 $\pm$ 0.015 & 0.663 $\pm$ 0.097\\
                                    & Using VB Mean & 0.904 & 0.069 & 0.255 $\pm$ 0.014 & 0.623 $\pm$ 0.094\\
\bottomrule
\end{tabular}
}
\caption{Assessing the quality of the uncertainty quantification provided by the VB posteriors when $\beta_{-1}$ is drawn according to $\hat{Q}_{-1}$ and when it is set equal to the VB posterior mean (with no randomization).}
\label{Tab:using_vb_mean}
\end{table}

\section{Heuristics}  \label{sec:heuristics}  

We provide some heuristic computations to explain the excellent behaviour of the I-SVB method in the simulations above in settings with possibly significant correlation. Making these heuristics rigorous is an interesting topic for future work. As above, suppose $X_{i\cdot}\sim \mathcal{N}(0,\Sigma)$ with  $\Sigma=\Sigma^{AR}_\rho$ or $\Sigma=\Sigma_\rho$ for a fixed $\rho\in[0,1)$. Recall that the cases where either $\rho=0$ or  $\rho=\rho_n=o(1)$ vanishes sufficiently fast are covered by our theoretical results from Section \ref{sec:BvM_VB_1d}.
 
Let us focus on the one-dimensional case and compare the variance of
the oracle `estimator' $\hat\beta^o_{S_0} := (X_{S_0}^T X_{S_0})^{-1} X_{S_0}^T Y$ with the expected variance of the I-SVB variational distribution. 
 We assume for simplicity that the coordinate of interest $\beta_1$ belongs to $S_0$. 
 We have
 \[ \text{Var}(\hat\beta^o_{S_0} | X) = (X_{S_0}^T X_{S_0})^{-1} \sim (\Sigma_{S_0})^{-1}/n, \]
 where $\Sigma_{S_0}$ is the submatrix of $\Sigma$ obtained by restricting  (both) indices to $S_0$.
If $v^o_1$ denotes the (unconditional) variance of the oracle estimate on the first coordinate, we also have
 \begin{equation}
 v^o_1 \sim [(\Sigma_{S_0})^{-1}]_{11}/n. \label{vor}
 \end{equation}
Recall on the other hand that I-SVB models $\be_1^*$ and $\be_{-1}$ independently by $\be_1^*\sim \mathcal{N}(X_1^TY/\|X_1\|_2^2,\|X_1\|_2^{-2})$ and with $\be_{-1}$ following a mean-field approximation of the posterior of $\be_{-1} | Y$. Since we use the model $\cY = \cW \beta_{-1} + \check{\eps}$ to fit the variational posterior on $\beta_{-1}$,
in these heuristics we assume that the posterior correctly recovers the true support $S_0^- := S_0\backslash\{1\}$ of $\be_{-1}^0$ and that, on $S_0^-$, it mimics the behaviour of the oracle estimate $\hat\beta^o_{S_0^-} = (\cW^T_{S_0^-} \cW_{S_0^-})^{-1}\cW_{S_0^-} \cY$ in the projected model. In practice, there may be some uncertainty in terms of the support, which will then generate  extra variance for I-SVB. As before, 
  $\hat\beta^o_{S_0^-}\approx \mathcal{N}(\be^0_{S_0^-}, (\cW_{S_0^-}^T\cW_{S_0^-})^{-1})$, and the mean-field VB posterior should be close to the best mean-field normal approximation of this Gaussian distribution, which is the normal distribution with diagonal {\it precision} matrix matching the diagonal of the matrix $\cW_{S_0^-}^T\cW_{S_0^-}$, that is the precision corresponding to the oracle variance $(\cW_{S_0^-}^T\cW_{S_0^-})^{-1}$ of $\hat\beta^o_{S_0^-}$ (see e.g. \cite{WangBlei2019}, Lemma 8).
One thus expects the covariance of the mean-field VB distribution $\hat{Q}_{-1}$ on $\be_{-1}$ to be close to the diagonal matrix $\text{Diag}(1/[\cW_{S_0^-}^T\cW_{S_0^-}]_{ii}])$, where we have $\cW_{S_0^-}^T\cW_{S_0^-} = X_{S_0^-}^T(I_n-H)X_{S_0^-} \approx n \Sigma_{S_0^-} - \|X_1\|_2^2 \gamma_{-1} \gamma_{-1}^T$, so that $\text{Diag}(1/[\cW_{S_0^-}^T\cW_{S_0^-}]_{ii}]) \approx \frac{1}{n} \cdot \textrm{Diag}(1/([\Sigma_{S_0^-}]_{ii} - \gamma_i^2))$. 
Since $\be_1=\be_1^*-\sum_{i=2}^p \gamma_i\be_i$ by definition, the overall expected variance of I-SVB on $\be_1$ is, with $\gamma_{-1}^T:=(\gamma_i)_{i\in S_0^-}$,
 \begin{align} 
  v^{\operatorname{\it{I-SVB}}}_1  \approx \frac{1}{\|X_1\|_2^2}+\text{Var}_{\hat{Q}_1}\Big(\sum_{i\in S_0^-} \gamma_i\be_i\Big) &= \frac{1}{\|X_1\|_2^2}+\gamma_{-1}^T\text{Var}_{\hat{Q}_1}(\be_{-1}) \gamma_{-1} \nonumber \\
  & \approx \frac1n \left\{1/\Sigma_{11} +
 \sum_{i\in S_0^-} \gamma_i^2/ ([\Sigma_{S_0^-}]_{ii} - \gamma_i^2)   \right\}, \label{visvb}
 \end{align}  
where we have used the mean-field approximation by taking the above diagonal variance for $\be_{-1}$ under $\hat{Q}_1$.

By the same reasoning based on the (diagonal) precision of the MF VB--approximation matching the diagonal precision of the oracle, the first term $1/(n\Sigma_{11})$ in \eqref{visvb} corresponds to the expected variance of the full MF VB approximation based on the spike-and-slab prior (again assuming the posterior recovers the true support). In particular, while the MF variational posterior has an asymptotic variance which is too small in correlated settings, and therefore corresponding credible intervals typically {\it undercover}, we see in \eqref{visvb} that since the terms of the sum are positive, the variance of I-SVB is always {\it larger}, and hence this results in I-SVB credible intervals having typically {\it better coverage} than their MF counterparts. We now further investigate quantitatively how much variance is gained in \eqref{visvb}  for I-SVB for 
$\Sigma= \Sigma^{AR}_\rho$ and $\Sigma=\Sigma_\rho$.

{\it Case $\Sigma=\Sigma^{AR}_\rho$.} 
Suppose as assumed in the simulations above that we are in a `least-favourable' situation that maximises correlation between features within $S_0$, which corresponds to $S_0=\{1,2,\ldots,s_0\}$ (in case $S_0$ does not have coordinates amongs the first integers, for instance if $S_0$ is drawn uniformly at random in $\{1,\ldots,p\}$, the corresponding covariance of $(X_i)_{i\in S_0}$ will typically be close to the identity as $\rho^q$ is small for large $q$, and so this is closer in spirit to the setting with low correlation already studied  in Section \ref{sec:BvM_VB_1d}). The matrix $\Sigma=\Sigma^{AR}_\rho$ can easily be seen to have a tridiagonal inverse
 and $v^o_1=[\Sigma_{S_0}^{-1}]_{11}/n=n^{-1}/(1-\rho^2)$. On the other hand, by the law of large numbers $\ga_j \approx E(X_{11}X_{1j})=\rho^{j-1}$,  so for $s_0 \geq 2$
 \[ v^{\operatorname{\it{I-SVB}}}_1 \approx 
 \frac1n \left\{ \frac{1}{\Sigma_{11}} +
 \sum_{i\in S_0^-} \frac{\gamma_i^2}{ ([\Sigma_{S_0^-}]_{ii} - \gamma_i^2)}  \right\}= \frac{1}{n}\left\{\frac{1}{1-\rho^2} + \sum_{j = 2}^{s_0-1} \frac{\rho^{2j}}{1-\rho^{2j}}\right\},
 \] 
 where the first term in brackets corresponds to the oracle variance, and the second part of the sum is always greater than 0 and small unless $\rho$ is very close to $1$.
 If $s_0 = 2$ (and the sum is empty) we even have $v^{\operatorname{\it{I-SVB}}}_1\approx v_1^o$. The last display shows that the variance of I-SVB nearly matches that of the oracle, as seen empirically in the simulations above, while the MF method has $v^{MF}_1\approx 1/n$ leading to undercoverage. The same heuristics also apply in two dimensions ($k=2$). In that case, it can be checked by a similar reasoning as above that, denoting by $v_{1:2}^o$ and $v_{1:2}^{\operatorname{\it{I-SVB}}}$ 
 respectively  the variance of the oracle and the expected variance of I-SVB, one has, with $c_\rho=1/(1-\rho^2)$, 
\[ v_{1:2}^o
= (c_\rho/n) \begin{bmatrix}
1 & -\rho \\
-\rho & 1+\rho^2
\end{bmatrix},\qquad 
v_{1:2}^{\operatorname{\it{I-SVB}}}
\approx v_{1:2}^0 + (c_\rho/n) \left[ \begin{matrix} 0& 0\\0&   \sum_{j=2}^{s_0 - 1}\frac{\rho^{2j}}{1-\rho^{2j}}\end{matrix} \right],
  \] 
  so that once again the variance of I-SVB equals the oracle variance plus a term which will be small for moderate values of $\rho$.
These findings are in line with the empirical observations in Section \ref{sec:sims2} above: I-SVB credible ellipses nearly match the oracle confidence ellipses associated to the oracle estimator, as seen in Figure \ref{fig:credible_regions_example}. Also, in this setting, the standard MF approximation, which has a diagonal variance, is necessarily far off the oracle variance as in the above display.
We note also that  the asymptotic variance of $\be_{1:2}^*$ is
\[ \text{Var}(\be_{1:2}^*) = (c_\rho/n) \begin{bmatrix}
1 & -\rho \\
-\rho & 1
 \end{bmatrix}, \]
which is different from $v_{1:2}^o\approx v_{1:2}^{\operatorname{\it{I-SVB}}} $. This confirms that in this case `undoing the transformation' from $\be_{1:2}^*$ to $\be_{1:2}$ has a non-negligible effect in terms of variance, 
as can be seen empirically in Figure \ref{fig:cov_structure_AR}.

{\it Case $\Sigma=\Sigma_\rho$.} By arguing similarly as in the AR case, one first computes explicitly the precision matrix $(\Sigma_{S_0})^{-1}
$ which again is available in explicit form and which leads (we omit the details) to $n\cdot v_1^o\sim [(\Sigma_{S_0})^{-1}]_{11}=e_\rho/(1-\rho)$, with $e_\rho=(1-2\rho+\rho s_0)/(1-\rho+\rho s_0)$ and $e_\rho\approx 1$ for `large' $s_0$. On the other hand, $v_1^{MF} \approx (\Sigma_{11})^{-1}/n=1/n$ and $v_1^{\operatorname{\it{I-SVB}}}\approx (1+(s_0-1)\rho^2/(1-\rho^2))/n$. In this setting with strong correlations, we still have that $v_1^{\operatorname{\it{I-SVB}}}$ is typically significantly larger than $v_1^{MF}$, thus compensating for the undercoverage of classical mean-field  VB. This time, the expected variance of I-SVB does not match in general that of the oracle. One can note though that $v_1^{\operatorname{\it{I-SVB}}}$ scales  (if $\rho$ is bounded away from $0$ so that $\rho^2/(1+\rho)>0$ is a constant) as $s_0(1-\rho)^{-1}/n$, 
which is $s_0$ times the oracle  variance, so in an (idealised) $s_0\to \infty$ limit (and assuming the centering of I-SVB is roughly of correct order) we have {\it over-coverage}; this is less critical compared to the {\it under-coverage} of MF-VB, as then the method is simply somewhat conservative (so credible sets may be slightly too large) but is then expected at least to cover the true parameter, which once again is in line with the empirical results observed in Table \ref{Tab:experiments_1D}.

\section{Theoretical results: Bernstein--von Mises for $k=k_n\geq 1$ and proofs}\label{sec:proofs}

\subsection{Bernstein--von Mises for $k=k_n\geq 1$}

We now provide BvM results for the variational distribution $\hat{Q}_{1:k}$ for $\beta_{1:k}\in\R^k$ with $k=k_n$ possibly growing, showing that it is asymptotically Gaussian centered at $\hat{\beta}_{k} := \beta_{1:k}^0 + \Sigma_k X_{1:k}^T \eps$ with covariance matrix $\Sigma_k := (X_{1:k}^T X_{1:k})^{-1}$, assuming  $X_{1:k}^T X_{1:k}$ is invertible. Setting $L_k = \Sigma_k^{1/2}$, we say the \textit{semiparametric Bernstein-von Mises (BvM)} holds for $\hat{Q}_{1:k}$ if
\begin{align*}
	\left\| \mathcal{L}_{\hat{Q}_{1:k}}(L_k^{-1}(\beta_{1:k} - \hat{\beta}_{k})) -  \mathcal{N}_k(0,I_k) \right\|_{TV}\xrightarrow{P_0} 0
\end{align*}
as $n\to\infty$, where $\|\cdot\|_{TV}$ is the total variation distance on $\R^k$.

We require similar assumptions to the $k=1$ case in Section \ref{sec:BvM_VB_1d}. To obtain concentration of the variational posterior on the nuisance part, we assume 
\begin{align}\label{assum::prior_2_kD}
	\frac{\| \cW_k \|}{p-k} \leq \lambda \leq 2 \bar{\lambda}, \qquad \bar{\lambda}=2\|\cW_k\|\sqrt{\log (p-k)}
\end{align}
and for some $\rho_n\rightarrow \infty$ and $c_0 > 0$, the $s_0$--sparse vector $\beta^0\in \R^p$ satisfies
\begin{align}\label{assum::beta_0_k}
    \phi^{\check{W}}(S_{\beta^0_{-k}}) \geq c_0, \qquad \tilde{\psi}^{\check{W}}_{\rho_n}(S_{\beta^0_{-k}}) \geq c_0,  \qquad s_0=o(n).
\end{align}
We now discuss two choices of priors $g$, namely improper and log-Lipschitz priors such as the Laplace density.
We say that $g$ is \textit{log--Lipschitz} (with respect to the $\|\cdot\|_1$--norm) with constant $c\ge 0$ if, for any $x,y\in\R^k$ and $\|\cdot\|_1$ the $L^1$--norm on $\R^k$,
\[ |g(x)-g(y)|\le c\|x-y\|_1. \] 
For a $k\times k$ matrix $L$, we denote by $\|L\|_{(1)}$ the maximum of the $L^1$--norms of its columns.
We assume that $c$ is small enough so that, with $L_k=\Sigma_k^{1/2}$, as $n\to\infty$,
\begin{equation} \label{cond:c}
c k \| L_k \|_{(1)} =o(1).
\end{equation}
We do not consider multivariate Gaussian $g$ since this already performs less well in both theory and practice when $k=1$.

\begin{theorem}\label{thm:asymptotic_normality_variational_kD} Let $\Pi$ be the  prior \eqref{general_prior_kd} satisfying \eqref{assum::prior_1} and \eqref{assum::prior_2_kD}, and assume $k=o(n)$. Suppose that $\beta^0\in \R^p$ and the design matrix $X$ satisfy \eqref{assum::beta_0_k}, that $\lambda = O(\|\cW_k\|\sqrt{\log (p-k)}/s_0)$ and
 \begin{align}\label{assum::design_matrix_kD}
		 \frac{ \max_{i=k+1, \dots, p}\|H_kX_{i}\|_2 }{\max_{i=k+1, \dots, p}\|(I-H_k)X_{i}\|_2} \rho_n s_0 k^\nu \sqrt{\log p} \rightarrow 0,
	\end{align}
for $\nu \in \{0,1\}$. If either
\begin{itemize}
\item $g \propto 1$ is improper and $\nu=0$;
\item $g$ is log-Lipschitz with constant $c$ satisfying \eqref{cond:c} and $\nu=1$;
\end{itemize}
then the semiparametric BvM holds for the VB posterior $\hat{Q}_{1:k}$.	
\end{theorem} 

 
Theorem \ref{thm:asymptotic_normality_variational_kD} is proved in Section \ref{sec:proof_kd} below, after the $1$--dimensional case. It provides conditions under which the VB posterior for $\beta_{1:k}$ is asymptotically Gaussian and optimal, and in particular motivates our use of credible ellipsoids \eqref{eq:cred_set_multi_d} for uncertainty quantification. Note the result also holds under noise misspecification as in Remark \ref{rem:noise} with the analogous $k$--dimensional condition.

Theorem \ref{thm:asymptotic_normality_variational_kD} allows $k=k_n$ to grow as a power of $n$. The condition $k=o(n)$ is almost minimal, and needed only to ensure invertibility of the relevant matrices such as $\Sigma_k = (X_{1:k}^T X_{1:k})^T$ and thus existence of $L_k$. The matrix $\check{W} \in \R^{(n-k) \times (n-k)}$ has decreasing dimension with $k$, and hence the conditions on the nuisance part $\beta_{-k}$ (e.g. compatibility conditions \eqref{assum::beta_0_k}) become weaker with growing $k$. 
The only further possible restriction on the growth of $k$ is captured by Condition \eqref{assum::design_matrix_kD} (and, for log--Lipschitz priors, \eqref{cond:c}). For an improper prior,  $\nu$ in \eqref{assum::design_matrix_kD} is $0$, so only the ratio of maxima on the left hand--side of the equation possibly depends on $k$; below we consider its behaviour in the case of a Gaussian design $X$, and see that it leads to fairly mild polynomial restrictions on the growth of $k$ for both considered classes of priors.
 
We now briefly discuss without a formal proof these conditions 
in the prototypical setting of a design matrix $X$ with iid $N(0,1)$ entries. 
In particular, we see that Theorem \ref{thm:asymptotic_normality_variational_kD} allows for $k$ growing as some power of the sample size $n$ (see also Section \ref{sec::Additional_result} below for a detailed discussion when $k=1$). With high probability under this design distribution, the matrix norm  $\|\cW_k\|$ is of order $\sqrt{n-k}\sim\sqrt{n}$ 
(since we project on a $n-k$ dimensional space and $k=o(n)$), leading to the same condition on $\lambda$ regardless of $k$. Similarly, the norms $\|(I-H_k)X_{i}\|_2$ in \eqref{assum::design_matrix_kD}  are all of order $\sqrt{n-k}\sim\sqrt{n}$ as well as their maximum. On the other hand, the norms $\|H_kX_{i}\|_2$ in \eqref{assum::design_matrix_kD} are all of order $\sqrt{k}$ (since one projects onto a $k$--dimensional subspace), which is also the order of their maximum.

For improper $g\propto 1$ and $X_{ij} \sim N(0,1)$ iid entries, 
Condition \eqref{assum::design_matrix_kD} with $\nu =0$ simplifies to
\[ s_0=  o\left(\sqrt{ \frac{n}{k\log{p}} } \right), \]
only slightly stronger than $s_0=o(\sqrt{n/\log{p}})$ required if $k$ is a constant. 
Note that for very small sparsities $s_0$ (e.g. logarithmic in $n$), this even allows to take $k$ nearly of the order $n$. 

For log--Lipschitz priors and $X_{ij} \sim N(0,1)$ iid entries, Condition \eqref{assum::design_matrix_kD} with $\nu =1$ is
\[ s_0=  o\left(\sqrt{ \frac{n}{k\log{p}}} \cdot \frac{1}{k} \right). \]
The extra condition \eqref{cond:c} for such priors is verified for the choice of prior parameter $c=1$: indeed, it can be checked for iid standard Gaussian designs that $ \| L_k \|_{(1)} $ is of order $1/\sqrt{n}$ when $k=o(\sqrt{n})$ (in which case the non-diagonal elements of $L_k$ can be seen to have negligible contribution to the norm), which is already needed for the last display to hold. So  \eqref{cond:c}  holds if $c=o( \sqrt{n}/k)$, allowing for $c=1$ (since $k=o(\sqrt{n})$ is required for such priors).

Also, for all considered priors, the choice of prior parameter $\lambda=1$ is allowed, since $\|\cW_k\|\sim \sqrt{n}$ as noted above, so the corresponding condition of the Theorem writes $\lambda=O(\sqrt{n\log{p}}/s_0)$, and since we assume $s_0=\sqrt{n}$ anyways (see the last two displays) for both classes of priors. 

As a summary, for Gaussian designs, both classes of priors for $g$, improper and log--Lipschitz, allow, in the case of a  Gaussian iid design $X$, for both $s_0$ and $k$ growing polynomially with $n$, with a significantly weaker condition for the improper prior,   thus also providing theoretical support for the better performance of the improper prior observed in practice.

\subsection{Proofs for the one dimensional case $k=1$: Lemmas \ref{lem:posterior_form}, \ref{lem:asymptotic_normality_gaussian_variational} and Theorem \ref{thm:asymptotic_normality_variational_1D}}

\begin{proof}[Proof of Lemma \ref{lem:posterior_form}]
Recall that we have the following likelihood decomposition 
\begin{align*}
		\mathcal{L}_n(\beta, Y) \propto  e^{-\frac{1}{2}\|X_1\|^2_2(\beta_1^* - \frac{X_1^TY}{\|X_1\|_2^2})^2} e^{-\frac{1}{2}\|\cW\beta_{-1} - \cY\|_2^2},
	\end{align*}
 where $\cW,\cY$ are defined in \eqref{eq:preprocess}. Using this decomposition and the form of the prior \eqref{general_prior}, we deduce that for any $f_1: \R \rightarrow \R^+$ and $f_2 : \R^{p-1} \rightarrow \R^+$ measurable, we have 
 \begin{align*}
     E(f_1(\beta_1^*)f_2(\beta_{-1}) | Y ) &\propto \int_{\R^p} f_1(\beta_1^*)f_2(\beta_{-1})  e^{-\frac{1}{2}\|X_1\|^2_2(\beta_1^* - \frac{X_1^TY}{\|X_1\|_2^2})^2} e^{-\frac{1}{2}\|\cW\beta_{-1} - \cY\|_2^2} g(\beta_1^*)d\beta_1 dMS_{p-1}(\nu, \lambda)(\beta_{-1}) \\
     &\propto \int_{\R} f_1(u) e^{-\frac{1}{2}\|X_1\|^2_2(u - \frac{X_1^TY}{\|X_1\|_2^2})^2} g(u)du \int_{\R^{p-1}} f_2(\beta_{-1})  e^{-\frac{1}{2}\|\cW\beta_{-1} - \cY\|_2^2}  dMS_{p-1}(\nu, \lambda)(\beta_{-1}).
 \end{align*}
This implies that, under the posterior distribution, $\beta_{-1}$ and $\beta_1^*$ are independent and their distributions are given by \eqref{eq:density_beta_minus_1_1D} and \eqref{density_beta_star_1D} respectively.
The assertion regarding the specific Examples \ref{gaussian_prior} and \ref{improper_prior} can be easily deduced from \eqref{density_beta_star_1D} when $g$ has the specific form chosen in Examples \ref{gaussian_prior} and \ref{improper_prior}.

For the frequentist distribution of $\cY$ under $P_0$, note that $\cY= P^TY = P^T(X\beta^0 +\eps) = P^T(X_{-1}\beta^0_{-1} +\eps) = \cW\beta^0_{-1} + P^T\eps$ and thus $\cY \sim \mathcal{N}(\cW\beta^0_{-1}, P^T P) = \mathcal{N}(\cW\beta^0_{-1}, I_{n-1})$.
\end{proof}

\begin{proof}[Proof of Theorem \ref{thm:asymptotic_normality_variational_1D}]
{\bf Step 1: Convergence of $\beta_{-1}$ to $\beta^0_{-1}$.}  Recall that the VB posterior $\hat{Q}_{-1}$ for $\beta_{-1}$ is the mean-field approximation of $\beta_{-1}$ based on the variational family $\mathcal{Q}_{-1}$ in \eqref{eq:variational_class}. By Lemma \ref{lem:posterior_form}, the posterior for $\beta_{-1}$ equals the posterior distribution in the linear regression model $\cY = \cW\beta_{-1} + \check{\eps}$ with model selection prior $\beta_{-1} \sim MS_{p-1}(\nu, \lambda)$. This is exactly the setting studied in \cite{Ray_Szabo_2020}, applied to this transformed linear regression model. Since $\beta^0$ satisfies the compatibility conditions \eqref{assum::beta_0} and the prior satisfies \eqref{assum::prior_1}, \eqref{assum::prior_2} and $\lambda = O(\|\cW\|\sqrt{\log (p-1)}/s_0)$, Theorem 1 in \cite{Ray_Szabo_2020} shows that for $M>0$ large enough depending only on the prior,
	\begin{align}\label{eq:convergence_beta_minus_1}
		\hat{Q}_{-1}\left( \left\{ \beta_{-1} \in \R^{p-1} : \|\beta_{-1}- \beta_{-1}^0\|_1 \leq \eps_n \right\} \right) = 1+o_{P_0}(1),
	\end{align}
where $\eps_n=  \frac{M\rho_ns_0\sqrt{\log p}}{\|\cW\|c_0^3}$. By Remark B.1 in \cite{Ray_Szabo_2020}, Theorem 1 of that paper still applies when the noise $\eps$ in \eqref{eq:linear_regression} is non-Gaussian as long as \eqref{eq:noise_miss} holds. Thus the conclusion \eqref{eq:convergence_beta_minus_1} continues to hold under the more general noise condition of Remark \ref{rem:noise}. The remainder of the proof does not depend on the distribution of $\eps$.

Using $\|\cW\| = \|(I-H)X_{-1}\| $ and \eqref{assum::design_matrix}, the rate $\eps_n$ satisfies
    \begin{align}\label{eq:convergence_beta_minus_k_1}
		\eps_n \|X_1\|_2 \max_{i = 2, \dots, p} |\gamma_i| \rightarrow 0.
	\end{align} 
	
{\bf Step 2: Asymptotic normality.} 
Define $A_n=\{\beta \in \R^{p} : \|\beta_{-1}- \beta_{-1}^0\|_1 \leq \eps_n\}$ and the conditioned VB posterior
$$\overline{Q}:=\frac{\hat{Q}(\cdot\cap A_n)}{\hat{Q}(A_n)}.$$
By a standard bound (e.g. p142 of \cite{vdVaart1998}) and \eqref{eq:convergence_beta_minus_1}, $\| \hat{Q} -  \overline{Q}\|_{TV} \le 2\hat{Q}(A_n^c)=o_{P_0}(1).$ It thus suffices to prove the result for $\overline{Q}$. 
Let $\hat{\beta}_1 = \beta_1^0 + X_1^T\eps/\|X_1\|_2^2$, so that it satisfies \eqref{eq:efficient_centering} with the remainder term identically zero.
Writing $\mathcal{L}_{\overline{Q}}(U)$ for the law of a random variable $U$ under $\overline{Q}$ and recalling that $\beta_1^* = \beta_1 + \sum_{i=2}^p \gamma_i \beta_i$,
\begin{align*}
\| \mathcal{L}_{\overline{Q}}(\be_1) - N(\hat\be_1, 1/\|X_1\|_2) \|_{TV} & 
= \left\| \mathcal{L}_{\overline{Q}} \left(\be_1^*- \sum_{i=2}^p \ga_i\be_i^0 - \sum_{i=2}^p \ga_i(\be_i-\be_i^0) \right) - N(\hat\be_1, 1/\|X_1\|_2) \right\|_{TV} \\
& \le \left\| \mathcal{L}_{\overline{Q}} \left(\be_1^*- \sum_{i=2}^p \ga_i\be_i^0 - \sum_{i=2}^p \ga_i(\be_i-\be_i^0)\right) - 
 \mathcal{L}_{\overline{Q}}(\be_1^*- \sum_{i=2}^p \ga_i\be_i^0 )
\right\|_{TV} \\
& \ \ +  \left\| \mathcal{L}_{\overline{Q}}\left(\be_1^*- \sum_{i=2}^p \ga_i\be_i^0\right) - N(\hat\be_1, 1/\|X_1\|_2^2) \right\|_{TV}. 
\end{align*}
Note that $\beta_{-1}$ and $\be_1^*$ are independent under $\overline{Q}$ since they are independent under $\hat{Q}$ and the conditioning event $A_n$ involves only $\beta_{-1}$. Set $T=T(\be_{-1}):= - \sum_{i=2}^p \ga_i(\be_i-\be_i^0)$, which is independent from $\beta_1^*$ under $\hat{Q}$.

Denoting by $q_1^*$ the true posterior density of $\be_1^*$ (which equals the VB one under both $\hat{Q}$ and $\overline{Q}$) and $F_T$ the distribution of $T$ under $\overline{Q}$, the first term in the last display equals 
\begin{align*}
\| \mathcal{L}_{\overline{Q}}(\be_1^* + T) - 
 \mathcal{L}_{\overline{Q}}(\be_1^* ) \|_{TV} 
 & = \int \left| \int q_1^*(v-u) dF_T(u) - q_1^*(v) \right| dv
\\
& \le  \int  \int \left| q_1^*(v-u)- q_1^*(v) \right| dv dF_T(u) 
\end{align*}
where the last line follows from Fubini's theorem. For fixed $u\in\R$, the 
 inner integral in that line equals the TV--distance between $\mathcal{L}_{\overline{Q}}(\be_1^*+u)$ and $\mathcal{L}_{\overline{Q}}(\be_1^*)$. Since $\overline{Q}$ is supported on $A_n$, for any $u$ in the support of $F_T$,
\[ | u | \le \max_{i=2,\ldots,p} |\ga_i| \varepsilon_n=:\delta_n, \]
using H\"older's inequality. Thus,
\[  \| \mathcal{L}_{\overline{Q}}(\be_1^* + T) - 
 \mathcal{L}_{\overline{Q}}(\be_1^* ) \|_{TV}
 \le \sup_{|u|\le \delta_n} \| \mathcal{L}_{\overline{Q}}(\be_1^*+u) - \mathcal{L}_{\overline{Q}}(\be_1^*) \|_{TV}.
 \]
In summary, using that $\hat{\beta}_1 + \sum_{i=2}^{p}\gamma_i\beta_i^0 = \frac{X_1^T Y}{\|X_1\|_2^2}$ under $P_0$ and invariance of the TV--distance under shifts,
\begin{equation}\label{eq:I and II}
\begin{split}
& \| \mathcal{L}_{\overline{Q}}(\be_1)  - N(\hat\be_1, 1/\|X_1\|_2^2) \|_{TV}  \\ 
& \quad  \le \sup_{|u|\le \delta_n} \| \mathcal{L}_{\overline{Q}}(\be_1^*+u) - \mathcal{L}_{\overline{Q}}(\be_1^*) \|_{TV} \ +\  \left\| \mathcal{L}_{\overline{Q}} \left(\be_1^*- \sum_{i=2}^p \ga_i\be_i^0 \right) - N(\hat\be_1, 1/\|X_1\|_2^2) \right\|_{TV} \\
& \quad = \sup_{|u|\le \delta_n} \| \mathcal{L}_{\overline{Q}}(\be_1^*+u) - \mathcal{L}_{\overline{Q}}(\be_1^*) \|_{TV} \ +\  \| \mathcal{L}_{\overline{Q}}(\be_1^*) - N( X_1^T Y/\|X_1\|_2^2 , 1/\|X_1\|_2^2) \|_{TV}  \\
& \ \ =:  (I)+(II). 
\end{split}
\end{equation}
We now separately bound the terms (I) and (II). For notational convenience, denote
\[ z := \frac{X_1^T Y}{\|X_1\|_2^2}. \]



Term (II): Recall the density $q_1^*$ of $\beta_1^*$ under the VB posterior $\hat{Q}$, and hence also $\overline{Q}$, equals the exact posterior for $\beta_1^*$. Using the explicit form of this posterior density given in Lemma \ref{lem:posterior_form},
\begin{align*}
(II) & =\int \left| \frac{e^{-\|X_1\|_2^2(w-z)^2/2} g(w)}{ \int e^{-\|X_1\|_2^2(w'-z)^2/2} g(w')dw'}  - \frac{e^{-\|X_1\|_2^2(w-z)^2/2}}{\sqrt{(2\pi)/\|X_1\|_2^2}} \right| dw\\
& = \int \left| \frac{e^{-v^2/2} g(z+v/\|X_1\|_2)}{ \int e^{-u^2/2} g(z+u/\|X_1\|_2)du}  - \frac{e^{-v^2/2}}{\sqrt{2\pi}} \right| dv\\
& = \int \phi(v) \left| \frac{g(z+v/\|X_1\|_2)}{ \int \phi(u) g(z+u/\|X_1\|_2)du}  - 1 \right| dv,
\end{align*}
with $\phi$ the standard Gaussian density. In the notation of Lemma \ref{techtv}, the last display equals $\mathcal{D}_{g,\|X_1\|_2,z}$. By Lemma \ref{techtv}, to show that this term is $o(1)$, it suffices to show that 
\[ \mathcal{H}:= \mathcal{H}_{g,\|X_1\|_2,z} = \int\phi(v)  \left| \frac{g(z+v/\|X_1\|_2)}{g(z)} - 1 \right| dv = o(1). \]
Set $K_n:=[-\|X_1\|_2m_n/c,\|X_1\|_2m_n/c]$ for $m_n\to 0$ slow enough that $\|X_1\|_2m_n/c \to \infty$, which is possible since $c=o(\|X_1\|_2)$ by assumption.
Then
\begin{align*}
\mathcal{H} = \int_{K_n} \phi(v)  \left| \frac{g(z+v/\|X_1\|_2)}{g(z)} - 1 \right| dv 
+ \int_{K_n^c} \phi(v)  \left| \frac{g(z+v/\|X_1\|_2)}{g(z)} - 1 \right| dv =: (i)+(ii).
\end{align*}
Using that $g$ is log-Lipschitz together with $|e^x-1|\le 2|x|$ for small enough $x$ yields, using that $c|v|/\|X_1\|_2\le m_n=o(1)$ for $v\in K_n$,
\[ (i)\le 2 \int_{K_n} \phi(v)  \frac{c|v|}{\|X_1\|_2} dv\le \frac{2c}{\|X_1\|_2} \int |v|\phi(v)dv =o(1),\]
since $c = o(\|X_1\|_2)$ by assumption. For term (ii), one first uses the triangle inequality $|a-1|\le |a|+1$ and next uses  the log-Lipschitz property of $g$ to get
\[ (ii) \le \int_{K_n^c}  \frac{e^{-v^2/2}}{\sqrt{2\pi}} e^{c|v|/\|X_1\|_2}  dv 
+  \int_{K_n^c} \phi(v)dv. \]
The second term of the last display goes to $0$ as $n\to\infty$ by the definition of $K_n$. Completing the square in the first term then gives
\begin{align*}
(ii) & \le 2\int_{\|X_1\|_2m_n/c}^\infty e^{-\frac{1}{2}(v-c/\|X_1\|_2)^2}e^{\frac{1}{2}c^2/\|X_1\|_2^2}dv +o(1) \\
& = 2\sqrt{2\pi}e^{o(1)} \overline{\Phi}\left(\frac{\|X_1\|_2m_n}{c}-\frac{c}{\|X_1\|_2} \right) + o(1).
\end{align*} 
Since $\|X_1\|_2m_n/c\to\infty$ and $c/\|X_1\|_2=o(1)$, the last display is $o(1)$. Combining the previous bounds leads to $\mathcal{H}=o(1)$, which in turns implies $(II)=o(1)$.

Term (I):  Using again Lemma \ref{lem:posterior_form}, and noting that the normalising constants of the two densities at stake are the same (as they only differ by a translation),
\begin{align*}
(I) & =  \sup_{|u|\le \delta_n} 
\int \left| \frac{e^{-\|X_1\|_2^2(w-u-z)^2/2} g(w-u)}{ \int e^{-\|X_1\|_2^2(w'-u-z)^2/2} g(w'-u)dw'}  - \frac{e^{-\|X_1\|_2^2(w-z)^2/2}g(w)}{\int e^{-\|X_1\|_2^2(w'-z)^2/2}g(w')dw'} \right| dw\\
& = 
 \sup_{|u|\le \delta_n} \int \left| \frac{e^{-(v-u\|X_1\|_2)^2/2} g(z-u+v/\|X_1\|_2)}{ \int e^{-t^2/2} g(z+t/\|X_1\|_2)dt}  - \frac{e^{-v^2/2}g(z+v/\|X_1\|_2)}{\int e^{-t^2/2}g(z+t/\|X_1\|_2)dt } \right| dv\\
& =  \sup_{|u|\le \delta_n}  \frac{g(z)}{D} \int \phi(v) \left| e^{uv\|X_1\|_2-u^2\|X_1\|_2^2/2}\frac{g(z-u+v/\|X_1\|_2)}{ g(z) }  - \frac{g(z+v/\|X_1\|_2)}{g(z)} \right| dv,
\end{align*}
where $D=D_{g,\|X_1\|_2,z}=\int \phi(t)g(z+t/\|X_1\|_2)dt$. Let $J(u)$ denote the integral in the last display. We have already proved above in the study of term (II) that $\mathcal{H}=o(1)$, which by Lemma \ref{techtv} implies that $g(z)/D\to 1$, so that it suffices to show $\sup_{|u|\leq \delta_n} J(u)=o(1)$ to obtain $(I) = o(1)$.
 
Similar to the term (II), split $J(u)$ into two parts $J_1(u)$ and $J_2(u)$ with integration domains $K_n$ and $K_n^c$, respectively, where $K_n:=[-\|X_1\|_2m_n/c,\|X_1\|_2m_n/c]$ with $m_n \to 0$ slow enough that $\|X_1\|_2m_n/c\to\infty$. For $v\in K_n$,
 \[ \frac{g(z-u+v/\|X_1\|_2)}{ g(z) } =e^{O(c\delta_n)+O(m_n)} \qquad \text{and} \qquad \frac{g(z+v/\|X_1\|_2)}{g(z)}=e^{O(m_n)}, 
 \]
uniformly over $|u| \leq \delta_n$. Turning to the integral in the second last display restricted to $K_n$ and putting all terms in exponential form,
\begin{align*}
\sup_{|u|\leq\delta_n} J_1(u) & \leq \int_{K_n} \phi(v) \left| e^{ O(|v| \delta_n\|X_1\|_2)+O(\delta_n^2\|X_1\|_2^2)+O(c\delta_n)+O(m_n)} - e^{O(m_n)} \right| dv.
\end{align*} 
Note that $\delta_n \|X_1\|_2 \to 0$ by \eqref{eq:convergence_beta_minus_k_1}, while $c\delta_n=(c/\|X_1\|_2)(\delta_n\|X_1\|_2) \to 0$ since also $c=o(\|X_1\|_2)$. For all $v \in K_n$, it holds that $|v| \delta_n\|X_1\|_2 \leq m_n (\|X_1\|_2/c) \delta_n\|X_1\|_2$. Taking $\frac{c}{\|X_1\|_2} \ll m_n \ll \min( \frac{c}{\|X_1\|_2} \frac{1}{\delta_n\|X_1\|_2},1)$, there is a valid choice $m_n \to 0$ such that $\|X_1\|_2m_n/c\to\infty$ and $m_n (\|X_1\|_2/c) \delta_n\|X_1\|_2 \to 0$. Thus, all the terms in the exponent in the last display are vanishing as $n\to\infty$. Using the inequality $|e^x-e^y|\le 2|x|+2|y|$ for $|x|,|y|$ sufficiently small then yields $\sup_{|u|\leq\delta_n} J_1(u) =o(1)$.
 
It remains to deal with $J_2(u)$.  Similar to its counterpart (ii) for the term (II) above, using that $g$ is log-Lipschitz,
 \[ \sup_{|u|\leq\delta_n} J_2(u) \leq \int_{K_n^c} \phi(v) \left( 
 e^{c\delta_n+ (c/\|X_1\|_2 + \delta_n\|X_1\|_2)|v|} + e^{c|v|/\|X_1\|_2} \right) dv.
 \] 
Since $\delta_n\|X_1\|_2=o(1)$ and $c/\|X_1\|_2 =o(1)$, a similar argument to that used to bound (ii) above shows that the last integral is $o(1)$, and hence $\sup_{|u|\leq\delta_n} J_2(u) = o(1)$. Combining everything yields $(I) = o(1)$.

Putting together the previous bounds, we have shown $\| \mathcal{L}_{\hat{Q}}(\be_1)  - N(\hat\be_1, 1/\|X_1\|_2^2) \|_{TV}=o(1)$, which concludes the proof for the log-Lipschitz prior. 
\end{proof}

\begin{proof}[Proof of Lemma \ref{lem:asymptotic_normality_gaussian_variational}]
The Gaussian prior $g\sim \mathcal{N}(0,\sigma_n^2)$ is not log-Lipschitz and so the BvM proof requires some modifications in this case. The arguments until \eqref{eq:I and II} in the proof of Theorem \ref{thm:asymptotic_normality_variational_1D} above do not require $g$ to be log-Lipschitz, and thus \eqref{eq:I and II} holds also for Gaussian $g$. It thus suffices to show $(I)$ and $(II)$ vanish in \eqref{eq:I and II}, which we do by direct arguments.

Recall that by Lemma \ref{lem:posterior_form}, $\beta_1^*|Y  \sim \mathcal{N}\left(\frac{\sigma_n^2 X_1^TY}{\|X_1\|_2^2\sigma_n^2 +1}, \frac{\sigma_n^2}{\|X_1\|_2^2\sigma_n^2 +1} \right)$ under the posterior, and hence also under $\hat{Q}$ by \eqref{intuitive_approximation}. Using Pinsker's inequality and the exact formula for the Kullback-Leibler divergence between Gaussian distributions,
\begin{align*}
\tfrac{1}{2} (I)^2 & \leq \sup_{|u|\le \delta_n} \operatorname{KL} \left( \mathcal{N}\left(\frac{\sigma_n^2 X_1^TY}{\|X_1\|_2^2\sigma_n^2 +1} + u, \frac{\sigma_n^2}{\|X_1\|_2^2\sigma_n^2 +1} \right)  \left\| \mathcal{N}\left(\frac{\sigma_n^2 X_1^TY}{\|X_1\|_2^2\sigma_n^2 +1} , \frac{\sigma_n^2}{\|X_1\|_2^2\sigma_n^2 +1} \right) \right.  \right) \\
& = \sup_{|u|\le \delta_n} u^2 \frac{\|X_1\|_2^2\sigma_n^2 +1}{2\sigma_n^2} \leq \delta_n^2 \frac{\|X_1\|_2^2\sigma_n^2 +1}{2\sigma_n^2}.
\end{align*}
Again by Pinsker's inequality,
\begin{align*}
(II)^2 & \leq 2\operatorname{KL} \left( \mathcal{L}_{\overline{Q}}(\be_1^*) \left\| N( X_1^T Y/\|X_1\|_2^2 , 1/\|X_1\|_2^2) \right. \right) \\
& =   - \frac{1}{\|X_1\|_2^2 \sigma_n^2 + 1} + \frac{(X_1^T Y)^2}{\|X_1\|_2^2(\|X_1\|_2^2 \sigma_n^2+1)^2} +\log \left( \frac{\|X_1\|_2^2 \sigma_n^2+1}{\|X_1\|_2^2 \sigma_n^2}
 \right) .
\end{align*}
Using $\log(1+x) = O(x)$ as $|x|\to 0$, that $\|X_1\|_2\sigma_n \to \infty$ by assumption and that $\|X_1\|_2\delta_n \to 0$ by \eqref{eq:convergence_beta_minus_k_1},
\begin{align*}
(I)^2 + (II)^2 \lesssim \delta_n^2 \|X_1\|_2^2 + \frac{(X_1^T Y)^2}{\|X_1\|_2^6 \sigma_n^4}  + O(1/(\|X_1\|_2^2\sigma_n^2)) = \frac{(X_1^T Y)^2}{\|X_1\|_2^6 \sigma_n^4} + o(1),
\end{align*}
so that it remains to show the last term vanishes. Since $X_1^T Y = \|X_1\|_2^2 \beta_1^0 + \|X_1\|_2^2 \sum_{i=2}^p \gamma_i \beta_i^0 + X_1^T \eps $ under $P_0$,
\begin{align*}
\frac{(X_1^T Y)^2}{\|X_1\|_2^6 \sigma_n^4} \lesssim \frac{\|X_1\|_2^4 |\beta_1^0|^2 + \|X_1\|_2^4 \max_{i} |\gamma_i|^2 \|\beta_{-1}^0\|_1^2 + (X_1^T \eps)}{\|X_1\|_2^6 \sigma_n^4} = \frac{(X_1^T \eps)^2}{\|X_1\|_2^6 \sigma_n^4} + o(1)
\end{align*}
using the assumptions on $\sigma_n$. But $X_1^T \eps /(\|X_1\|_2^3 \sigma_n^2) \sim \mathcal{N}(0,\tfrac{1}{\|X_1\|_2^4\sigma_n^4})$, so the last term is $o_{P}(1)$ since $\|X_1\|_2\sigma_n \to \infty$. Combining the above, we have $(I)^2 + (II)^2 = o_P(1)$ and hence \eqref{eq:I and II} is $o_P(1)$ as desired.
\end{proof}

\subsection{Proofs of BvM theorems for general $k\ge 1$} \label{sec:proof_kd}

We first need the following Lemma, which is the $k-$dimensional version of Lemma \ref{lem:posterior_form}.
\begin{lemma}[Posterior distribution for $k\geq 1$]\label{lem:posterior_form_kD}
 	Let $\Pi$ be the prior \eqref{general_prior_kd}. Then under the posterior distribution, $\beta_{-k}$ and $\beta_{1:k}^*$ are independent, and their distributions takes the form
 	\begin{align}\label{eq:density_beta_minus_k_kD}
 		d\pi(\beta_{-k} | Y) &\propto {e^{-\frac{1}{2} \| \cW_k\beta_{-k} - \cY_k \|^2_2}} dMS_{p-k}(\nu, \lambda)(\beta_{-k}),
 	\end{align}	
 	\begin{align}\label{density_beta_star_kD}
 		d\pi(\beta_{1:k}^* |Y) &\propto e^{-\frac{1}{2} (\beta_{1:k}^* - \Sigma_k X_{1:k}^TY )^T \Sigma_k^{-1} (\beta_{1:k}^* - \Sigma_k X_{1:k}^TY)} g(\beta_{1:k}^*)d\beta_{1:k}^*.
 	\end{align}
  Under the frequentist assumption $Y \sim \mathcal{N}_n( X\beta^0,I_n)$, we have $\cY_k \sim\mathcal{N}(\cW_k\beta^0_{-k} , I_{n-k})$.
 \end{lemma}

\begin{proof}[Proof of Lemma \ref{lem:posterior_form_kD}]
Recall that we have the following decomposition for the likelihood
\begin{align*}
		\mathcal{L}_n(\beta, Y) \propto \exp\left\{-\frac{1}{2}\left\|H_kY - X_{1:k} \beta_{1:k}^*\right\|_2^2 \right\} \cdot \exp\left\{-\frac{1}{2}\|\cY_k - \cW_k \beta_{-k}\|_2^2 \right\}.	\end{align*}
 Using this decomposition and the form of the prior, we deduce that for any $f_1: \R^k \rightarrow \R^+$ and $f_2 : \R^{p-k} \rightarrow \R^+$ measurable, we have 
 \begin{align*}
     E(f_1(\beta_{1:k}^*)f_2(\beta_{-k}) | Y ) &\propto \int_{\R^p} f_1(\beta_{1:k}^*)f_2(\beta_{-k})  e^{-\frac{1}{2}\left\|H_kY - X_{1:k} \beta_{1:k}^*\right\|_2^2} e^{-\frac{1}{2}\|\cY_k - \cW_k \beta_{-k}\|_2^2}	g(\beta_{1:k}^*) d\beta_{1:k}^* dMS_{p-k}(\nu, \lambda)(\beta_{-k}) \\
     &\propto \int_{\R^k} f_1(\beta_{1:k}^*) e^{-\frac{1}{2}\left\|H_kY - X_{1:k} \beta_{1:k}^*\right\|_2^2} g(\beta_{1:k}^*)d\beta_{1:k}^* \\ &\quad \times \int_{\R^{p-k}} f_2(\beta_{-k})  e^{-\frac{1}{2}\|\cY_k - \cW_k \beta_{-k}\|_2^2} dMS_{p-k}(\nu, \lambda)(\beta_{-k}).
 \end{align*}
This implies that, under the posterior distribution, $\beta_{-k}$ and $\beta_{1:k}^*$ are independent and their distributions are given by \eqref{eq:density_beta_minus_k_kD} and \eqref{density_beta_star_kD}, respectively.

For the frequentist distribution of $\cY$ under $P_0$, note that $\cY_k= P^T_kY = P^T_k(X_{-k}\beta^0_{-k} +\eps) = \cW_k\beta^0_{-k} + P^T_k\eps$ and thus $\cY_k \sim \mathcal{N}(\cW_k\beta^0_{-k}, P^T_k P_k) = \mathcal{N}(\cW_k\beta^0_{-k}, I_{n-k})$ under the frequentist assumption.
\end{proof}

\begin{proof}[Proof of Theorem \ref{thm:asymptotic_normality_variational_kD}] 
The proof follows the same ideas as the one-dimensional case $k=1$ in Theorem \ref{thm:asymptotic_normality_variational_1D}, but with more complicated notation and dependencies.

{\bf Step 1: Convergence of $\beta_{-k}$ to $\beta^0_{-k}$.}  
Recall from Section \ref{sec:dimk} that the variational posterior $\hat{Q}_{-k}$ of $\beta_{-k}$ is the mean-field approximation of the posterior distribution of $\beta_{-k}$, in a linear regression model with design matrix  $\cW_k$  and model selection prior $MS_{p-k}(\nu, \lambda)$. Since $\beta^0$ satisfies \eqref{assum::beta_0_k}, the prior satisfies \eqref{assum::prior_1}, \eqref{assum::prior_2_kD} and $\lambda = O(\|\cW_k\|\sqrt{\log (p-k)}/s_0)$, applying Theorem 1 in \cite{Ray_Szabo_2020} gives that there exists $M$ large enough such that, with $\eps_n= M \frac{\rho_ns_0\sqrt{\log (p-k)}}{\|\cW_k\|c_0^3}$,
	\begin{equation} \label{eq:rateminusk}
	\hat{Q}_{-k}\left(\{\beta_{-k} \in \R^{p-k} : \|\beta_{-k}- \beta_{-k}^0\|_1 \leq \eps_n \}\right) = 1+o_{P_0}(1).
	\end{equation}
	Using that $\|\cW_k\| = \|(I-H_k)X_{-k}\| $ and \eqref{assum::design_matrix_kD},
    \begin{align}\label{eq:condepsk}
		\delta_n:= \eps_n \max_{i=k+1, \dots, p}\|H_kX_{i}\|_2 =\eps_n \| H_k X_{-k} \| \rightarrow 0.
	\end{align}

{\bf Notation: some useful sequences.} 
Set $\delta_n':=k\delta_n$ and $\zeta_n:=ck\|L_k\|_{(1)}$, which satisfy $\delta_n'\to 0$ and $\zeta_n \to 0$ by Conditions \eqref{assum::design_matrix_kD} and \eqref{cond:c}, respectively. Further define
\begin{equation}
R_n= \min\left( 1/{\sqrt{\zeta_n}} , 1/{\sqrt{\delta_n'}} \right),
\end{equation}
which satisfies $R_n \to \infty$ as $n\to\infty$. For $R_n$ as in the last display, set
\begin{align}
m_n & := \zeta_n R_n=c k \|L_k\|_{(1)}R_n,  \qquad \qquad M_n  := R_n k. \label{Mn}
\end{align}
In particular, $m_n \leq \sqrt{\zeta_n}\to 0$, while $k=o(M_n)$, since $R_n$ is a diverging sequence.
	
{\bf Step 2: Convergence in TV--norm.}  
Define $A_n=\{\beta \in \R^{p} : \|\beta_{-k} - \beta_{-k}^0\|_1 \leq \eps_n\}$ and the conditioned VB posterior
$$\overline{Q}:=\frac{\hat{Q}(\cdot\cap A_n)}{\hat{Q}(A_n)}.$$
By a standard bound (e.g. p142 of \cite{vdVaart1998}) and \eqref{eq:rateminusk}, $\| \hat{Q} -  \overline{Q}\|_{TV} \le 2\hat{Q}(A_n^c)=o_{P_0}(1).$ It thus suffices to prove the result for $\overline{Q}$. 
Let $\hat{\beta}_{k} := \beta_{1:k}^0 + \Sigma_k X_{1:k}^T \eps$, where we recall that  $\Sigma_k = (X_{1:k}^T X_{1:k})^{-1}$.
Writing $\mathcal{L}_{\overline{Q}}(U)$ for the law of a random variable $U$ under $\overline{Q}$, recalling that $\beta^*_{1:k} := \beta_{1:k} + \Sigma_k X_{1:k}^T X_{-k} \beta_{-k}$ and denoting $\Gamma_k:=\Sigma_k X_{1:k}^T X_{-k}$,
\begin{align*} 
\| & \mathcal{L}_{\overline{Q}}(\be_{1:k}) - N(\hat\be_k, \Sigma_k) \|_{TV} 
= \| \mathcal{L}_{\overline{Q}}(L_k^{-1}\be_{1:k}) - N(L_k^{-1}\hat\be_k, I_k) \|_{TV} \\
& = \| \mathcal{L}_{\overline{Q}}\left(
L_k^{-1}\left( \be_{1:k}^* - \Gamma_k(\be_{-k}-\be_{-k}^0) \right) 
\right) - N\left(
L_k^{-1}\left( \hat\be_k + \Gamma_k\be_{-k}^0 \right), I_k
\right) \|_{TV}, 
\end{align*}   
using the invariance of the TV distance under shifts and dilations. Using the definition of $\hat\be_k$, one notes
$\hat\be_k + \Gamma_k\be_{-k}^0=\hat\be_k + \Sigma_kX_{1:k}^TX_{-k}\be_{-k}^0=\Sigma_kX_{1:k}^TY$. For notational convenience, write
\begin{align*}
 z & := \Sigma_k X_{1:k}^T Y, \qquad \qquad T  := -L_k^{-1} \Gamma_k(\be_{-k} - \be_{-k}^0). 
\end{align*}  
Using these definitions and the triangle inequality,
\begin{align*}   
\| \mathcal{L}_{\overline{Q}}(\be_{1:k}) - N(\hat\be_k, \Sigma_k) \|_{TV} & 
\leq  \left\| 
\mathcal{L}_{\overline{Q}} \left(
L_k^{-1} \be_{1:k}^* + T \right) -  \mathcal{L}_{\overline{Q}} \left(
L_k^{-1}\be_{1:k}^*  \right) 
\right\|_{TV} \\
& \quad + \left\|    
\mathcal{L}_{\overline{Q}} \left(
L_k^{-1} \be_{1:k}^*  \right) 
- N\left(L_k^{-1} z, I_k
\right)  
\right\|_{TV}. 
\end{align*}
Note that $\beta_{-k}$ and $\be_{1:k}^*$ are independent under $\overline{Q}$ since they are independent under $\hat{Q}$ and the conditioning event $A_n$ involves only $\beta_{-k}$. In particular the above shift $T$ is independent from $\beta_{1:k}^*$ under $\hat{Q}$ and, thus, $\overline{Q}$. 

Denoting by $q_k^*$ the true posterior density of $\be_{1:k}^*$ (which equals the VB one under both $\hat{Q}$ and $\overline{Q}$) and $F_T$ the distribution of $T$ under $\overline{Q}$, the first term on the right hand side of   the last display equals
\begin{align*}
\| \mathcal{L}_{\overline{Q}}(L_k^{-1} \be_{1:k}^* + T) - 
 \mathcal{L}_{\overline{Q}}(L_k^{-1} \be_{1:k}^* ) \|_{TV} 
 & = \int \left| \int q_k^*(L_k(v-u)) dF_T(u) - q_k^*(L_k v) \right| dv
\\
& \le  \int  \int \left| q_k^*(L_k(v-u))- q_k^*(L_k v) \right| dv dF_T(u),
\end{align*}
where the last line follows from Fubini's theorem. For fixed $u\in\R^k$, the 
 inner integral in that line equals the TV--distance between $\mathcal{L}_{\overline{Q}}(L_k^{-1} \be_{1:k}^*+u)$ and $\mathcal{L}_{\overline{Q}}(L_k^{-1} \be_{1:k}^*)$. Since $\overline{Q}$ is supported on $A_n$, for any $u$ in the support of $F_T$, there exists some $b\in A_n$ with
\begin{align*}
 \| u \|_2^2 & = \|L_k^{-1} \Gamma_{-k}(b_{-k}-\be_{-k}^0) \|_2^2 \\ 
 & = (b_{-k}-\be_{-k}^0)^T \Gamma_{-k} L_k^{-1} L_k^{-1}\Gamma_{-k}(b_{-k}-\be_{-k}^0)\\
 & = (b_{-k}-\be_{-k}^0)^T X_{-k}^T X_{1:k} \Sigma_k X_{1:k}^T X_{-k}  (b_{-k}-\be_{-k}^0)\\
 & = (b_{-k}-\be_{-k}^0)^T X_{-k}^T H_k X_{-k}  (b_{-k}-\be_{-k}^0).
\end{align*}  
Since $H_k$ is a projection $H_k^T H_k = H_k$, so that the last display equals $\| H_k X_{-k}(b_{-k}-\be_{-k}^0)\|_2^2$. Using the standard bound $\| Ay\|_2 \le \|A\| \|y\|_1$, where $\|A\|$ is the maximum of the $\|\cdot\|_2$--norms of the columns of $A$, for any $u$ in the support of $F_T$,
\[ \| u\|_2  \le \| H_k X_{-k} \| \|b_{-k}-\be_{-k}^0\|_1 \le \| H_k X_{-k} \| \eps_n =: \delta_n \to 0 \]
using \eqref{eq:condepsk}.
In particular,
\[ 
\| \mathcal{L}_{\overline{Q}}\left(L_k^{-1} \be_{1:k}^* + T\right) - 
 \mathcal{L}_{\overline{Q}}\left(L_k^{-1} \be_{1:k}^* \right) \|_{TV} 
\le 
\sup_{\|u\|_2 \le \delta_n} \|
\mathcal{L}_{\overline{Q}}\left(L_k^{-1} \be_{1:k}^* + u\right) - 
 \mathcal{L}_{\overline{Q}}\left(L_k^{-1} \be_{1:k}^* \right) \|_{TV}.
\]
In summary, 
\begin{align*}
\| & \mathcal{L}_{\overline{Q}}(\be_{1:k}) - N(\hat\be_k, \Sigma_k) \|_{TV}   \\ 
& \quad  \le \sup_{\|u\|_2\le \delta_n}   
\| \mathcal{L}_{\overline{Q}}\left(L_k^{-1} \be_{1:k}^* + u\right) - 
 \mathcal{L}_{\overline{Q}}\left(L_k^{-1} \be_{1:k}^* \right) \|_{TV}
\ +\  \left\|    
\mathcal{L}_{\overline{Q}} \left(
L_k^{-1} \be_{1:k}^*  \right) 
- N\left(L_k^{-1} z, I_k
\right)  
\right\|_{TV}\\
&  =: (I) + (II).
\end{align*}
We now separately bound the terms (I) and (II) in the two cases for $g$.

\textbf{Improper prior $g \propto 1$.}
Recall the density of $\beta_{1:k}^*$ under the VB posterior $\hat{Q}$, and hence  $\overline{Q}$, equals the exact posterior for $\beta_{1:k}^*$. By Lemma \ref{lem:posterior_form_kD} with $g \propto 1$, this posterior density is $\mathcal{L}(\be_{1:k}^* | Y) = N(z,\Sigma_k)$ and hence $(II)=0$.
It thus suffices to bound (I), where now
\[ (I) = \sup_{\|u\|_2\le \delta_n} \| N(L_k^{-1}z+u,I_k) - N(L_k^{-1}z,I_k)\|_{TV}. \]
By Pinsker's inequality, the TV distance in the last display is bounded by $\sqrt{2K}$, where $K$ is the KL--divergence between the two Gaussian distributions in question. The latter equals $\|u\|_2^2/2$, so the last display is $O(\delta_n)=o(1)$, which concludes the proof. 

\textbf{Log-lipschitz prior $g$ with Lipschitz constant satisfying \eqref{cond:c}.}
Term (II): Recall the density $q_{1:k}^*$ of $\beta_{1:k}^*$ under the VB posterior $\hat{Q}$, and hence  $\overline{Q}$, equals the exact posterior for $\beta_{1:k}^*$. Using the explicit form of this posterior density given in Lemma \ref{lem:posterior_form_kD},
\begin{align*}
(II) & =\int \left| \frac{e^{-(L_kw-z)^T\Sigma_k^{-1}(L_kw-z)/2 } g(L_k w)}{ \int e^{-(L_kw'-z)^T\Sigma_k^{-1}(L_kw'-z)/2 } g(L_k w')dw'}  - \frac{e^{-\|w-L_k^{-1}z\|^2/2}}{(2\pi)^{k/2}} \right| dw\\
& = \int \left| \frac{e^{-\|v\|^2/2} g(z+L_k v)}{ \int e^{-\|u\|^2/2} g(z+L_k u)du}  - \frac{e^{-\|v\|^2/2}}{(2\pi)^{k/2}} \right| dv\\
& = \int \phi_k(v) \left| \frac{g(z+L_k v)}{ \int \phi_k(u) g(z+L_k u)du}  - 1 \right| dv,
\end{align*}  
with $\phi_k$ the density of a Gaussian $N(0,I_k)$ distribution and using the change of variable $v=w-L_k^{-1}z$.  In the notation of Lemma \ref{techtv}, the last display equals $\mathcal{D}_{g,L_k,z}$. 
By Lemma \ref{techtv}, to show that this term is $o(1)$, it suffices to show that 
\[ \mathcal{H}:= \mathcal{H}_{g,L_k,z} = \int\phi_k(v)  \left| \frac{g(z+L_k v)}{g(z)} - 1 \right| dv = o(1). \]
For a matrix $A$, denote by $\|A\|_{(1)}$ the maximum of the $L^1$--norms of columns of $A$.
Define the following growing ball in $\R^k$, with radius $M_n$ as in \eqref{Mn},
\begin{align}
 K_n & :=\left\{v\in\R^k:\ \|v\|_1\le M_n\right\}. \label{defkn}
\end{align} 
By the definition \eqref{Mn} of $M_n$, we have $c\|L_k\|_{(1)}M_n=o(1)$. Then 
\begin{align*}
\mathcal{H} = \int_{K_n} \phi_k(v)  \left| \frac{g\left(z+L_kv\right)}{g(z)} - 1 \right| dv 
+ \int_{K_n^c} \phi_k(v)  \left| \frac{g\left(z+L_kv\right)}{g(z)} - 1 \right| dv =: (i)+(ii).
\end{align*}
For $v\in K_n$, we have
 $c\|L_k v\|_1\le c\|L_k\|_{(1)}\|v\|_1\le  c\|L_k\|_{(1)}M_n=o(1)$ as noted above.  
Using that $g$ is log-Lipschitz together with $|e^x-1|\le 2|x|$ for small enough $x$,
\[ (i)\le 2 \int_{K_n} \phi_k(v)  c \|L_k\|_{(1)} \|v\|_1 dv\le 2 c \|L_k\|_{(1)}M_n \int \phi_k(v)dv =o(1).\]
For (ii), one first uses the triangle inequality $|a-1|\le |a|+1$ and next  the log-Lipschitz property of $g$ to get
\[ (ii) \le \int_{K_n^c} \phi_k(v) e^{c\|L_k\|_{(1)}\|v\|_1}  dv 
+  \int_{K_n^c} \phi_k(v)dv. \]
Since 
$ck\|L_k\|_{(1)}=o(1)$ and $k=o(M_n)$ by assumption, Lemma \ref{lemboundphi} implies
\[  \int_{K_n^c}  \phi_k(v) 
  \left( e^{c\|L_k\|_{(1)}\|v\|_1}+1 \right)  dv
\le e^{2k^2 c\|L_k\|_{(1)}-M_n/4} + e^{-M_n/4} = O(e^{o(k)-M_n/2})+e^{-M_n/2}=o(1)
\] 
since $M_n\to \infty$. Combining the previous bounds yields to $\mathcal{H}=o(1)$, which implies $(II)=o(1)$.

  
Term (I):  Using again Lemma \ref{lem:posterior_form_kD}, and noting that the normalising constants of the two densities at stake are the same (as they only differ by a translation),
\begin{align*}
(I) & =  \sup_{\|u\|_2 \le \delta_n} 
\int \Big| \frac{e^{-(L_k(w-u)-z)^T\Sigma_k^{-1}(L_k(w-u)-z)/2 } g(L_k (w-u))}{ \int e^{-(L_k(w'-u)-z)^T\Sigma_k^{-1}(L_k(w'-u)-z)/2 } g(L_k (w'-u))dw'} \\
& \qquad\qquad \qquad - \frac{e^{-(L_kw-z)^T\Sigma_k^{-1}(L_kw-z)/2 } g(L_k w)}{ \int e^{-(L_kw'-z)^T\Sigma_k^{-1}(L_kw'-z)/2 } g(L_k w')dw'}   \Big| dw\\ 
&  = \sup_{\|u\|_2 \le \delta_n}  \int \left| \frac{e^{-\|v-u\|^2/2} g(z+L_k (v-u))}{ \int e^{-\|v'-u\|^2/2} g(z+L_k (v'-u))dv'}  - \frac{e^{-\|v\|^2/2} g(z+L_k v)}{ \int e^{-\|v'\|^2/2} g(z+L_k v')dv'} \right| dv\\
& = \sup_{\|u\|_2 \le \delta_n} \int \phi_k(v) \left| \frac{e^{\langle u,v \rangle -\|u\|^2/2} g(z+L_k (v-u))}{ \int \phi_k(v'-u) g(z+L_k (v'-u))dv'}  - \frac{g(z+L_k v)}{ \int \phi_k(v')  g(z+L_k v')dv'} \right| dv\\
& = \frac{g(z)}{D} \sup_{\|u\|_2 \le \delta_n} \int \phi_k(v) \left| \frac{e^{\langle u,v \rangle -\|u\|^2/2} g(z+L_k (v-u))}{  g(z) }  - \frac{g(z+L_k v)}{ g(z) } \right| dv
\end{align*}
where $D=D_{g,L_k,z}=\int \phi_k(v)g(z+ L_k v)dv$ in the notation of Lemma \ref{techtv}. Let $J(u)$ denote the integral in the last display. We have already proved above in the study of term (II) that $\mathcal{H}=o(1)$, which by Lemma \ref{techtv} implies that $g(z)/D\to 1$, so that it suffices to show $\sup_{\|u\|_2 \leq \delta_n} J(u)=o(1)$ to obtain $(I) = o(1)$.
  
Similar to the term (II), we split $J(u)$ into two parts $J_1(u)$ and $J_2(u)$ with integration domains $K_n$ and $K_n^c$, respectively, where $K_n$ and its radius $M_n$ are defined in \eqref{defkn} and \eqref{Mn}. For $v\in K_n$,
 \[ \frac{g(z+L_k(v-u))}{ g(z) } =e^{O\left(c\|L_k\|_{(1)}( \|u\|_1 + \|v\|_1) \right)}
 = e^{O\left(c\sqrt{k}\|L_k\|_{(1)}\delta_n\right)+O(m_n)},
  \]
  using $\|u\|_1\le \sqrt{k}\|u\|_2 \le \sqrt{k} \delta_n$, the definition of $K_n$ and $m_n = c\|L_k\|_{(1)}M_n \to 0$ by \eqref{Mn}. Setting $u=0$, one also gets 
\[ \frac{g(z+L_k v)}{g(z)}=e^{O(m_n)}, 
 \]
uniformly over $\|u\|_2 \le \delta_n$. Turning to the integral in the second last display restricted to $K_n$ and putting all terms in exponential form,
\begin{align*}
\sup_{\|u\|_2\le\delta_n} J_1(u) & \leq \int_{K_n} \phi_k(v) \left| e^{ O(\|u\|_\infty\|v\|_1+\|u\|^2)+O\left(c\sqrt{k}\|L_k\|_{(1)}\delta_n\right)+O(m_n)} - e^{O(m_n)} \right| dv.
\end{align*}   
The last absolute value can also be written, using $\|v\|_1\le M_n$ on $K_n$,
\[ \left| e^{ O(\delta_n M_n+\delta_n^2)+O\left(c\sqrt{k}\|L_k\|_{(1)}\delta_n\right)+O(m_n)} - e^{O(m_n)} \right|  \]
Now $\delta_n M_n=(\delta_n k) R_n=\delta_n' R_n=o(1)$ and $c\sqrt{k}\|L_k\|_{(1)}\delta_n=o(\delta_n/\sqrt{k}) = o(1)$ due to the condition \eqref{cond:c}.   
 Thus, all the terms in the exponent in the last display are vanishing as $n\to\infty$. Using the inequality $|e^x-e^y|\le 2|x|+2|y|$ for $|x|,|y|$ sufficiently small then yields $\sup_{|u|\leq\delta_n} J_1(u) =o(1)$.
 
It remains to deal with $J_2(u)$.  Similar to its counterpart (ii) for the term (II) above, using that $g$ is log-Lipschitz and $\|u\|_1\le \sqrt{k}\|u\|_2$,
 \begin{align*}
  \sup_{\|u\|_2\leq\delta_n} J_2(u) & \leq \int_{K_n^c} \phi_k(v) \left( 
 e^{c\|L_k\|_{(1)}(\|v\|_1+\|u\|_1)} + e^{c\|L_k\|_{(1)}\|v\|_1} \right) dv \\
 &   \leq \int_{K_n^c} \phi_k(v) \left( 
 e^{c\|L_k\|_{(1)}(\|v\|_1+\sqrt{k}\delta_n)} + e^{c\|L_k\|_{(1)}\|v\|_1} \right) dv\\
 & \le 4 \int_{K_n^c}  \phi_k(v) 
  \left( e^{2c\|L_k\|_{(1)}\|v\|_1}+e^{2c\|L_k\|_{(1)}\sqrt{k}\delta_n}  \right)  dv,
 \end{align*} 
where one writes the first exponential in the last but one line as a product and uses the inequality $ab\le 2a^2+2b^2$. One now bounds the integral in the last display using Lemma \ref{lemboundphi}. Since 
$ck\|L_k\|_{(1)}=o(1)$ by \eqref{cond:c} and $k=o(M_n)$ by the definition \eqref{Mn} of $M_n$, 
\begin{align*}
  \int_{K_n^c}  \phi_k(v) 
  \left( e^{c\|L_k\|_{(1)}\|v\|_1}+e^{2c\|L_k\|_{(1)}\sqrt{k}\delta_n}  \right)  dv
& \le 4e^{4k^2 c\|L_k\|_{(1)}-M_n/4} + 4e^{o(1)-M_n/4} \\
& = O(e^{o(k)-M_n/2}+e^{o(1)-M_n/2})=o(1)
\end{align*}
since $M_n\to \infty$. 
Hence $\sup_{\|u\|_2 \le\delta_n} J_2(u) = o(1)$. Combining everything yields $(I) = o(1)$.
 
Putting together the previous bounds, we have shown $\| \mathcal{L}_{\hat{Q}}(\be_{1:k})  - N(\hat\be_k, \Sigma_k) \|_{TV}=o(1)$, which concludes the proof for the log-Lipschitz prior. 
\end{proof}

\subsection{Technical Lemmas} \label{sectechvb}

\begin{lemma} \label{techtv}
Let $z\in\R^k$, $L$ be a $k\times k$ matrix, $\phi_k$ denote the density of a $\mathcal{N}_k(0,I_k)$ variable, and $g:\R^k \to (0,\infty)$ be a measurable map. Define
\begin{align*}
D_{g,L,z}&:=\int\phi_k(u)g(z+Lu)du,\\
\mathcal{D}_{g,L,z}&:= \int \left| \frac{\phi_k(v)g(z+Lv)}{\int\phi_k(u)g(z+Lu)du} - \phi_k(v) \right| dv 
=  \int\phi_k(v)  \left| \frac{g(z+Lv)}{D_{g,L,z}} - 1 \right| dv , \\
\mathcal{H}_{g,L,z} &:=\int\phi_k(v)  \left| \frac{g(z+Lv)}{g(z)} - 1 \right| dv.
\end{align*}
If $\mathcal{H}_{g,L,z}<1$, then
\[ \left| \frac{g(z)}{D_{g,L,z}} -1 \right| \le \frac{\mathcal{H}_{g,L,z}}{1-\mathcal{H}_{g,L,z}}, \qquad \quad \text{and} \qquad \quad  \mathcal{D}_{g,L,z} \le  \frac{2\mathcal{H}_{g,L,z}}{1-\mathcal{H}_{g,L,z}}.  \]
\end{lemma}

\begin{proof}
We drop the subscripts in the above quantities, writing $D,\mathcal{D},\mathcal{H}$ for conciseness. Using that $\phi_k$ is a density,
\begin{align*}
 \left| \frac{g(z)}{D} -1 \right|
= \left| \frac{g(z) - \int \phi_k(v)g(z+Lv)dv}{D} \right|  
& = \frac{g(z)}{D} \left| \int \phi_k(v)\left(1 -  \frac{g(z+Lv)}{g(z)}\right) dv   \right|\\
&  \leq \frac{g(z)}{D} \mathcal{H}  \le \left(1+  \left| \frac{g(z)}{D} -1 \right| \right)\mathcal{H}.
\end{align*}
Rearranging yields the first assertion $\left| \frac{g(z)}{D} -1 \right|\le \frac{\mathcal{H}}{1-\mathcal{H}}$, provided $\mathcal{H}<1$. Setting
\[ x = \frac{g(z+Lv)}{g(z)}, \qquad \qquad  y= \frac{g(z)}{D} 
 \]
and using that $|xy-1|\le |x-1|+|y-1|+|x-1||y-1|$, one can bound
\begin{align*}
 \mathcal{D}  & \le \int \phi_k(v)  \left| \frac{g(z+Lv)}{g(z)} - 1 \right| dv  \ + \  \int \phi_k(v) \left| \frac{g(z)}{D} -1 \right|dv \\
& \ \ +   \left| \frac{g(z)}{D} -1 \right| \int \phi_k(v)  \left| \frac{g(z+Lv)}{g(z)} - 1 \right| dv\\
& = \mathcal{H} + \left| \frac{g(z)}{D} -1 \right| \left( 1 + \mathcal{H} \right).
\end{align*}
Using the first assertion of the lemma in the last display then gives $\mathcal{D} \leq 2\mathcal{H}/(1-\mathcal{H})$.
\end{proof}

\begin{lemma} \label{lemboundphi}
Let $\phi_k$ denote the $\mathcal{N}_k(0,I_k)$ density. Then for any $M,B>0$ such that $kB<1/2$ and $M>4k$, 
\[  \int_{\|v\|_1> M} \phi_k(v) e^{B\|v\|_1} dv 
\le  e^{2k^2B-M/4} .
 \]
\end{lemma}
\begin{proof}
One can bound $\|v\|_1\le \sqrt{k}\|v\|_2\le 1/2 + k\|v\|_2^2/2$, so that the integral in the Lemma is bounded from above by
\begin{align*} 
e^{B/2}& \int_{\|v\|_1>M} (2\pi)^{-k/2}  e^{-(1-kB)\|v\|_2^2/2} dv\\
& =e^{B/2}(1-kB)^{-k/2} \int_{\|v\|_1>M} (2\pi/(1-kB))^{-k/2} e^{-(1-kB)\|v\|_2^2/2} dv\\
& =e^{B/2-(k/2)\log(1-kB)} P\left[ \left\| \mathcal{N}_k(0,\tfrac{1}{1-kB}I_k) \right\|_1 >M\right]\\
& \le e^{B/2+k^2B} P\left[ \| \mathcal{N}_k(0,I_k) \|_1 > (1-kB)^{1/2} M\right],
\end{align*}
where in the last line we have used the inequality $\log(1-x)\ge -2x$ which holds for $0\le x\le 1/2$, setting $x=kB<1/2$ by assumption. Noting that $(1-kB)M>M/2\geq 2k > 2k \sqrt{\log 2}$ by assumption, the result follows by Lemma \ref{lembasic1}.
\end{proof}

\begin{lemma} \label{lembasic1}
For $\mathcal{N}_k(0,I_k)$ a $k$--dimensional standard Gaussian distribution and any $u>2k\sqrt{\log 2}$,
\[ P\left( \| \mathcal{N}_k(0,I_k) \|_1 >u \right) \le e^{-u^2/(4k)} \leq e^{-u/2}. \]  
\end{lemma}
\begin{proof}
By Markov's inequality, since $\|\mathcal{N}_k(0,I_k)\|_1$ equals in distribution $\sum_{i=1}^k |\eps_i|$ for $\eps_i$ independent $N(0,1)$ variables, the probability in the Lemma is bounded by
\[ e^{-ut} E\left[ e^{t \sum_{i=1}^k |\eps_i|}\right]=e^{-ut}(Ee^{t|\eps_1|})^k.\]  
for any $t>0$. Since $Ee^{t|\eps_1|} = 2e^{t^2/2} P(N(0,1) \leq t) \leq 2e^{t^2/2}$, the last display is bounded by $2^k e^{kt^2/2-ut}$. Optimizing and setting $t = u/k$ gives the bound $2^k e^{-u^2/(2k)}$. For $u>2k \sqrt{\log 2}$, this gives the required bound.
\end{proof}

\section{Computational cost}\label{sec:computational_cost}

We discuss the computational cost of Algorithm \ref{alg:sample} for sampling from the VB posterior when using Coordinate Ascent Variational Inference (CAVI) (e.g. using the R-package \texttt{sparsevb} \citep{sparsevb}) and an improper prior slab $g \propto 1$. We show that one can generate $n_s$ VB posterior samples at cost $O(n^2 p + cost_{CAVI} + n_s p)$, where $cost_{CAVI}$ is the cost of running the CAVI algorithm to solve \eqref{eq:optimization_k=1}. Since the optimization problem \eqref{eq:optimization_k=1} is high-dimensional and non-convex, we discuss its computational cost in more detail below.

For the pre-computations, one can compute the orthonormal basis $\{u_1,\dots,u_{n-1}\}$ of $\textrm{span}(X_1)^\perp$, and hence $P$, in $O(n^2)$ time using Householder reflection. One can then compute $\cY = P^TY$ and $\cW = P^TX_{-1}$ in $O(n^2 p)$ time using standard matrix multiplications, which dominates the cost of computing $P$. Since one can similarly compute $(\gamma_i)_{i=2}^p$ with cost $O(n^2 p)$, the overall cost of the pre-computations is $O(n^2 p)$.

Turning to $cost_{CAVI}$, each CAVI iteration cycles through the $O(p)$ parameters $(\mu_i,\tau_i^2,q_i)$ in the variational family \eqref{eq:variational_class}, updating one parameter at a time keeping all other paramters fixed using the updating equations (16)-(17) in \cite{Ray_Szabo_2020}. These equations require upfront computation of the Gram-matrix $X^T X$ as well as further $O(np)$ computations. Computing $X^TX$ exactly costs $O(np^2)$, but efficient low-rank approximations exist that can reduce this cost, such as Nystr\"om approximations \citep{Drineas2005}. Each $q_i$ update can then be performed exactly and requires $O(p)$ computations, while each $\mu_i$ and $\tau_i^2$ update can be performed using any one-dimensional optimization algorithm (e.g. the limited memory Broyden-Fletcher-Goldfarb-Shanno (BFGS) algorithm is used in \cite{sparsevb}). The cost per iteration of CAVI is thus $O(p^2 + p\cdot cost_{1D})$ with an upfront cost of $O(np^2)$, where $cost_{1D}$ is the cost of the one-dimensional optimizer used. Note that \cite{Ray_Szabo_2020} recommends using a data-driven initialization and updating order, which requires pre-computation of an estimator such as the LASSO.

Lastly, each VB posterior sample $\beta_1 = (\beta_1^*) - \sum_{i\geq 2} \gamma_i \beta_i$ requires $O(1)$ to sample $\beta_1^*$ (Lemma \ref{lem:posterior_form} after an upfront cost of $O(n^2)$), $O(p)$ to sample $\beta_{-1} \sim \hat{Q}_{-1}$ (see \eqref{eq:variational_class}) and $O(p)$ to compute, leading to an overall cost of $O(n^2 + n_s p)$. Summing these gives the overall cost $O(n^2 p + cost_{CAVI} + n_s p)$ stated above. 

Note that in general, Algorithm \ref{alg:sample} requires only a pre-processing step of cost $O(n^2p)$ to yield two separate optimization problems for the computation of $\beta_1^*$ and $\beta_{-1}$, which are just transformed linear regression problems which can be solved separately and in parallel using whichever algorithm is desired. Let $cost_{\beta_1^*,n_s}$ be the cost of generating $n_s$ samples of $\beta_1^*$ (e.g. using conjugacy if $g\propto 1$ or MCMC if $g$ is a Laplace or other distribution) and $cost_{\beta_{-1},n_s}$ be the cost of generating $n_s$ samples of $\beta_{-1}$ from the variational posterior $\hat{Q}_{-1}$ (e.g. using CAVI or another sampling algorithm). Then the overall cost of the method is $O(n^2 p + cost_{\beta_1^*,n_s} + cost_{\beta_{-1},n_s} + n_s p)$.

\section{Design matrix conditions}\label{sec::Additional_result}

We provide here some further discussion to help understand the conditions on the design matrix required by our theoretical results. Our Bernstein-von Mises results in Section \ref{sec:BvM_VB_1d} require conditions on the compatibility constants of the transformed matrix $\cW$ in \eqref{eq:preprocess} rather than the original design $X$, which has been extensively studied in the literature \citep{Buhlmann2011}. The next lemma shows how one can relate these.

\begin{lemma}\label{lem:compatibility_number}
For any $S \subset \{1,\dots,p\}$ with $|S| > 1$, 
 $$
 \phi^{\cW}(S\backslash\{1\}) \geq \sqrt{1 - \frac{1}{|S|}}\phi^X(S) - \frac{\|X_1\|_2}{\|\cW\|} \sqrt{|S|} \max_{i \geq 2} |\gamma_i| .
 $$
 Furthermore, for any $s \in \{1,\dots,p\}$, $$\tilde{\phi}^{\cW}(s-1) \geq \tilde{\phi}^X(s) - \frac{\|X_1\|_2}{\|\cW\|}\sqrt{s} \max_{i \geq 2} |\gamma_i|.$$
 In particular, if \eqref{assum::design_matrix} is satisfied, then
 $$\phi^{\cW}(S\backslash\{1\}) \geq \sqrt{1 - \frac{1}{|S|}}\phi^X(S) - o(1) \quad \textrm{and} \quad \tilde{\phi}^{\cW}(s-1) \geq \tilde{\phi}^X(s) - o(1).$$
\end{lemma}

\begin{proof}[Proof of Lemma \ref{lem:compatibility_number}]
For the first statement, write $S_* = S\backslash\{1\}$. 
	Observe that, for any $\beta \in \mathbb{R}^{p-1}$,
	\begin{align}\label{eq:proof_inequality}
	\|\cW\beta\|_2 &= \left\|X\left[\begin{matrix}
0 \\ \beta	
\end{matrix}
 \right] - X_1\sum_{i \geq 2}\gamma_i \beta_i \right\|_2  \geq \left\|X\left[\begin{matrix}
0 \\ \beta	
\end{matrix}
 \right]\right\|_2 - \|X_1\|_2 \left|\sum_{i \geq 2}\gamma_i \beta_i \right| 
 & \geq \left\|X\left[\begin{matrix}
0 \\ \beta	
\end{matrix}
 \right]\right\|_2 - \|X_1\|_2\|\beta\|_1\max_{i \geq 2}|\gamma_i |.
	\end{align}
Then we have,
\begin{align*}
\phi^{\cW}(S_*) &= \inf_{\|\beta_{S_*^C}\|_1 \leq 7\|\beta_{S_*}\|_1}	\frac{\sqrt{|S_*|}\|\cW\beta\|_2}{\|\cW\|\|\beta_{S_*}\|_1} \geq \inf_{\|\beta_{S_*^C}\|_1 \leq 7\|\beta_{S_*}\|_1}	\frac{\sqrt{|S_*|}\left[\left\|X\left[\begin{matrix}
0 \\ \beta	
\end{matrix}
 \right]\right\|_2 - \|X_1\|_2 \|\beta\|_1 \max_{i \geq 2}|\gamma_i|\right] }{\| \cW \|\|\beta_{S_*}\|_1} \\
 &\geq \inf_{\|\beta_{S_*^C}\|_1 \leq 7\|\beta_{S_*}\|_1}	\frac{\sqrt{|S_*|}\left\|X\left[\begin{matrix}
0 \\ \beta	
\end{matrix}
 \right]\right\|_2  }{\| \cW \|\|\beta_{S_*}\|_1} - \frac{ \sqrt{|S_*|}\|X_1\|_2 \|\beta\|_1 \max_{i \geq 2}|\gamma_i|}{\| \cW \|\|\beta_{S_*}\|_1}.
\end{align*}
Now, when $\|\beta_{S_*^C}\|_1 \leq 7\|\beta_{S_*}\|_1$ we have $\|\beta\|_1/\|\beta_{S_*}\|_1 \leq 8$, so the second term is bounded below by $-\frac{\|X_1\|_2}{\|\cW\|}\sqrt{|S_*|}\max_{i \geq 2}|\gamma_i| = -\frac{\|X_1\|_2}{\|\cW\|}\sqrt{|S_*|} \max_{i \geq 2}|\gamma_i|$. For the first term we have
\begin{align*}
	\inf_{\|\beta_{S_*^C}\|_1 \leq 7\|\beta_{S_*}\|_1}	\frac{\sqrt{|S_*|}\left\|X\left[\begin{matrix}
0 \\ \beta	
\end{matrix}
 \right]\right\|_2  }{\|\cW\|\|\beta_{S_*}\|_1} &\geq \sqrt{1 - \frac{1}{|S|}}\inf_{\|\beta'_{S^C}\|_1 \leq 7\|\beta'_{S}\|_1}	\frac{\sqrt{|S|}\left\|X\beta'\right\|_2  }{\|X\|\|\beta'_{S}\|_1} \\
 &= \sqrt{1 - \frac{1}{|S|}}\phi^X(S).
\end{align*}
For the second statement, write $s_* = s-1$. Once again using Equation \eqref{eq:proof_inequality}, we have
\begin{align*}
\tilde{\phi}^{\cW}(s_*) &= \inf_{0 \neq |S_\beta| \leq s_*}	\frac{\|\cW\beta\|_2}{\|\cW\|\|\beta\|_2} \geq \inf_{0 \neq |S_\beta| \leq s_*}	\frac{\left[\left\|X\left[\begin{matrix}
0 \\ \beta	
\end{matrix}
 \right]\right\|_2 - \|X_1\|_2\|\beta\|_1 \max_{i \geq 2}|\gamma_i| \right] }{\| \cW \|\|\beta\|_2} \\
 &\geq \inf_{0 \neq |S_\beta| \leq s_*}	\frac{\left\|X\left[\begin{matrix}
0 \\ \beta	
\end{matrix}
 \right]\right\|_2  }{\| \cW \|\|\beta\|_2} - \frac{ \|X_1\|_2 \|\beta\|_1 \max_{i \geq 2}|\gamma_i|}{\| \cW \|\|\beta\|_2}.
\end{align*}

Now, when $0 \neq |S_\beta| \leq s_*$ we have $\|\beta\|_1/\|\beta\|_2 \leq \sqrt{s_*}$, so the second term is bounded below by $-\frac{\|X_1\|_2}{\|\cW\|}\sqrt{s_*} \max_{i \geq 2}|\gamma_i| = -\frac{\|X_1\|_2}{\|\cW\|}\sqrt{s_*} \max_{i \geq 2}|\gamma_i|$. For the first term we have,
\begin{align*}
	\inf_{0 \neq |S_\beta| \leq s_*}	\frac{\left\|X\left[\begin{matrix}
0 \\ \beta	
\end{matrix}
 \right]\right\|_2  }{\|\cW\|\|\beta\|_2} &\geq \inf_{0 \neq |S_{\beta'}| \leq s}	\frac{\left\|X\beta'\right\|_2  }{\|X\|\|\beta'\|_2} = \phi^X(s).
\end{align*}
\end{proof}

Lemma \ref{lem:compatibility_number} shows that if \eqref{assum::design_matrix} holds, then the compatibility constants for $\cW$ can be bounded from below by those for $X$, which is a much more commonly studied problem (e.g. Section D of \cite{Ray_Szabo_2020}). 

A related (but stronger) notion is the \textit{mutual coherence} of $X$, given by
    $$
    mc(X) = \max_{1 \leq i \neq j \leq p} \frac{|X_i^T X_j|}{\|X_i\|_2 \|X_j\|_2}.
    $$
Lemma 1 of \cite{CS-HV2015} allows one to bound the compatibility numbers of $X$ from below if one has control of the mutual coherence via the following inqualities
$$\phi(S)^2 \geq \bar{\phi}(1)^2 - 15\; |S| \; mc(X), \qquad \tilde{\phi}(s)^2 \geq \bar{\phi}(1)^2 - s \; mc(X),$$
where $\bar{\phi}(1) = \min_i \|X_i\|_2/\|X\|$ is typically bounded away from 0. We can relate the first quantity in our `no-bias' condition \eqref{assum::design_matrix} to the mutual coherence.

\begin{lemma}\label{lem:design_bound_mc}
We have
\begin{align*}
    \frac{\|X_1\|_2 \max_{i = 2, \dots, p} |\gamma_i|}{\max_{i = 2, \dots, p}\|(I - H)X_i\|_2} = \frac{ \max_{i=2, \dots, p}\|HX_{i}\|_2 }{\max_{i=2, \dots, p}\|(I-H)X_{i}\|_2} \leq \frac{mc(X)}{1-mc(X)}.
\end{align*}
\end{lemma}

\begin{proof}
    We begin with
    \begin{align*}
\left[\frac{ \max_{i=2, \dots, p}\|HX_{i}\|_2 }{\max_{i=2, \dots, p}\|(I-H)X_{i}\|_2}\right]^2 \leq \left[\max_{i = 2,\dots, p} \frac{\|H X_i\|_2}{\|(I-H)X_i\|_2}\right]^2 =  \max_{i = 2,\dots, p}\frac{\|H X_i\|^2_2}{\|X_i\|^2_2 - \|HX_i\|_2^2}.
    \end{align*}
As $\|H X_i\|_2^2 = (X_i^T X_1)^2/\|X_1\|_2^2$,
\begin{align*}
 \max_{i = 2,\dots, p}\frac{\|H X_i\|^2_2}{\|X_i\|^2_2 - \|HX_i\|_2^2} &=  \max_{i = 2,\dots, p}\frac{(X_i^T X_1)^2/\|X_1\|_2^2}{\|X_i\|^2_2 - (X_i^T X_1)^2/\|X_1\|_2^2} \\
 &= \max_{i = 2,\dots, p}\frac{(X_i^T X_1)^2/\|X_1\|_2^2\|X_i\|_2^2}{1 - (X_i^T X_1)^2/\|X_1\|_2^2\|X_i\|_2^2}\\
 &= \frac{\max_{i = 2,\dots, p}(X_i^T X_1)^2/\|X_1\|_2^2\|X_i\|_2^2}{1 - \max_{i = 2,\dots, p}(X_i^T X_1)^2/\|X_1\|_2^2\|X_i\|_2^2},
\end{align*}
since the function $x \mapsto x/(1-x)$ is increasing.
We now define
$$mc_1(X) := \max_{i = 2,\dots, p}(X_i^T X_1)/\|X_1\|_2\|X_i\|_2 \leq mc(X),$$
so that $$\left[\frac{ \max_{i=2, \dots, p}\|HX_{i}\|_2 }{\max_{i=2, \dots, p}\|(I-H)X_{i}\|_2}\right]^2  \leq \frac{mc_1(X)^2}{1-mc_1(X)^2} \leq \frac{mc(X)^2}{1-mc(X)^2} \leq \left[\frac{mc(X)}{1-mc(X)} \right]^2$$
since $x \mapsto \frac{x^2}{1-x^2}$ is increasing on $[0,1]$. Taking the square root of both sides completes the proof.
\end{proof}

If the entries of the design matrix $X_{ij}$ are i.i.d. with either $|X_{ij}| \leq C$ and $\log p = o(n)$ or $E e^{t_0 |X_{ij}|^\alpha} < \infty$ for some $\alpha, t_0 > 0$ and $\log p = o\left(n^{\frac{\alpha}{4+\alpha}} \right)$, then Theorems 1 and 2 of \cite{CaiJiang2011} give that $\sqrt{\frac{n}{\log p}} mc(X) \xrightarrow{P} 2$ as $n \rightarrow \infty$. The compatibility numbers are then bounded away from 0 with probability 1 asymptotically for models of size $s_0 = o(\sqrt{n/\log p})$. In addition, using Lemma \ref{lem:design_bound_mc}, our `no-bias' condition \eqref{assum::design_matrix} is satisfied for models of size $s_0 = o\left(\sqrt{n}/\log p \right)$.

As a concrete example to illustrate our conditions, we consider the case of the i.i.d. random design matrix $X \sim^{iid} \mathcal{N}(0,1)$. One can check that if the truth is sparse enough, then the conditions required in Theorem \ref{thm:asymptotic_normality_variational_1D} hold with probability tending to 1.

\begin{corollary}\label{cor::random_design}
Suppose that $\log p=o(n^{1/3})$ and $\beta^0\in\R^p$ is $s_0$-sparse with $s_0 =  o(\sqrt{n}/\log p)$. Consider the prior \eqref{general_prior} with $\nu$ satisfying \eqref{assum::prior_1}, $\lambda$ satisfying $2 \frac{\sqrt{n-1}}{p-1}  \leq \lambda \leq \, 2 \frac{\sqrt{(n-1) \log(p-1)}}{s_0}$ and $g$ is $c$-Lipschitz with $c=o(\sqrt{n})$. Then the semiparametric BvM holds for $\hat{Q}_{1}$. 
\end{corollary}

\begin{proof}[Proof of Corollary \ref{cor::random_design}]
	 Since the $X_{ij} \sim^{iid} \mathcal{N}(0,1)$ and $P^T P = I_{n-1}$, we have that the $\check{W}_{ij}$ are also i.i.d. $\mathcal{N}(0,1)$ and therefore the compatibility conditions can be verified using results on the compatibility of Gaussian matrices.
    	More precisely, since $\log p=o(n^{1/3})$ and $s_0=  o\left(\sqrt{n / \log p}\right)$, the compatibility numbers are bounded away from 0 on an event of probability 1 asymptotically as discussed above. Similarly, since $s_0= o(\sqrt{n}/\log p)$, condition \eqref{assum::design_matrix} is satisfied with probability 1 asymptotically.
	
	Since the $X_{ij}$ are i.i.d. $\mathcal{N}(0,1)$ and the $\check{W}_{ij}$ are also i.i.d. $\mathcal{N}(0,1)$,  we have $\|X_1\|_2 = \sqrt{n}(1+o_P(1))$ and $\|\check{W}\| = \sqrt{n-1}(1+o_P(1))$. Therefore, by the assumptions on $\lambda$ and $c$, the conditions required on $\lambda$ and $c$ in the deterministic case are satisfied on an event of probability 1 asymptotically.
\end{proof}

\section{A connection to a larger variational class}\label{sec:additional_VB}

As discussed in the introduction, our method can be viewed more generally as placing a rough MF variational approximation on the high-dimensional nuisance parameters $\beta_{-1}$ and then carefully modelling the conditional distribution $\beta_1|\beta_{-1}$. We now make this link precise, showing that our variational approximation \eqref{intuitive_approximation} corresponds to the KL-minimizer of such a variational family. More precisely, define the set of conditional densities:
$$
\mathcal{D} = \left\{ q:  \R \times \R^{p-1} \rightarrow  [0,\infty): \; q(\cdot| \beta_{-1}) \; \text{is a density with respect to Lebesgue measure for all } \beta_{-1} \in \R^{p-1} \right\},
$$
and the variational class 
$$\mathcal{V} = \left\{ q(\beta_{1} | \beta_{-1}) d\beta_{1} dQ_{-1}(\beta_{-1})  :\  q \in \mathcal{D}\,,\,  Q_{-1} \in \mathcal{Q}_{-1}  \right\}.$$ 
In Lemma \ref{lem::general_variational_class}, we show that the KL-minimizer between $\mathcal{V}$ and the posterior is given by $q=\pi(\cdot |\beta_{-1}, Y)$ with $\pi(\cdot |\beta_{-1}, Y)$ being the conditional distribution of $\beta_1|\beta_{-1}$ under the posterior, and $Q_{-1}=\hat{Q}_{-1}$. By rewriting the distribution \eqref{intuitive_approximation} in term of a distribution on $\beta_{-1}$ and a conditional distribution of $\beta_1|\beta_{-1}$ , one can see that this minimizer is exactly the approximation \eqref{intuitive_approximation}. Such partial factorisations have also been used in other settings, for instance binary regression \citep{Fasano2022}.

 \begin{lemma}\label{lem::general_variational_class}
 	For any $V = V(q, Q_{-1})$ in $\mathcal{V}$, 
 	\begin{align*}
 		\operatorname{KL}(V \| \Pi(\cdot | Y)) = \operatorname{KL}(Q_{-1} \| \Pi_{-1}(\cdot| Y)) + \int_{\R^{p-1}} \operatorname{KL}( q(\cdot |\beta_{-1}) \| \pi(\cdot |\beta_{-1}, Y)) dQ_{-1}.
 	\end{align*}
  Consequently,
 \begin{align*}
     \argmin_{V \in \mathcal{V}} \operatorname{KL}(V,\Pi(\cdot | Y)) = V(\pi(\cdot | \beta_{-1}, Y), \hat{Q}_{-1} ).
 \end{align*}
 \end{lemma}

 \begin{proof}[Proof of Lemma \ref{lem::general_variational_class}]
    We prove the first assertion. Let $V = V(q, Q_{-1}) \in \mathcal{V}$ with $Q_{-1}$ being absolutely continuous with respect to $\Pi_{-1}(\cdot|Y)$, otherwise the result is immediate. We have that 
    \begin{align*}
		dV(\beta)&=q(\beta_{1} | \beta_{-1}) d\beta_{1} \frac{dQ_{-1}}{d\Pi_{-1}(\cdot|Y)}(\beta_{-1})d\Pi_{-1}(\beta_{-1}|Y) \\
		& = \frac{q(\beta_{1} | \beta_{-1})}{\pi( \beta_1|\beta_{-1}, Y)}  \pi( \beta_1|\beta_{-1}, Y) d\beta_{1} \frac{dQ_{-1}}{d\Pi_{-1}(\cdot|Y)}(\beta_{-1})d\Pi_{-1}(\beta_{-1}|Y)\\
		&= \frac{q(\beta_{1} | \beta_{-1})}{\pi( \beta_1|\beta_{-1}, Y)} \frac{dQ_{-1}}{d\Pi_{-1}(\cdot|Y)}(\beta_{-1}) d\Pi(\beta|Y).
	\end{align*}
	Consequently, $V$ is absolutely continous with respect to $\Pi(\cdot|Y)$ and
	\begin{align*}
		\operatorname{KL}(V||\Pi(\cdot|Y)) &= \int \log(\frac{q(\beta_{1} | \beta_{-1})}{\pi( \beta_1|\beta_{-1}, Y)} \frac{dQ_{-1}}{d\Pi_{-1}(\cdot|Y)}(\beta_{-1})) dV(\beta) \\ 
		&=\int \log(\frac{q(\beta_{1} | \beta_{-1})}{\pi( \beta_1|\beta_{-1}, Y)}) dV(\beta) + \int \log( \frac{dQ_{-1}}{d\Pi_{-1}(\cdot|Y)}(\beta_{-1})) dV(\beta) \\
		&=\int_{\R^{p-1}} \int_{\R} \log(\frac{q(\beta_{1} | \beta_{-1})}{\pi( \beta_1|\beta_{-1}, Y)}) q(\beta_{1} | \beta_{-1})d\beta_{1} dQ_{-1}(\beta_{-1}) + \int_{\R^{p-1}} \log( \frac{dQ_{-1}}{d\Pi_{-1}(\cdot|Y)}) dQ_{-1} \\
		&=\int_{\R^{p-1}} \operatorname{KL}( q(\cdot |\beta_{-1}) \| \pi(\cdot |\beta_{-1}, Y)) dQ_{-1} + \operatorname{KL}(Q_{-1}|| \Pi_{-1}(\cdot|Y)).
	\end{align*}
 The second assertion can be easily deduced form the first assertion.
\end{proof}

\textbf{Data availability statement.} The real data used in Section \ref{sec:realdata} is that of the paper \cite{riboflavinData} and is publicly available from the journal's webpage  (\href{https://www.annualreviews.org/content/journals/10.1146/annurev-statistics-022513-115545#supplementary_data}{link}).

\bibliography{ref.bib}

\end{document}